%% file: main-ICML.tex
\documentclass{article} % For LaTeX2e
\usepackage{times}
\usepackage[accepted]{icml2025}
\input{math_commands.tex}

\usepackage{hyperref}
\usepackage{url}

\usepackage{wrapfig}

\usepackage[utf8]{inputenc} % allow utf-8 input
\usepackage[T1]{fontenc}    % use 8-bit T1 fonts
\usepackage{pgf, tikz, pgfplots}

\usepackage{booktabs}       % professional-quality tables
\usepackage{subcaption}
\usepackage{amsfonts}       % blackboard math symbols
\usepackage{bm} 
\usepackage{bbm}
\usepackage{nicefrac}       % compact symbols for 1/2, etc.
\usepackage{microtype}      % microtypography
\usepackage{xcolor}         % colors
\usepackage{graphicx}
\usepackage{amsmath}
\usepackage{amsthm}

\input{latex_resources/mysymbol.sty}
\input{latex_resources/juan_defs.sty}

\begin{document}

\twocolumn[
\icmltitlerunning{A Manifold Perspective on the Statistical Generalization of Graph Neural Networks}
\icmltitle{A Manifold Perspective on the Statistical Generalization of \\Graph Neural Networks}


\icmlsetsymbol{equal}{*}

\begin{icmlauthorlist}
\icmlauthor{Zhiyang Wang}{equal,upenn}
\icmlauthor{Juan Cervi\~no}{equal,mit}
\icmlauthor{Alejandro Ribeiro}{upenn}
\end{icmlauthorlist}

\icmlaffiliation{upenn}{Department of Electrical and Systems Engineering, University of Pennsylvania, Philadelphia, USA}
\icmlaffiliation{mit}{Laboratory of Information and Decision Systems
(LIDS), Massachusetts Institute of Technology, Cambridge, USA}

\icmlcorrespondingauthor{Zhiyang Wang}{zhiyangw@seas.upenn.edu}
\icmlcorrespondingauthor{Juan Cervi\~no}{jcervino@mit.edu}


\vskip 0.3 in
]



\printAffiliationsAndNotice{\icmlEqualContribution} % otherwise use the standard text.

\begin{abstract}
Graph Neural Networks (GNNs) extend convolutional neural networks to operate on graphs.  Despite their impressive performances in various graph learning tasks, the theoretical understanding of their generalization capability is still lacking. Previous GNN generalization bounds ignore the underlying graph structures, often leading to bounds that increase with the number of nodes -- a behavior contrary to the one experienced in practice. In this paper, we take a manifold perspective to establish the statistical generalization theory of GNNs on graphs sampled from a manifold in the spectral domain. As demonstrated empirically, we prove that the generalization bounds of GNNs decrease linearly with the size of the graphs in the logarithmic scale, and increase linearly with the spectral continuity constants of the filter functions. Notably, our theory explains both node-level and graph-level tasks. Our result has two implications: i) guaranteeing the generalization of GNNs to unseen data over manifolds; ii) providing insights into the practical design of GNNs, i.e., restrictions on the discriminability of GNNs are necessary to obtain a better generalization performance. We demonstrate our generalization bounds of GNNs using synthetic and multiple real-world datasets. 
\end{abstract}

\section{Introduction}
\input{introduction}

\section{Related works}
\input{relatedworks}

\section{Preliminaries}
\input{preliminary}

% \section{Problem Formulation}
\input{formulation}

\section{Experiments}

\input{simulations}

\section{Conclusion}
\input{conclusion}
\newpage
\section*{Impact Statement}
In this work, we explore the statistical generalization of GNNs from a manifold perspective by considering graphs sampled from manifolds. The impact of our work relies on showing that GNNs can effectively generalize to unseen data from the manifolds when the number of sampled points is large enough and the filter functions are continuous in the spectral domain. Our work also motivates the practical design of large-scale GNNs given that training on larger graphs attains a smaller generalization gap. Lastly, we observe that other than training on larger graphs, it is essential to restrict the discriminability of GNNs by putting assumptions on the spectral continuity of the filter functions in the GNNs.

\bibliography{bib}
\bibliographystyle{icml2025}

%%%%%%%%%%%%%%%%%%%%%%%%%%%%%%%%%%%%%%%%%%%%%%%%%%%%%%%%%%%%
\newpage
\appendix

\input{appendix}

\end{document}

%% file: math_commands.tex
\usepackage{amsmath,amsfonts,bm}

\def\eqref#1{equation~\ref{#1}}
\def\1{\bm{1}}

\DeclareMathAlphabet{\mathsfit}{\encodingdefault}{\sfdefault}{m}{sl}
\SetMathAlphabet{\mathsfit}{bold}{\encodingdefault}{\sfdefault}{bx}{n}

%% file: introduction.tex
Graph convolutional neural networks (GNNs) \citep{scarselli2008graph, defferrard2016convolutional, bruna2013spectral} have emerged as one of the leading tools for processing graph-structured data.
There is abundant evidence of their empirical success across various fields, including but not limited to weather prediction \citep{lam2023learning}, protein structure prediction in biochemistry \citep{jumper2021highly, strokach2020fast}, resource allocation in wireless communications \citep{wang2022learning}, social network analysis in sociology \citep{fan2020graph}, point cloud in 3D model reconstruction \citep{shi2020point} and learning simulators \citep{fortunato2022multiscale}. 

The effectiveness of GNNs relies on their empirical ability to \textit{predict} over unseen data. This capability is evaluated theoretically with \textit{statistical generalization} in deep learning theory \citep{kawaguchi2017generalization}, which quantifies the difference between the \textit{empirical risk} (i.e. training error) and the \textit{statistical risk }(i.e. testing error). Despite the abundant evidence of GNNs' generalization capabilities in practice, developing concrete theories to explain their generalization is an active area of research. Many recent works have studied the generalization bounds of GNNs without any dependence on the underlying model responsible for generating the graph data \citep{scarselli2018vapnik, garg2020generalization, verma2019stability}. Generalization analysis on graph classification, when graphs are drawn from random limit models, is also studied in a series of works \citep{ruiz2023transferability, maskey2022generalization, maskey2024generalization, levie2024graphon}. In this work, we take the manifold perspective to formulate graph data on continuous topological spaces, i.e., manifolds. We emphasize that manifolds are realistic models to generate graph data that enable rigorous theoretical analysis and a deep understanding of the behaviors of GNNs.

%GNNs success is measured by the ability to make correct predictions over unseen data. 

%In several applications, graphs can be considered as samples of a manifold. This is sometimes quite explicit as in the case of, e.g., point clouds \citep{bronstein2017geometric, bronstein2021geometric} and implicit as in the case of, e.g., wireless communication networks and citation networks \citep{ wu2019session}. 

We explore the generalization bound of GNNs through the lens of manifold theory on both node-level and graph-level tasks in the spectral domain. The graphs are constructed based on points randomly sampled from underlying manifolds, indicating that the manifold can be viewed as a statistical model for these discretely sampled points. As deep learning architectures have been established over manifolds \citep{wang2022convolutional, chew2024geometric}, the convergence of GNNs to manifold neural networks (MNNs) and the algebraical equivalence of these two frameworks facilitate a detailed generalization understanding of GNNs through spectral analysis.  
We demonstrate that, with an appropriate graph construction based on the sampled points from the manifold, the generalization gap between empirical and statistical risks decreases with the number of sampled points in the graphs (Figure \ref{fig:syntethic_subfig3}) on both node-level and graph-level tasks. More importantly, the generalization gap increases linearly with the continuity constants of frequency response functions of graph filters composing the GNN (Figure \ref{fig:syntethic_subfig4}). We observe that with spectral continuous filters, the GNNs are generalizable across different nodes or graphs generated from the same underlying manifold. This provides insight into the practical graph filter design from a spectral perspective. Moreover, the theoretical results indicate a trade-off between the discriminability and generalization capability of GNNs, suggesting that restrictions on the discriminability of GNNs are necessary to maintain generalization performance.

\begin{figure*}[htbp]
    \centering
    \begin{subfigure}{0.19\textwidth}
        \centering
        \includegraphics[width=1.\linewidth]{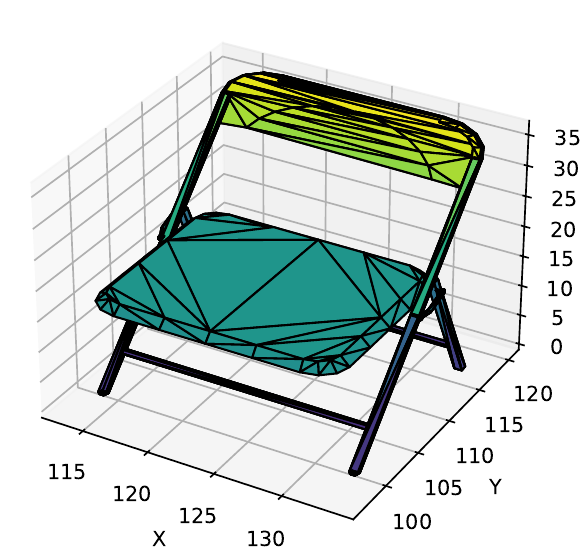}
        \caption{Chair Manifold}
        \label{fig:syntethic_subfig1}
    \end{subfigure}
    % \hspace{-0.5cm}
    \hfill
    \begin{subfigure}{0.19\textwidth}
        \centering
        \includegraphics[width=1.\linewidth]{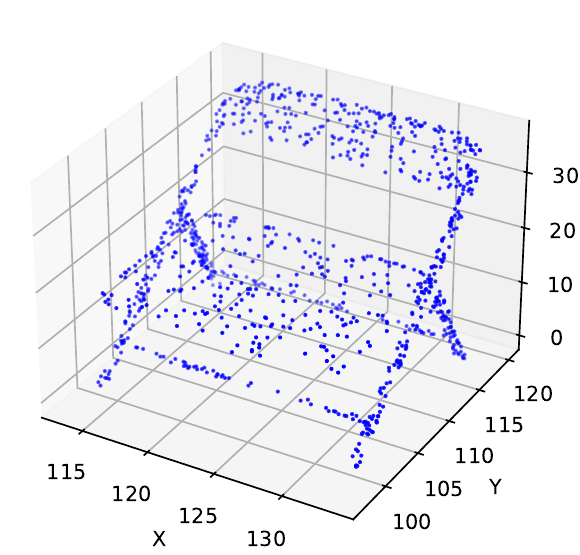}
        \caption{Sampled Chair}
        \label{fig:syntethic_subfig2}
    \end{subfigure}
    \hfill
    \begin{subfigure}{0.3\textwidth}
        \centering
        \includegraphics[width=\linewidth]{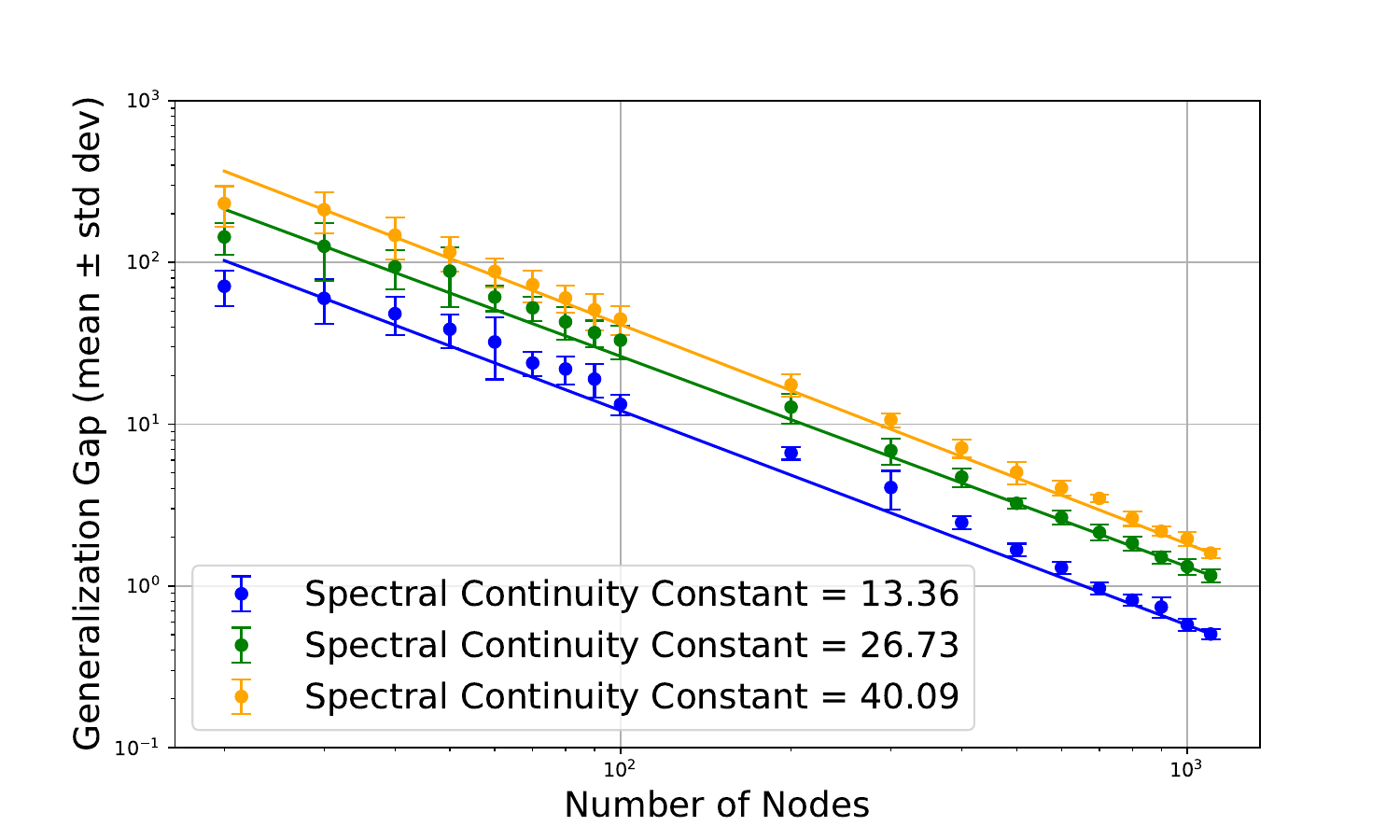}
        \caption{Gen. Gap vs. Num. of Nodes}
        \label{fig:syntethic_subfig3}
    \end{subfigure}
    \hfill
    \begin{subfigure}{0.3\textwidth}
        \centering
        \includegraphics[width=\linewidth]{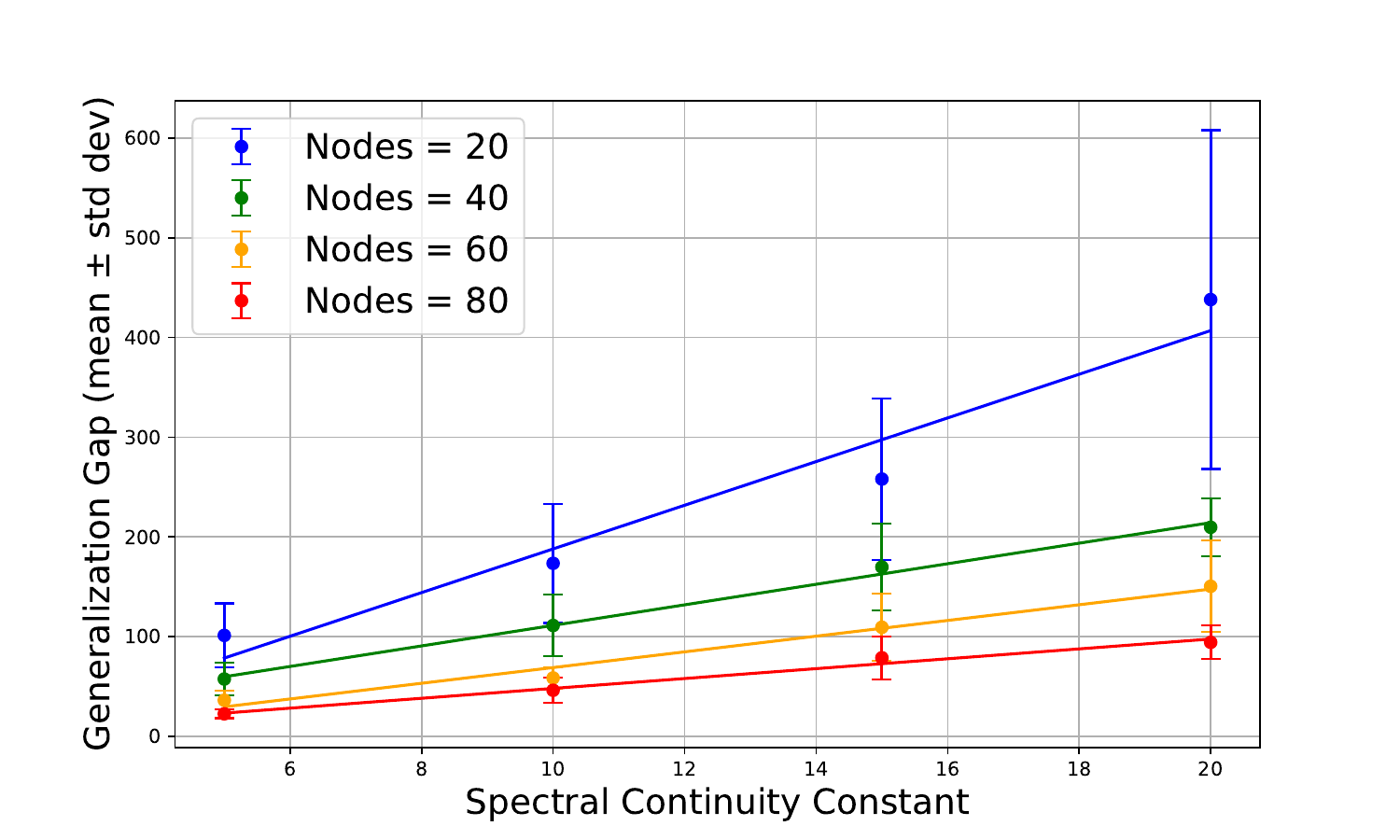}
        \caption{Gen. Gap vs. Continuity Constants}
        \label{fig:syntethic_subfig4}
    \end{subfigure}
    \caption{Synthetic experimental results are shown on the uniformly sampled chair manifold. We construct a graph with different numbers of nodes, fix the weights of a GNN, and compute the generalization gap. 
    % In , we show the manifold of a chair, and in Figure \ref{fig:syntethic_subfig2} we present a graph with $1000$ points sampled uniformly from the chair manifold.  
    % To compute the generalization gap, we fix the value of the GNN coefficients, and calculate the loss \eqref{eqn:empirical-graphloss-node} for the different number of nodes. 
    We construct the graph by computing the edges for nodes that are $\epsilon$ close (cf. \eqref{def:epsilon}). 
% In Figure \ref{fig:syntethic_subfig3} we show the generalization gap as a function of the number of nodes for GNNs with different Lipschitz constants. In Figure \ref{fig:syntethic_subfig4} we show the generalization gap as a function of the Lipschitz constants for different numbers of nodes in the underlying graph. As can be seen in these two Figures, the generalization gap is linear in the logarithmic scale (power law). 
In Figure  \ref{fig:syntethic_subfig3}, we fix the spectral continuity constant (see Assumption \ref{ass:low-pass}) and vary the number of nodes. As our theory predicts, we see that a smaller spectral continuity constant translates into a smaller generalization gap -- as the blue line is below the green line which is below the orange line. In Figure \ref{fig:syntethic_subfig4} we fix the number of nodes in the graph and vary the spectral continuity constant in the GNN. For the same number of nodes, a larger spectral continuity constant translates into a larger generalization gap. 
}
    \label{fig:syntethic_main}
\end{figure*}
We introduce a novel unified analysis of the generalization of GNNs to unseen nodes and graphs, by relating the GNNs with MNNs in the spectral domain. We further propose restrictions on the discriminability of GNNs from the spectral perspective which results from assumptions on the continuity of the filter frequency response functions. We provide extensive experiments both on synthetic and real-world datasets to verify our generalization conclusions. Our contribution is four-fold:
\begin{enumerate}
    \item We prove the generalization bound of GNNs on graphs generated from an underlying manifold on both node-level (Theorem \ref{thm:generalization-node-gauss}) and graph-level (Theorem \ref{thm:generalization-graph-gauss}) by relating the algebraically equivalent GNNs and MNN in the spectral domain.
    \item We provide novel generalization gap bounds that decrease linearly with the nodes of the graph in the logarithmic scale, and increase linearly with the spectral continuity constants (Assumption \ref{ass:low-pass}) of the filter functions.
    \item We uncover an important trade-off between the discriminability and the generalization gap of GNNs, which guides practical GNN designs.
    \item We verify the dependence of our generalization gaps on parameters, especially the continuity parameter, with a synthetic dataset -- chair manifold -- and eight real-world datasets -- ArXiv, Citeseer, etc.
\end{enumerate}

%% file: relatedworks.tex
\subsection{ Generalization bounds of GNNs}
\paragraph{Node level tasks} 
We first give a brief recap of the generalization bounds of GNNs on node level tasks. In \cite{scarselli2018vapnik}, the authors give a generalization bound of GNNs with a Vapnik–Chervonenkis dimension of GNNs. 
The authors in \cite{verma2019stability} analyze the generalization of a single-layer GNN based on stability analysis, which is further extended to a multi-layer GNN in \cite{zhou2021generalization}. %scales with the largest eigenvalue of the graph Laplacian. 
In \cite{ma2021subgroup}, the authors give a novel PAC-Bayesian analysis on the generalization bound of GNNs across arbitrary subgroups of training and testing datasets.
The authors derive generalization bounds for GNNs via transductive uniform stability and transductive Rademacher complexity in \cite{esser2021learning, cong2021provable, tang2023towards}. The authors in \cite{yehudai2021local} propose a size generalization analysis of GNNs correlated to the discrepancy between local distributions of graphs. Different from these works, we consider a continuous manifold model when generating the graph data, which is theoretically powerful and realistic when characterizing real-world data. Furthermore, the generalization bounds proved in these works either grow with the size of the graph \cite{esser2021learning,tang2023towards, scarselli2018vapnik}, with the node degree of the graphs \cite{  cong2021provable} or the maximum eigenvalues of the graph \cite{verma2019stability}. Notably, our generalization bound decreases with the size of the graph given that it depends on the spectral properties of the filter functions over the manifold.

% With the limit manifold model introduced, our generalization bound decreases with the size of the graph while depending on the spectral properties of the filter functions, which can be designed appropriately.} 

\paragraph{Graph level tasks} 
There are also related works on the generalization analysis of GNNs on graph-level tasks. In \cite{garg2020generalization}, the authors form the generalization bound via Rademacher complexity.
The authors in \cite{liao2020pac} build a PAC-Bayes framework to analyze the generalization capabilities of graph convolutional networks \cite{kipf2016semi} and message-passing GNNs \cite{gilmer2017neural}, based on which the authors in \cite{ju2023generalization} improve the results and prove a lower bound. The bounds either grow with the number of nodes \cite{liao2020pac} or the degree of the graphs \cite{garg2020generalization} while our bound decreases with the number of nodes in the graph given that it better approximates the underlying model -- the manifold. The works 
in \cite{maskey2022generalization, maskey2024generalization,levie2024graphon}  are most related to ours, which also consider the generalization of GNNs on a graph limit model, in their case a \textit{graphon}. Different from our setting, the authors see the graph limit as a random continuous model. They study the generalization of graph classification problems with message-passing GNNs with graphs belonging to the same category sampled from a continuous limit model. The generalization bound grows with the model complexity and decreases with the number of nodes in the graph. We show that a GNN trained on a single graph sampled from each manifold is enough, and can generalize and classify unseen graphs sampled from the manifold set.

\subsection{Neural networks on manifolds} Geometric deep learning has been proposed in \cite{bronstein2017geometric} with neural network architectures raised in manifold space. The authors in \cite{monti2017geometric} and \cite{chakraborty2020manifoldnet} provide neural network architectures for manifold-valued data.  In \cite{wang2024stability} and \cite{wang2022convolutional}, the authors define convolutional operation over manifolds and see the manifold convolution as a generalization of graph convolution, which establishes the limit of neural networks on large-scale graphs as manifold neural networks (MNNs). The authors in \cite{wang2023geometric, chew2023convergence,johnson2025manifoldfiltercombinenetworks} further establish the relationship between GNNs and MNNs with non-asymptotic convergence results for different graph constructions. Some studies have used graph samples to infer properties of the underlying manifold itself. These properties include the validity of the manifold assumption \cite{fefferman2016testing}, the manifold dimension \cite{farahmand2007manifold} and the complexity of these inferences \cite{narayanan2009sample,aamari2021statistical}. Other research has focused on prediction and classification using manifolds and manifold data, proposing various algorithms and methods. Impressive examples include the Isomap algorithm \cite{choi2004kernel,wu2004extended,yang2016multi} and other manifold learning techniques \cite{talwalkar2008large}. These techniques aim to infer manifold properties without analyzing the generalization capabilities of GNNs operated on the sampled manifold.

%% file: preliminary.tex
 
\subsection{Graph neural networks}

\paragraph{Setup } An undirected graph $\bbG = (\ccalV, \ccalE, \ccalW)$ contains a node set $\ccalV$ with $N$ nodes and an edge set $\ccalE \subseteq \ccalV\times \ccalV$. The weight function $\ccalW: \ccalE \rightarrow \reals$ assigns values to the edges. We define the graph Laplacian $\bbL =\text{diag}(\bbA \textbf{1})-\bbA$ where   $\bbA\in \reals^{N\times N}$ is the weighted adjacency matrix. Graph signals are functions mapping nodes to a feature value. We write it as a vector $\bbx \in \reals^N$, with each entry $[\bbx]_i$ representing the function value on node $i$. 
%A graph shift operator (GSO) \citep{ortega2018graph, shuman2013emerging} $\bbL \in \reals^{N\times N}$ is a graph matrix% with $[\bbL]_{ij} \neq 0$ if and only if $(i,j)\in \ccalE$ or $i=j$
% , e.g., the adjacency matrix, or the  

\paragraph{Graph convolutions and frequency response} 
A graph convolutional filter $\bbh_\bbG$ is composed of consecutive graph shifts by graph Laplacian, defined as $\bbh_\bbG(\bbL) \bbx = \sum_{k=0}^{K-1} h_k \bbL^k \bbx$ with 
%A graph convolution is defined based on a consecutive graph shift operation. A graph convolutional filter $\bbh_\bbG$ \cite{gama2019convolutional,ortega2018graph,sandryhaila2013discrete} with filter coefficients $\{h_k\}_{k=0}^{K-1}$ is formally defined as  
% \begin{equation}
%     \label{eqn:graph_convolution}
% \bbh_\bbG(\bbL) \bbx = \sum_{k=0}^{K-1} h_k \bbL^k \bbx,
% \end{equation}
$\{h_k\}_{k=0}^{K-1}$ as filter parameters. We replace $\bbL$ with eigendecomposition $\bbL = \bbV \bm\Lambda \bbV^H$, where $\bbV$ is the eigenvector matrix and $\bm\Lambda$ is a diagonal matrix with eigenvalues $\{\lambda_{i,N}\}_{i=1}^N$ as the entries. The spectral representation of a graph filter is
\begin{equation}
\label{eqn:graph_convolution_spectral}
    \bbV^H \bbh_\bbG(\bbL) \bbx =  \sum_{k=1}^{K-1} h_k \bm\Lambda^k \bbV^H \bbx = \hat{h}(\bm\Lambda)\bbV^H \bbx.
\end{equation}
This leads to a point-wise frequency response of the graph convolution as $\hat{h}(\lambda)= \sum_{k=0}^{K-1} h_k \lambda^k$.
%, depending only on the weights $\{h_k\}_{k=0}^{K-1}$ and on the eigenvalues $\bm\Lambda$ of $\bbL$.

\paragraph{Graph neural networks } A graph neural network (GNN) is a layered architecture, where each layer consists of a bank of graph convolutional filters followed by a point-wise nonlinearity $\sigma:\reals\to\reals$. Specifically, the $l$-th layer of a GNN that produces $F_l$ output features $\{\bbx_l^p\}_{p=1}^{F_l}$ with $F_{l-1}$ input features $\{\bbx^q_{l-1}\}_{q=1}^{F_{l-1}}$ is written as 
\begin{equation} 
\label{eqn:gnn-eq}
    \bbx_l^p = \sigma\left(\sum_{q=1}^{F_{l-1}} \bbh_\bbG^{lpq}(\bbL) \bbx^q_{l-1} \right),
\end{equation}
for each layer $l=1,2\cdots, L$. The graph filter $\bbh_\bbG^{lpq}(\bbL)$ maps the $q$-th feature of layer $l-1$ to the $p$-th feature of layer $l$. We denote the GNN as a mapping $\bm\Phi_\bbG(\bbH, \bbL, \bbx)$, where $\bbH\in \ccalH\subset \reals^P$ denotes a set of the graph filter coefficients with a finite $P$ dimension at all layers and $\ccalH$ denotes the set of all possible parameter sets.

\subsection{Manifold neural networks}\label{sec:mnn}
\paragraph{Setup}We consider a $d$-dimensional compact, smooth and differentiable Riemannian submanifold $\ccalM$ embedded in a $\mathsf{M}$-dimensional space $\reals^\mathsf{M}$ with finite volume. %\red{It is confusing to use $\reals^\mathsf{N}$, can we use $M,\reals^\mathsf{M}$ }. 
This induces a measure $\mu$ which has a non-vanishing Lipschitz continuous density $\rho$ with respect to the Riemannian volume over the manifold with $\rho:\ccalM\rightarrow(0,\infty)$, assumed to be bounded as $0<\rho_{min}\leq \rho(x) \leq \rho_{max}<\infty$ for all $x\in\ccalM$. The manifold data supported on each point $x\in\ccalM$ is defined by scalar functions $f:\ccalM\rightarrow \reals$ \citep{wang2024stability}. %In other words, we define manifold data in a point-wise fashion by attaching value $f(x)$ to the datum at each point $x \in \ccalM$. 
We use $L^2(\ccalM)$ to denote $L^2$ functions over $\ccalM$ with respect to measure $\mu$. The manifold with probability density function $\rho$ is equipped with a weighted Laplace operator \citep{grigor2006heat}, generalizing the Laplace-Beltrami operator as
\begin{equation}
    \label{eqn:weight-Laplace}
    \ccalL_\rho f = -\frac{1}{2\rho} \text{div}(\rho^2 \nabla f),
\end{equation}
with $\text{div}$ denoting the divergence operator of $\ccalM$ and $\nabla$ denoting the gradient operator of $\ccalM$ \citep{bronstein2017geometric, gross2023manifolds}. 
\paragraph{Manifold convolutions and frequency responses }The manifold convolution operation is defined relying on the Laplace operator $\ccalL_\rho$ and on the heat diffusion process over the manifold \citep{wang2024stability}. For a function $f\in L^2(\ccalM)$ as the initial heat condition over $\ccalM$, the heat condition diffused by a unit time step can be explicitly written as $e^{-\ccalL_\rho} f$. %Analogous to graph convolution, 
A manifold convolutional filter \citep{wang2024stability} can be defined in a diffuse-and-sum manner as  
\begin{align} \label{eqn:manifold-convolution}
   g(x) = \bbh(\ccalL_\rho)f(x) =\sum_{k=0}^{K-1} h_ke^{-k\ccalL_\rho}f(x) \text{,}
\end{align}
with the $k$-th diffusion scaled with a filter parameter $h_k\in\reals$. We consider the case in which the Laplace operator is self-adjoint, positive-semidefinite and the manifold $\ccalM$ is compact. In this case, $\ccalL_\rho$ has real, positive and discrete eigenvalues $\{\lambda_i\}_{i=1}^\infty$, written as $\ccalL_\rho \bm\phi_i =\lambda_i \bm\phi_i$ where $\bm\phi_i$ is the eigenfunction associated with eigenvalue $\lambda_i$. The eigenvalues are ordered in increasing order as $0=\lambda_1\leq \lambda_2\leq \lambda_3\leq \hdots$, and the eigenfunctions are orthonormal and form an eigenbasis of $L^2(\ccalM)$. When mapping a manifold signal onto the eigenbasis $ [\hat{f} ]_i=\langle f, \bm\phi_i\rangle_{ \ccalM} = \int_\ccalM f(x) \bm\phi_i(x)\text{d}\mu(x)$, the manifold convolution can be seen in the spectral domain as
\begin{align}
    [\hat{g}]_i = \sum_{k=0}^{K-1} h_k e^{-k\lambda_i}  [\hat{f}]_i\text{.}
\end{align}
Hence, the frequency response of manifold filter is given by $\hat{h}(\lambda)=\sum_{k=0}^{K-1} h_k e^{-k\lambda}$.
%, depending only on the filter coefficients $\{h_k\}_{k=0}^{K-1}$ and eigenvalues of $\ccalL_\rho$. 

\paragraph{Manifold neural networks } A manifold neural network (MNN) is constructed by cascading $L$ layers, each of which contains a bank of manifold convolutional filters and a pointwise nonlinearity $\sigma
:\reals \rightarrow\reals$. The output manifold function of each layer $l=1,2\cdots, L$ can be explicitly denoted as
\begin{equation}\label{eqn:mnn}
f_l^p(x) = \sigma\left( \sum_{q=1}^{F_{l-1}} \bbh_l^{pq}(\ccalL_\rho) f_{l-1}^q(x)\right),
\end{equation}
where $f_{l-1}^q$, $1 \leq q \leq F_{l-1}$ is the $q$-th input feature from layer $l-1$ and $f_l^p$, $1 \leq p \leq F_l$ is the $p$-th output feature of layer $l$. %In each layer, manifold convolutional filters map $F_{l-1}$ input features to $F_l$ output features. 
We denote MNN  as a mapping $\bbPhi(\bbH,\ccalL_\rho,f)$% for the ease of representation
, where $\bbH\in\ccalH\subset \reals^P$ is a collective set of filter parameters in all the manifold convolutional filters.

%% file: formulation.tex
% \section{Generalization gap of graph neural networks}
\section{Generalization analysis of GNNs based on manifolds}
% \subsection{Graphs sampled from manifolds}
We consider a manifold $\ccalM$ as defined in Section \ref{sec:mnn},
% .
% an embedded manifold $\ccalM\subset \reals^\mathsf{M}$ which is compact, smooth, and differentiable without boundary. The embedding induces a probability measure $\mu$ over the manifold with density function $\rho:\ccalM\rightarrow(0,\infty)$, which is assumed to be bounded as $0<\rho_{min}\leq \rho(x) \leq \rho_{max}<\infty$ for all $x\in\ccalM$. 
with a weighted Laplace operator $\ccalL_\rho$ as defined in \eqref{eqn:weight-Laplace}.
% (see  \eqref{eqn:weight-Laplace}), which is self-adjoint and positive-semidefinite. %Considering that manifold $\ccalM$ is compact without boundary, the operator $\ccalL_rho$ possesses a real discrete spectrum $\{\lambda_i\}_{i=1}^\infty$ with the eigenvalues $\lambda_i$ and the corresponding eigenfunctions $\phi_i$ satisfying $\ccalL \bm\phi_i =\lambda_i \bm\phi_i$. By projecting a manifold signal $f$ onto the eigenfunction, we can write the \emph{frequency representation} $\hat{f}$ as $[\hat{f}]_i= \langle f, \bm\phi_i \rangle_{ \ccalM }$ with the inner product defined in \eqref{eqn:innerproduct}. 
Since functions $f\in L^2(\ccalM)$ characterize information over manifold $\ccalM$,
we restrict our analysis to a finite-dimensional subset of $L^2(\ccalM)$ up to some eigenvalue of $\ccalL_\rho$, defined as a bandlimited signal.
% A bandlimited manifold signal is explicitly defined as follows.
\begin{definition}
\label{def:band}
      A manifold signal $f\in L^2(\ccalM)$ is bandlimited if there exists some $\lambda>0$ such that for all eigenpairs $\{\lambda_i, \bm\phi_i\}_{i=1}^\infty$ of the weighted Laplacian $\ccalL_\rho$ when $\lambda_i>\lambda$, we have $\langle f, \bm\phi_i \rangle_\ccalM = 0$.
\end{definition}

% The spectrum and eigenbasis of $\ccalL_\rho$ also help to view the manifold filter $\bbh(\ccalL_\rho)$ defined in \eqref{eqn:manifold-convolution} in the spectral domain. By projecting both the input and output manifold signals onto the eigenfunction $\bm\phi_i$, we can get
% \begin{equation}\label{eqn:projection}
%     [\hat{g}]_i = \sum_{k=0}^{K-1} h_k e^{-k\lambda_i}  [\hat{f}]_i \text{.}
% \end{equation}
% The function solely dependent on $\lambda_i$ is defined as the \emph{frequency response} of the filter $\bbh(\ccalL_rho)$, which can be stated formally as follows.

% %%%%%%%%%%%%%%%%%%%%%%%%%%%%%%%%%%%%%%%%%%%%%%%%
% %%%%%%%%%%%%%%%%%% DEFINITION %%%%%%%%%%%%%%%%%% 
% %%%%%%%%%%%%%%%%%%%%%%%%%%%%%%%%%%%%%%%%%%%%%%%%
% \begin{definition}[Frequency response of manifold filter]
% \label{def:frequency-response}
% The frequency response of the filter $\bbh(\ccalL)$ is given by
% \begin{equation}\label{eqn:operator-frequency}
% \hat{h}(\lambda)=\sum_{k=0}^{K-1} h_k e^{-k\lambda } \text{,}
% \end{equation}
% which leads \eqref{eqn:projection} to
% $
% [\hat{g}]_i = \hat{h}(\lambda_i)[\hat{f}]_i \text{.}$
% \end{definition}

Suppose we are given a set of $N$ i.i.d. randomly sampled points $X_N = \{x_i\}_{i=1}^N$ over $\ccalM$, with $x_i\in\ccalM$ sampled according to measure $\mu$. %The discrete measure is defined to be $\mu'=\sum_{i=1}^N \mu_i\delta_{x_i}$.
We construct a graph $\bbG(\ccalV,\ccalE,\ccalW)$ on these $N$ sampled points $X_N$, where each point $x_i$ is a vertex of graph $\bbG$, i.e. $\ccalV = X_N$. Each pair of vertices $(x_i,x_j)$ is connected with an edge while the weight attached to the edge $\ccalW(x_i,x_j)$ is determined by a kernel function $K_\epsilon$. The kernel function is decided by the Euclidean distance $\|x_i-x_j\|$ between these two points. The graph Laplacian denoted as $\bbL_N$ can be calculated based on the weight function \citep{merris1995survey}. The constructed graph Laplacian with an appropriate kernel function has been proved to approximate the Laplace operator $\ccalL_\rho$ of $\ccalM$ \citep{calder2022improved, belkin2008towards, dunson2021spectral}. We present the following two definitions of $K_\epsilon$. 

\begin{definition}[Gaussian kernel based graph \citep{belkin2008towards}]\label{def:gauss}
    The graph $\bbG(X_N, \ccalE, \ccalW)$ can be constructed in $\quad (x_i,x_j)\in\ccalE$, as a dense graph degree when the kernel function is defined as
\begin{align}
\label{eqn:gauss_kernel}
\ccalW(x_i,x_j) &= K_{\epsilon,1}\left(\frac{\|x_i-x_j\|^2}{\epsilon}\right)\\
&= \frac{1}{N}\frac{1}{\epsilon^{d/2+1}(4\pi)^{d/2}} e^{-\frac{\|x_i-x_j\|^2}{4\epsilon}}.
\end{align}
\end{definition} 
The weight function of a Gaussian kernel based graph is defined on unbounded support (i.e. $[0, \infty)$), which connects $x_i$ and $x_j$ regardless of the distance between them. This results in a dense graph with $N^2$ edges.
In particular, this Gaussian kernel based graph has been widely used to define the weight value function due to the good approximation properties of the corresponding graph Laplacians to the manifold Laplace operator \citep{dunson2021spectral, belkin2008towards, xie2013multiple}.

\begin{definition}[$\epsilon$-graph \citep{calder2022improved}]\label{def:epsilon}
The graph $\bbG(X_N, \ccalE, \ccalW)$ can be constructed as an $\epsilon$-graph with the kernel function defined as
\begin{align}
\label{eqn:compact_kernel}
\ccalW(x_i,x_j) &= K_{\epsilon,2}\left(\frac{\|x_i-x_j\|^2}{\epsilon}\right)\\
&= \frac{1}{N}\frac{d+2}{\epsilon^{d/2+1}\alpha_d} \mathbbm{1}_{[0,1]}\left(\frac{\|x_i-x_j\|^2}{\epsilon}\right),
\end{align}  
with $\quad (x_i,x_j)\in\ccalE$, where $\alpha_d$ is the volume of a unit ball of dimension $d$ and $\mathbbm{1}$ is the characteristic function.
\end{definition} 
The weight function of an $\epsilon$-graph is defined on a bounded support, i.e., only nodes that are within a certain distance of one another can be connected by an edge. It has also been shown to provide a good approximation of the manifold Laplace operator \citep{calder2022improved}. 
\begin{figure}
  \centering
        \input{figures/frequency_response_small.txt}
        \caption{Frequency response illustration}
        \label{fig:frequency_response}
\end{figure}

\subsection{Manifold label prediction via node label prediction}
%Let $\ccalM\subset \reals^\mathsf{N}$ be an embedded manifold with weighted Laplace operator $\ccalL_\rho$. 
Suppose we have an input manifold signal $f\in L^2(\ccalM)$ and a label (i.e. target) manifold signal $g\in L^2(\ccalM)$ over $\ccalM$. With an MNN $\bm\Phi(\bbH,\ccalL_\rho, \cdot)$, we predict the target value $g(x)$ based on input $f(x)$ at each point $x\in\ccalM$. By sampling $N$ points $X_N$ over this manifold, we can approximate this problem in a discrete graph domain. Consider a graph $\bbG(X_N,\ccalE,\ccalW)$ constructed with $X_N$ as either a Gaussian kernel based graph (Definition \ref{def:gauss}) or an $\epsilon$-graph (Definition \ref{def:epsilon}) equipped with the graph Laplacian $\bbL_N$. Suppose we are given graph signal $\{\bbx , \bby \}$ sampled from $\{f,g\}$ to train a GNN $\bm\Phi_\bbG(\bbH,\bbL_N,\cdot)$, explicitly written as 
\begin{equation}
    [\bbx]_i = f(x_i),\qquad [\bby]_i=g(x_i) \quad \text{for all }x_i \in X_N.
\end{equation}

We assume that the filters in MNN $\bm\Phi(\bbH,\ccalL_\rho, \cdot)$ and GNN $\bm\Phi_\bbG(\bbH,\bbL_N,\cdot)$ satisfy a continuity assumption as follows, which is illustrated in Figure \ref{fig:frequency_response}.
%are low-pass filters, which are defined explicitly as follows and illustrated in Figure \ref{fig:frequency_response}.

% Manifold neural networks (MNNs) and graph neural networks (GNNs), as defined in \eqref{eqn:mnn} and \eqref{eqn:gnn-eq} respectively, operate on manifold $\ccalM$ to process the information on a continuous domain and on graph $\bbG$ to process the information on a discrete domain. MNNs and GNNs share similar architectures, with the only difference in the operator implemented in the filters.
% %as shown in \eqref{eqn:manifold-convolution} and \eqref{eqn:graph_convolution}. 
% %We denote $M$ as the cardinality of the limited spectrum of $\ccalL_\rho$, i.e. $M =\# \{\lambda_i <\lambda_M\}$. 

 % \begin{figure} [h]  
 % \centering \input{figures/frequency_response} 
 % \caption{The $x$-axis stands for the spectrum, with each sample representing an eigenvalue. The black line illustrates a low-pass filter with red lines initiating the frequency responses of a filter on a given manifold. The blue dotted line shows the upper bound of the frequency response in Definition \ref{def:low-pass}.}
 %      \end{figure} 
 \begin{assumption}\label{ass:low-pass}
     The frequency response function of the filter satisfies 
\begin{equation}
        \label{eqn:low-pass}
        \left|\hat{h}(\lambda)\right| =\ccalO \left(\lambda^{-d}\right),\quad \left|\hat{h}'(\lambda)\right| \leq C_L \lambda^{-d-1},\quad \lambda\in (0,\infty),
    \end{equation}
    with $C_L$ a \emph{spectral continuity constant} that regularizes the smoothness of the filter function.
 \end{assumption}
% \begin{definition}\label{def:low-pass}
% A filter is a low-pass filter if its frequency response function satisfies 
% \begin{equation}
%         \label{eqn:low-pass}
%         \left|\hat{h}(\lambda)\right| =\ccalO \left(\lambda^{-d}\right),\quad \left|\hat{h}'(\lambda)\right| \leq C_L \lambda^{-d-1},\quad \lambda\in (0,\infty),
%     \end{equation}
%     with $C_L$ a \emph{spectral continuity constant} that regularizes the smoothness of the filter function.
% \end{definition}
\vspace{-2mm}
To introduce the first of our two main results, we require introducing two assumptions. 
%%%%%%%%%%%%%%%%%%%%%%%%%%%%%%%%%%%%%%%%%%%%%%%%
%%%%%%%%%%%%%%%%%% ASSUMPTION %%%%%%%%%%%%%%%%%% 
%%%%%%%%%%%%%%%%%%%%%%%%%%%%%%%%%%%%%%%%%%%%%%%%
\begin{assumption}(Normalized Lipschitz nonlinearity)\label{ass:activation}
 The nonlinearity $\sigma$ is normalized Lipschitz continuous, i.e., $|\sigma(a)-\sigma(b)|\leq |a-b|$, with $\sigma(0)=0$.
\end{assumption}
% We note that this assumption 

%A GNN trained on $\bbG$ is denoted as $\hat\bby = \bm\Phi(\bbH, \bbL_N, \bbx) \in \reals^N$, where $\bbH$ is the filter coefficient set, $\bbL_N \in\reals^{N\times N}$ is the graph Laplacian and $\bbx\in\reals^N$ is the input graph signal while $\bby\in\reals^N$ is the target output graph signal. 
% A loss function $\ell$ is employed to measure the prediction performance point-wisely. %between $[\bm\Phi(\bbH,\bbL_N,\bbx_i)]_i$ and $[\bby_i]_i$ point-wisely. 
%%%%%%%%%%%%%%%%%%%%%%%%%%%%%%%%%%%%%%%%%%%%%%%%
%%%%%%%%%%%%%%%%%% ASSUMPTION %%%%%%%%%%%%%%%%%% 
%%%%%%%%%%%%%%%%%%%%%%%%%%%%%%%%%%%%%%%%%%%%%%%%
\begin{assumption}(Normalized Lipschitz loss function)\label{ass:loss}
 The loss function $\ell$ is normalized Lipschitz continuous, i.e., $|\ell(y_i,y)-\ell(y_j,y)|\leq |y_i-y_j|$, with $\ell(y,y)=0$.
\end{assumption}
Assumption \ref{ass:activation} is satisfied by most activations used in practice such as ReLU, modulus and sigmoid. 
%For the loss function, we import an assumption on its continuity as normalized Lipschitz for the ease of presentation.
% The same filter parameter set can be transferred to a MNN by replacing the graph Laplacian as Laplace-Beltrami operator $\ccalL$ and the input signal as a manifold signal $f\in L^2(\ccalM)$. The target graph signal $\bby\in \reals^N$ is supposed to be sampled from a target manifold signal $g\in L^2(\ccalM)$, i.e. $\bby = \bbP_N g$.

The generalization gap is evaluated between the \emph{empirical risk} over the discrete graph model and the \emph{statistical risk} over manifold model, with the manifold model viewed as a statistical model since the expectation of the sampled point is with respect to the measure $\mu$ over the manifold. 
% Suppose we train the GNN on  input and target graph signals as $\{\bbx ,\bby \}$, where $\bbx, \bby \in \reals^N$ are sampled from $\{f$
%only has an entry in the $i$-th element, as
% \begin{equation}
%   [\bbx_i]_j =
%     \begin{cases}
%       f(x_i) & \text{if $i\leq j$}\\
%       0 & \text{otherwise}
%     \end{cases}       \qquad \qquad 
%      [\bby_i]_j =
%     \begin{cases}
%       g(x_i) & \text{if $i\leq j$}\\
%       0 & \text{otherwise.}
%     \end{cases}  
% \end{equation}
%\red{I know we discussed this, but it is not very intuitive. If this is the way we write the problem, there is no graph at all. We need to rewrite this part. }
The empirical risk over the sampled graph that we trained to minimize is therefore defined as 
\begin{align}
    \label{eqn:empirical-graphloss-node}
    R_\bbG(\bbH) =  \frac{1}{N}\sum_{i=1}^N\ell \left(  [\bm\Phi_\bbG(\bbH, \bbL_N, \bbx )]_i  ,  [\bby ]_i \right).
\end{align}
The statistical risk over the manifold is defined as 
\begin{align}
    \label{eqn:statistical-manifoldloss-node}
    R_\ccalM(\bbH) =\int_\ccalM \ell \left( \bm\Phi(\bbH, \ccalL_\rho, f)(x), g(x)\right) \text{d}\mu(x).
\end{align}
The generalization gap is defined to be
\begin{align}
    \label{eqn:generalization-gap}
    GA = \sup_{\bbH\in \ccalH} \left| R_\ccalM(\bbH) -   R_\bbG(\bbH)\right|.
\end{align}

\begin{theorem}
\label{thm:generalization-node-gauss}
    %Let $\ccalM\subset \reals^\mathsf{N}$ be an embedded manifold with weighted Laplace operator $\ccalL_\rho$. Let the input manifold signal $f\in L^2(\ccalM)$ and target manifold signal $g\in L^2(\ccalM)$ be functions over $\ccalM$. 
    %Consider a graph constructed with $N$ nodes sampled i.i.d. with measure $\mu$ over $\ccalM$ equipped with the graph Laplacian $\bbL_N$. Suppose we are given pairs of graph signals $\{\bbx_i, \bby_i\}_{i=1}^N$ sampled from $\{f,g\}$ to train a GNN $\bm\Phi(\bbH^*,\bbL_N,\cdot)$. Let $\bm\Phi(\bbH^\dagger,\ccalL_\rho,\cdot)$ be an MNN on $\ccalM$ \eqref{eqn:mnn} trained with input and output manifold functions $\{f, g\}$. 
    %Both the GNN and MNN contain low-pass filters and normalized Lipschitz nonlinearities in each layer. 
    Suppose the GNN and MNN with filters satisfying Assumption \ref{ass:low-pass} have $L$ layers with $F$ features in each layer and the input signal is bandlimited (Definition \ref{def:band}). Under Assumptions \ref{ass:activation} and \ref{ass:loss} it holds in probability at least $1-\delta$ that 
    \begin{align}
        GA\leq& F^L C_3 \left(\frac{\log N}{N}\right)^{\frac{1}{d}} \\
        &+ L F^{L-1}\left((C_1C_L+C_2) \sqrt{\frac{\epsilon }{{N}}}
     +  \frac{\pi^2\sqrt{\log(1/\delta)}}{6N}  \right) \nonumber
     ,
    \end{align}
when $d\geq 3$. If $d=2$, the first term would be $F^L C_3 \frac{(\log N) ^{3/4}}{N^{1/2}}$, with $C_1$, $C_2$, and $C_3$ depending on the geometry of $\ccalM$, $C_L$ is the spectral continuity constant in Assumption \ref{ass:low-pass}.
    \begin{enumerate}
    \item When the graph is constructed with a Gaussian kernel \eqref{eqn:gauss_kernel}, then $\epsilon \sim \left(\frac{\log(C/\delta)}{N} \right)^{\frac{2}{d+4}}$.
    %and the graph is dense with $\Theta(N)$ degree. 
    \vspace{-1mm}
    \item When the graph is constructed as an $\epsilon$-graph as \eqref{eqn:compact_kernel}, then $\epsilon \sim \left(\frac{\log(CN/\delta)}{N} \right)^{\frac{2}{d+4}}$.
    %and the graph is relatively sparse with $\Theta(\log N)$ degree.
    \end{enumerate}
\end{theorem}
\begin{proof}
    See Appendix \ref{app:proof-node} for proof and the definitions of $C_1$, $C_2$ and $C_3$.
\end{proof}

%\red{I would remove the "and the graph is dense with ..." in both Theorems, and both cases.}
\vspace{-1mm}

Theorem \ref{thm:generalization-node-gauss} shows that the generalization gap decreases approximately linearly with the number of nodes $N$ in the logarithmic scale, that is, $\log(GA) = \tilde \ccalO(-\log N)$ with $\tilde\ccalO$ as the $\ccalO$ notation that ignores logarithmic orders, and that it also increases with the dimension of the underlying manifold $d$. 
Another observation is that the generalization gap scales with the size of the GNN architecture. 
Most importantly, we note the bound increases linearly with the spectral continuity constant $C_L$ (Assumption \ref{ass:low-pass}) -- a smaller $C_L$ leads to a smaller generalization gap bound, and thus a better generalization capability. While a smaller $C_L$ leads to a smoother GNN, it discriminates fewer spectral components and, therefore, possesses worse discriminability. Consequently, we may observe a larger training loss with these smooth filters, as filters with worse discriminability encompass a smaller hypothesis function class and deteriorate the GNNs' approximation to the target functions during training. %which means we may end up with \red{a larger training loss} with these smooth filters. This shows an interesting trade-off between the generalization and discriminative capabilities. 
Since the testing loss can be upper bounded by the sum of training loss and the bound of generalization gap, on a smoother GNN (a smaller $C_L$), the performance on the training data will be closer to the performance on unseen testing data. 
Therefore, having a GNN with a smaller spectral continuity constant $C_L$ can guarantee more generalizable performance over unseen data from the same manifold. This also indicates that similar testing performance can be achieved by either a GNN with smaller training loss and worse generalization or a GNN with larger training loss and better generalization.
%However, there exists a trade-off between the discriminability and the generalization gap -- the smoother the GNN the smaller the generalization gap it will attain at the cost of a worse training performance given by the lack of discriminability. 
% and larger $N$ yields a larger training loss but a smaller generalization gap. 
In all, this indicates that there exists an optimal point to take the best advantage of the trade-off between a smaller generalization gap and better discriminability, resulting in a smaller testing loss decided by the spectral continuity constant of the GNN. 

\subsection{Manifold classification via graph classification }
Suppose we have a set of manifolds $\{\ccalM_k\}_{k=1}^K$, each of which is $d_k$-dimensional, smooth, compact, differentiable and embedded in $\reals^\mathsf{M}$ with measure $\mu_k$. Each manifold $\ccalM_k$ equipped with a weighted Laplace operator $\ccalL_{\rho_k,k}$ is labeled with $y_k \in \reals$. We assume to have access to $N_k$ randomly sampled points according to measure $\mu_k$ over each manifold $\ccalM_k$ and construct $K$ graphs $\{\bbG_k\}_{k=1}^K$ with graph Laplacians $\bbL_{N_k,k}$. The GNN $\bm\Phi_{\bbG_\cdot}(\bbH, \bbL_{N_\cdot,\cdot},\bbx_\cdot)$ is trained on this set of graphs with $\bbx_k$ as the input graph signal sampled from the manifold signal $f_k\in L^2(\ccalM_k)$ and $y_k\in \reals$ as the scalar target label. The final output of the GNN is set to be the average of the output signal values on each node while the output of MNN $\bm\Phi(\bbH,\ccalL_{\rho_\cdot,\cdot},f_\cdot)$ is the statistical average value of the output signal over the manifold. A loss function $\ell$ evaluates the difference between the output of GNN and MNN with the target label. 
%Suppose that GNN and MNN consist of low-pass filters (Definition \ref{def:low-pass}), normalized Lipschitz nonlinearities (Assumption \ref{ass:activation}) and normalized Lipschitz loss functions (Assumption \ref{ass:loss}). 
The empirical risk of the GNN is
\begin{align}
    \label{eqn:empirical-graphloss-graph}
    R_\bbG(\bbH) =  \sum_{k=1}^K \ell\left( \frac{1}{N_k}\sum_{i=1}^{N_k}  [\bm\Phi(\bbH, \bbL_{N_k,k}, \bbx_k)]_i  ,  y_k \right).
\end{align}

\begin{figure*}[h!]
    \centering
    % First row: Arxiv dataset
    \begin{subfigure}{\textwidth}
        \centering
        \includegraphics[width=\linewidth]{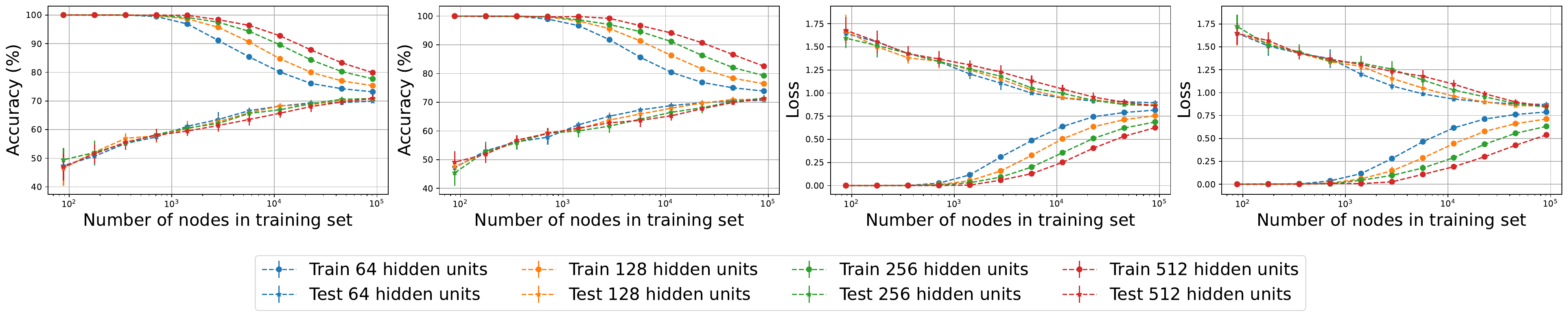}
    \end{subfigure}
    \begin{subfigure}{\textwidth}
        \centering
        \includegraphics[width=\linewidth]{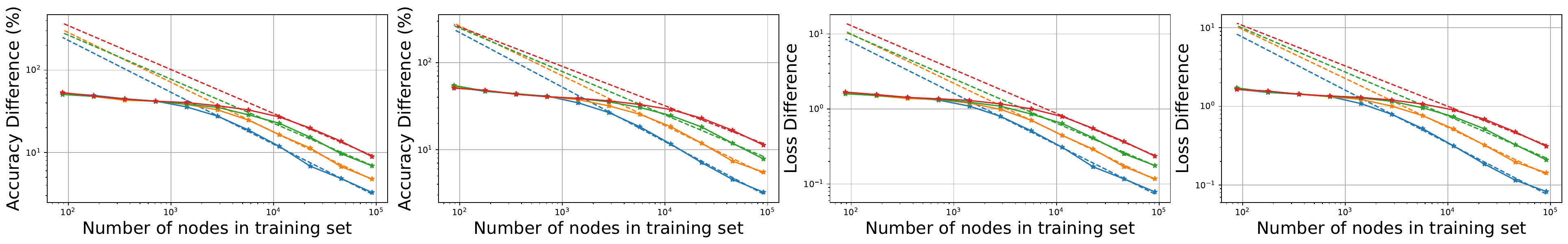}
    \end{subfigure}
    \begin{subfigure}{0.24\textwidth}
        \centering
        \caption{Arxiv: Two Layers}
        \label{fig:arxiv_acc_2_nodes}
    \end{subfigure}
    \begin{subfigure}{0.24\textwidth}
        \centering
        \caption{Arxiv: Three Layers}
        \label{fig:arxiv_acc_3_nodes}
    \end{subfigure}
    \begin{subfigure}{0.24\textwidth}
        \centering
        \caption{Arxiv: Two Layers (Loss)}
        \label{fig:arxiv_loss_2_nodes}
    \end{subfigure}
    \begin{subfigure}{0.24\textwidth}
        \centering
        \caption{Arxiv: Three Layers (Loss)}
        \label{fig:arxiv_loss_3_nodes}
    \end{subfigure}

    % Second row: Planetoid dataset
    \begin{subfigure}{\textwidth}
        \centering
        \includegraphics[width=\linewidth]{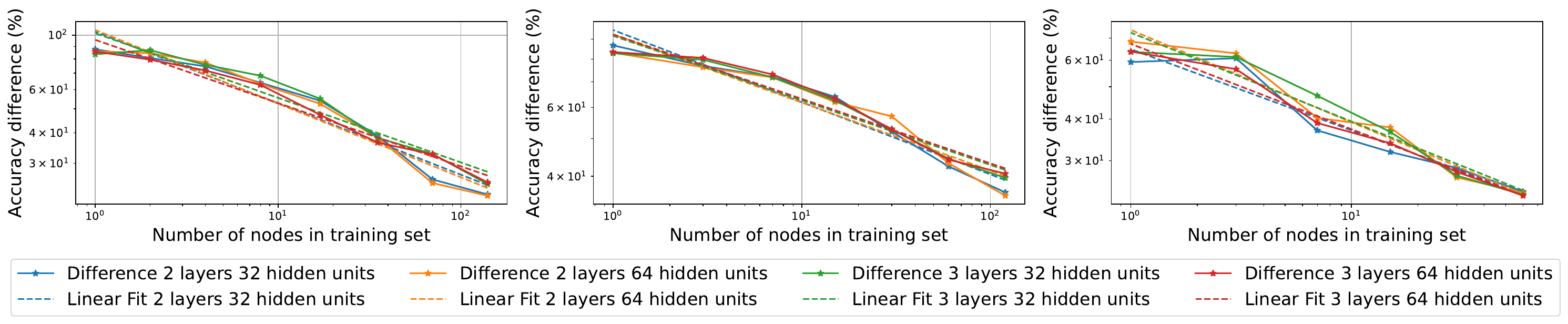}
    \end{subfigure}
    \begin{subfigure}{\textwidth}
        \centering
        \includegraphics[width=\linewidth]{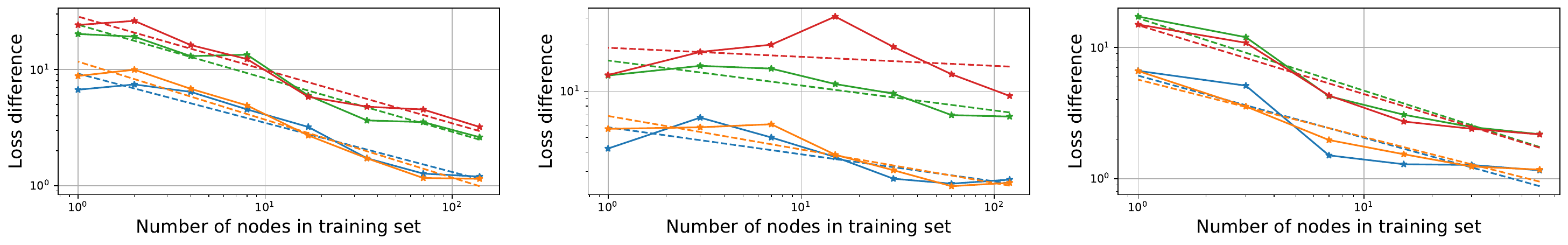}
    \end{subfigure}
    \begin{subfigure}{0.32\textwidth}
        \centering
        \caption{Cora}
        \label{fig:Cora}
    \end{subfigure}
    \begin{subfigure}{0.32\textwidth}
        \centering
        \caption{CiteSeer}
        \label{fig:CiteSeer}
    \end{subfigure}
    \begin{subfigure}{0.32\textwidth}
        \centering
        \caption{PubMed}
        \label{fig:pubmed}
    \end{subfigure}

    % Third row: Heterophilic and CoAuthors datasets
    \begin{subfigure}{\textwidth}
        \centering
        \includegraphics[width=\linewidth]{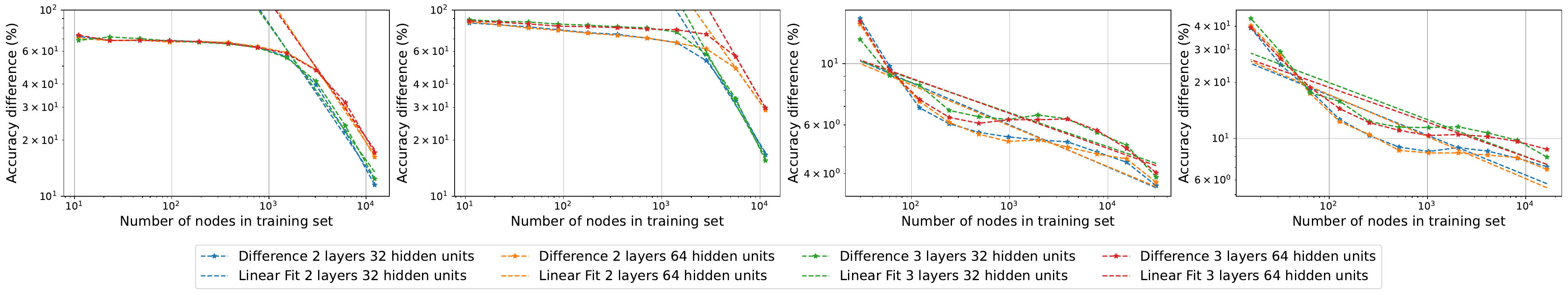}
    \end{subfigure}
    \begin{subfigure}{0.24\textwidth}
        \centering
        \caption{Amazon-Ratings}
        \label{fig:Amazon}
    \end{subfigure}
    \begin{subfigure}{0.24\textwidth}
        \centering
        \caption{Roman-Empire}
        \label{fig:roman}
    \end{subfigure}
    \begin{subfigure}{0.24\textwidth}
        \centering
        \caption{CoAuthors CS}
        \label{fig:coauthorsCS}
    \end{subfigure}
    \begin{subfigure}{0.24\textwidth}
        \centering
        \caption{CoAuthors Physics}
        \label{fig:coauthorsPhysics}
    \end{subfigure}
    \caption{Merged visualization of all datasets: Arxiv (top row), Planetoid (middle row), and Heterophilic and CoAuthors datasets (bottom row). Each row provides accuracy and loss generalization gaps across different configurations and datasets.}
    \label{fig:merged_all}
\end{figure*}

\begin{figure*}[h!]
    \centering
    \includegraphics[width=\linewidth]{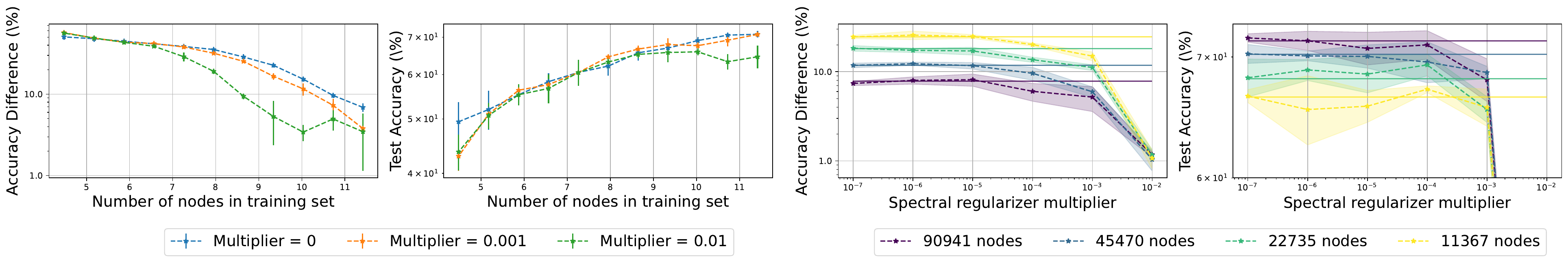}
    \begin{subfigure}{0.24\textwidth}
        \centering
        \caption{Accuracy gap vs nodes}
        \label{fig:gap_lip_accuracy_difference_vs_nodes}
    \end{subfigure}
    \begin{subfigure}{0.24\textwidth}
        \centering
        \caption{Test accuracy vs nodes}
        \label{fig:test_lip_accuracy_difference_vs_nodes}
    \end{subfigure}%
    \begin{subfigure}{0.24\textwidth}
        \centering
        \caption{Accuracy gap vs regularizer}
        \label{fig:gap_lip_accuracy_difference_vs_multiplier}
    \end{subfigure}
    \begin{subfigure}{0.24\textwidth}
        \centering
        \caption{Test accuracy vs regularizer}
        \label{fig:test_lip_accuracy_difference_vs_multiplier}
    \end{subfigure}
    \caption{Spectral continuity constant effect on generalization gap and test accuracy. }
    \label{fig:lipschitz}
\end{figure*}
While the output of MNN is the average value over the manifold, the statistical risk is defined based on the loss evaluated between the MNN output and the label as
\begin{align}
    \label{eqn:statistical-manifoldloss-graph}
    R_\ccalM(\bbH) = \sum_{k=1}^K\ell \left( \int_{\ccalM_k} \bm\Phi(\bbH, \ccalL_{\rho_k,k}, f_k)(x) \text{d}\mu_k(x), y_k\right).
\end{align}

The generalization gap is therefore
\begin{align}
    \label{eqn:generalization-gap-mani}
    GA = \sup_{\bbH\in \ccalH} \left| R_\ccalM(\bbH) -   R_\bbG(\bbH)\right|.
\end{align}
\begin{theorem}
\label{thm:generalization-graph-gauss}
    %Let $\ccalM_k\subset \reals^\mathsf{N}$ be an embedded manifold with weighted Laplace operator $\ccalL_{\rho,k}$ for each $k=1,2\cdots, K$ and $\lambda_M$-bandlimited manifold signals $f_k$ over each manifold. Consider $K$ graph constructed with $N$ nodes sampled i.i.d. with measure $\mu_k$ over $\ccalM$. The graph Laplacian $\bbL_{N,k}$ is calculated based on \eqref{eqn:gauss_kernel}. Suppose we are given pairs of graph signals $\{\bbx_k, \bby_k\}_{k=1}^K$ to train a GNN $\bm\Phi(\bbH,\bbL_N,\cdot)$. Let $\bm\Phi(\bbH,\ccalL_\rho,\cdot)$ be a MNN on $\ccalM$ \eqref{eqn:mnn} trained with input and output manifold functions $\{f_k, y_k\}$. 
    Suppose the GNN and MNN with filters satisfying Assumption \ref{ass:low-pass} have $L$ layers with $F$ features in each layer and the input signal is bandlimited (Definition \ref{def:band}). Under Assumptions \ref{ass:activation} and \ref{ass:loss} it holds in probability at least $1-\delta$ that
    \begin{align}
    GA\leq& LF^{L-1}\sum_{k=1}^K (C_1 C_L + C_2)\bigg(\sqrt{\frac{\epsilon_k}{{N_k}}}  \\
    &+\frac{\pi^2\sqrt{\log(1/\delta)}}{6N_k} \bigg) + F^L C_3 \sum
_{k=1}^K \left(\frac{\log N_k}{N_k}\right)^{\frac{1}{d_k}},\nonumber
     %    GA\leq LF^{L-1}\left(\frac{ {C_1} + C_2\theta_M^{-1}}{\sqrt{N}}\epsilon 
     % + C_3 \frac{\sqrt{\log(1/\delta)}}{N} + C_4 \frac{M^{-1}}{\sqrt{N}} + L \left(\frac{\log N}{N}\right)^{\frac{1}{d}}\right)
    \end{align}
when $d\geq 3$. If $d=2$, the last term would be $F^L C_3 \sum_{k=1}^K\frac{(\log N_k) ^{3/4}}{N_k^{1/2}}$,with $C_1$, $C_2$, and $C_3$ depending on the geometry of $\ccalM$, $C_L$ is the spectal continuity constant in Assumption \ref{ass:low-pass}.
     \begin{enumerate}
    \item When the graphs are constructed with a Gaussian kernel \eqref{eqn:gauss_kernel}, then $\epsilon_k \sim  \left(\frac{\log(C/\delta)}{N_k} \right)^{\frac{2}{d_k+4}}$.
    %and the graph is dense with $\Theta(N)$ degree. 
    \item When the graphs are constructed as $\epsilon$-graphs as \eqref{eqn:compact_kernel}, then $\epsilon_k \sim   \left(\frac{\log(CN_k/\delta)}{N_k} \right)^{\frac{2}{d_k+4}}$.
    %and the graph is relatively sparse with $\Theta(\log N)$ degree.
    \end{enumerate}
\end{theorem}
\begin{proof}
    See Appendix \ref{app:proof-graph} for proof and the definitions of $C_1$, $C_2$ and $C_3$.
\end{proof}

Theorem \ref{thm:generalization-graph-gauss} shows that a single graph sampled from the underlying manifold with large enough sampled points $N_k$ from each manifold $\ccalM_k$ can provide an effective approximation to classify the manifold itself. The generalization gap also attests that the trained GNN can generalize to classify other unseen graphs sampled from the same manifold. Similar to the generalization result in node-level tasks, the generalization gap decreases with the number of points sampled over each manifold while increasing with the manifold dimension. A higher dimensional manifold, i.e. higher complexity, needs more samples to guarantee the generalization. The generalization gap also shows a trade-off between the generalization and discriminability as the bound increases linearly with the spectral continuity constant $C_L$. That is, to guarantee that a GNN for graph classification can generalize effectively, we must impose restrictions on the continuity of its filter functions, which in turn limits the filters' ability to discriminate between different graph features.

We note that our assumption of a constant number of features can be generalized to include a different number of features in each layer for both node and graph classification.
%The generalization gap also scales with the size of GNN architecture polynomially with the number of features and exponentially with the number of layers.

% \begin{theorem}
% \label{thm:generalization-graph-sparse}
%     Let $\ccalM_k\subset \reals^\mathsf{N}$ be an embedded manifold with weighted Laplace operator $\ccalL_{\rho,k}$ for each $k=1,2\cdots, K$ and $\lambda_M$-bandlimited manifold signals $f_k$ over each manifold. Consider $K$ graph constructed with $N$ nodes sampled i.i.d. with measure $\mu_k$ over $\ccalM$. The graph Laplacian $\bbL_{N,k}$ is calculated based on \eqref{eqn:compact_kernel}. Suppose we are given pairs of graph signals $\{\bbx_k, \bby_k\}_{k=1}^K$ to train a GNN $\bm\Phi(\bbH,\bbL_N,\cdot)$. Let $\bm\Phi(\bbH,\ccalL_\rho,\cdot)$ be a single layer MNN on $\ccalM$ \eqref{eqn:mnn} trained with input and output manifold functions $\{f_k, y_k\}$. The GNN and MNN are single-layer structures containing low-pass filters and a normalized Lipschitz nonlinearity. It holds in probability at least $1-\delta$ that
%     \begin{equation}
%         GA\leq \frac{C_1}{\sqrt{N}}\left(\frac{\log(CN/\delta)}{N} \right)^{\frac{1}{d+4}}  +  \frac{C_2\theta_M^{-1}}{\sqrt{N}}   \left(\frac{\log(CN/\delta)}{N} \right)^{\frac{1}{d+4}}
%      + C_3 \frac{\sqrt{\log(1/\delta)}}{N} + C_4 \frac{M^{-1}}{\sqrt{N}} + L \left(\frac{\log N}{N}\right)^{\frac{1}{d}}
%     \end{equation}
% \end{theorem}

%% file: figures/frequency_response_small.txt
%!TEX root = ../../blackstoneSlides.tex

\def \thisplotscale {3} % Reduced scale factor for smaller plots
\def \unit {\thisplotscale cm}

\def \frequencyresponse 
 { 0.7*exp(-(0.8*(x-0.5))^2)+ 0.5* 1/((x+3)^2)+0.3*exp(-(0.8*(x-2))^2) +0.3*exp(-(0.8*(x))^2)
       }

\def \frequencyresponsebound 
 {0.8 * 1/(0.3*(x)^2)}

\resizebox{6cm}{!}{ % Reduced width while maintaining aspect ratio
\begin{tikzpicture}[x = 1*\unit, y=1*\unit]
\begin{axis}[scale only axis,
             width  = 2.4*\unit,
             height = 0.8*\unit,
             xmin = 0.1, xmax=7.5,
             xtick = { 0.1, 0, 2.63, 7.5},
             xticklabels = {0, \red{$\lambda_{1}$}, 
                            \red{$\lambda_{i}$}, 
                            $\qquad \lambda$},
             ymin = -0, ymax = 1.15,
             ytick = {1.15},
             yticklabels = {$\hat{h}(\lambda)$},
             enlarge x limits=false]
\node[color = blue] at (600, 0.2) {$\ccalO(\lambda^{-2})$};
\addplot+[samples at = {0.00, 0.91, 1.57, 
                        2.63}, 
          color = red!60, 
          ycomb, thick,
          mark=otimes*, 
          mark options={red!60}]
         {\frequencyresponse};

\addplot[ domain=0.1:7.5, 
          samples = 80, 
          color = black,
          line width = 1.2]
         {\frequencyresponse};
\addplot[ dashed, domain=0.1:7.5, 
          samples = 80, 
          color = blue,
          line width = 1.2]
         {\frequencyresponsebound};

\end{axis}
\end{tikzpicture}
}

%% file: simulations.tex
In this section, we empirically study the generalization gap in $8$ real-world datasets. The task is to predict the label of a node given a set of features. The datasets vary in the number of nodes from $169,343$ to $3,327$, and in the number of edges from $1,166,243$ to $9,104$. 
% The details of datasets can be found in the Appendix \ref{?}.
The feature dimension also varies from $8,415$ to $300$ features, and the number of classes of the node label from $40$ to $3$. 
In all cases, we vary the number of nodes in the training set by partitioning it in $\{1,2,4,8,16,32,64,32,64,128,256,512,1024\}$ partitions when possible. 
For both the training and testing sets, we computed the loss in \textit{cross-entropy loss}, and the accuracy in percentage ($\%$). 
% Juan Removed Spaces
Our main goal is to show that the rate presented in Theorem \ref{thm:generalization-node-gauss} holds in practice. 
%That is to say, if we plot the logarithm of the generalization gap as a function of the logarithm of the number of nodes we see a linear rate. To be consistent with the theory, we also want to show that if the number of layers or the size of the features increases, so does the generalization gap. 
% \input{tables/correlation}
% Juan Removed Spaces
In Figure \ref{fig:merged_all}, we plot the generalization gap of the accuracy in the logarithmic scale for a two-layered GNN (Figure \ref{fig:arxiv_acc_2_nodes}), and for a three-layered GNN (Figure \ref{fig:arxiv_acc_3_nodes}). On the upper side, we can see that the generalization bound decreases with the number of nodes and that outside of the strictly overfitting regime (when the training loss is below $95\%$), the generalization gap shows a linear decay, as depicted in the dashed line. The same behavior can be seen in Figures \ref{fig:arxiv_loss_2_nodes}, and \ref{fig:arxiv_loss_3_nodes} which correspond to the loss for $2$ and $3$ layered GNNs. As predicted by our theory, the generalization gap increases with the number of features and layers in the GNN. The behavior of the training and testing accuracy as a function of the number of nodes is intuitive. 
For the training loss, when the number of nodes in the training set is small, the GNN can overfit the training data. As the number of features increases, the GNN's capacity to overfit also increases.
%For the training loss, when the number of nodes is small, the GNN can overfit the training set -- the larger the number of features, the GNN can overfit more nodes.
% Juan Removed Spaces
In Figures \ref{fig:Cora} to \ref{fig:coauthorsPhysics}, we present the accuracy generalization gaps for $2$ and $3$ layers with $32$ and $64$ features. In the overfitting regime, the rate of our generalization bound seems to hold  -- decreases linearly with the number of nodes in the logarithmic scale. In the non-overfitting regime, our rate holds for the points whose training accuracy is below $95\%$. Also, we validate that the bound increases both with the number of features and the number of layers. 

% In Table \ref{tab:correlation}, we present the Pearson correlation coefficient to measure the linear relationships in the generalization gaps of a $2$ layers GNN with $64$ hidden units in all datasets considered. In almost every case, the coefficient is above $0.95$ which translates into a strong linear correlation. In Appendix \ref{sec:Appendix_Experiments} we explain how we computed these values.
% As seen in the experiment, the GNN generalization gap experiences a linear decay with respect to the number of nodes in the logarithmic scale. 
% Theorem \ref{thm:generalization-node-gauss} presents an upper bound on the generalization gap, whose rate can be seen to match the one seen in practice both for the loss, as well as the accuracy gaps.

% \paragraph{Spectral Continuity Constant Effect.}
To measure the impact of the spectral continuity constant $C_L$, we add a regularizer to the cross-entropy loss (see Appendix \ref{appendix_subsection:SCCR}). We vary the value of the regularizer, noting that a larger regularizer translates into a smaller $C_L$ and therefore a smoother function. In Figures \ref{fig:gap_lip_accuracy_difference_vs_nodes} and \ref{fig:gap_lip_accuracy_difference_vs_multiplier} we see the empirical manifestation of the bound that we showed (cf. Theorem \ref{thm:generalization-node-gauss}) -- a GNN with a smaller $C_L$ (a larger regularizer) will attain a smaller generalization gap.  
% GNN with a smaller spectral continuity constant (trained with a larger regularizer), translates into a smaller generalization gap. 
We can see that a larger regularizer (smaller continuity constant $C_L$, green line, regularizer $0.01$) attains a smaller generalization gap, and as the regularization decreases ($C_L$ increases), the generalization gap increases. 
% It must be noted that the regularization gap is the difference between train and test errors, and being small is not necessarily positive. 
The effect of having smaller spectral continuity constants $C_L$ is the lack of discriminability of the GNN.
As can be seen in Figures \ref{fig:test_lip_accuracy_difference_vs_nodes} and \ref{fig:test_lip_accuracy_difference_vs_multiplier}, the test error decreases when the multiplier is too large ($C_L$ too small). Therefore, a spectral regularize not too large can be shown to guarantee good test accuracy, but if the regularizer is too large, the test accuracy will be hurt by the lack of discriminability of the GNN as shown in Figure \ref{fig:test_lip_accuracy_difference_vs_multiplier}. In all, we verify the fact that a GNN with a smoother spectral response will have a smaller generalization gap as shown in Theorem \ref{thm:generalization-node-gauss}. 

% \paragraph{Graph classification} 

% We evaluate the generalization gap on graph prediction using the ModelNet10   dataset \citep{wu20153d}.  We set the coordinates of each point as input graph signals, and the weights of the edges are calculated based on the Euclidean distance between the nodes. The generalization gap is calculated by training GNNs on graphs with $N= 20, 40 \dots, 100$ sampled points, and plotting the differences between the average output of the trained GNNs on the trained graph and a testing graph with size $N=100$. Figure \ref{fig:diff_3layer} shows the generalization gaps for GNNs with 2 layers and Figure \ref{fig:diff_4layer} shows the results of GNNs with 3 layers. We can see that the output differences between the GNNs decrease with the number of nodes and decrease with the multiplier (increase with $C_L$). This verifies the claims of Theorem \ref{thm:generalization-graph-gauss}. In Appendix \ref{subsec:model}, we present experiment results on more model datasets.

%The architectures contain $F_0 = 3$ input features which are the 3-d coordinates of each point, $F_1= 64$ and $F_2=32$ features  in each layer. In GNN and Lipschitz GNN, ReLU is used as the nonlinearity function. The filters in Lipschitz GNN is regularized as Lipschitz continuous by imposing a penalty term $C_L h'(\lambda)$ to the loss function with $C_L$ set as 0.3. All architectures include a linear readout layer to map the final classification outputs. 

%% file: conclusion.tex
We study the statistical generalization of GNNs from a manifold perspective. We consider graphs sampled from manifolds and prove that GNNs could effectively generalize to unseen data from the manifolds when the number of sampled points is large enough and the filter functions are continuous in the spectral domain. We verify our theoretical results on both synthetic and real-world datasets. The impact of this paper is to show a better understanding of GNN generalization capabilities from a spectral perspective relying on a continuous model. Our work also motivates the practical design of large-scale GNNs. Specifically, in order to achieve a better generalization, it is essential to restrict the discriminability of GNNs by putting assumptions on the spectral continuity of the filter functions in the GNNs.
%as the generalization gap allows deriving a minimum number of nodes needed to achieve a satisfying generalization capability.  
% Juan Removed the future work. 
%For future work, we will study the generalization of GNNs in more settings include transductive learning and out-of-distribution generalization. We are also willing to look into more application scenarios to fully utilize our theory on more complex and general manifold models.
%We will consider a better explanation and exploration deep into the overfitting regime of node classification, which is of great interest where the figures show that our proposed generalization upper bounds fit the rate. 
%Furthermore, more real-world applications with large amount of datasets are required as we need to approximate a continuous manifold model.

% I have a MNN on $\bbM\reals^{P_M\times P_M}$
% \begin{align}
%     \bby_M=\sum_{n=0}^N h_n \bbM^n \bbx
% \end{align}

% Generated the data with an invented set of filters $h$. 
% For the GNN, I subsample $G\in \reals^{P_g\times P_g}$,
% \begin{align}
%     \bby_G=\sum_{n=0}^N h_n \bbG^n \bbx
% \end{align}
% What do I compare, the average 

% \red{Plot with number of nodes, and plot with Lipschitz.}

% \begin{align}
%    \frac{1}{P_M} |\bby_M| - \frac{1}{P_G}|\bby_G|
% \end{align}

%% file: appendix.tex
\onecolumn 

\tableofcontents
\newpage
% Juan added this for easier adaptation

\section{Induced manifold signals}
\label{sec:induced}
The graph signal attached to this constructed graph $\bbG$ can be seen as the discretization of the continuous function over the manifold. Suppose $f\in L^2(\ccalM)$, the graph signal $\bbx_N$ is composed of discrete data values of the function $f$ evaluated at $X_N$, i.e. $[\bbx_N]_i = f(x_i)$ for $ i=1,2\cdots, N$. With a sampling operator $\bbP_N:L^2(\ccalM)\rightarrow L^2(X_N)$, the discretization can be written as 
\begin{equation}
    \label{eqn:sampling}
    \bbx_N = \bbP_N f.
\end{equation}
Let $\mu_N$ be the empirical measure of the random sample as 
\begin{equation}
    \mu_N =\frac{1}{N} \sum_{i=1}^N \delta_{x_i}.
\end{equation}
Let $\{V_i\}_{i=1}^N$ be the decomposition \citep{garcia2020error} of 
$\ccalM$ with respect to $X_N$ with $V_i\subset B_r(x_i)$, 
%and $\text{vol}(V_i) =\mu_i=1/N$, 
where $B_r(x_i)$ denotes the closed metric ball of radius $r$ centered at $x_i\in\ccalM$ with respect to the Euclidean distance in the Euclidean ambient space. The decomposition can be achieved by the optimal transportation map $T:\ccalM\rightarrow X_N$, which is defined by the $\infty$-Optimal Transport distance between $\mu$ and $\mu_N$.
\begin{equation}
    d_\infty (\mu,\mu_N) := \min_{T:T_\# \mu =\mu_N} \text{esssup}_{x\in\ccalM} d(x,T(x)),
\end{equation}
where $T_\# \mu =\mu_N$ indicates that $\mu(T^{-1}(V))=\mu_N(V)$ for every $V_i$ of $\ccalM$. This transportation map $T$ induces the partition $V_1, V_2,\cdots V_N$ of $\ccalM$, where $V_i:= T^{-1}(\{x_i\})$ with $\mu(V_i)=\frac{1}{N}$ for all $i =1 ,\cdots N$.
% \begin{equation}
%     \label{eqn:voronoi}
%     V_i = \{x\in \ccalM: \text{dist}(x_i,x) \leq \text{dist}(x_j,x), j= 1,2\cdots i-1,i+1,\cdots N\},
% \end{equation}
%where $\text{dist}$ represents the curvature distance between two points on the manifold \citep{gross2023manifolds}.
% where $\text{dist}$ represents the curvature distance between two points on the manifold \citep{gross2023manifolds}. 
The radius of $V_i$ can be bounded as $r\leq A (\log N/N)^{1/d}$ when the manifold dimension $d\geq 3$ and $r\leq A (\log N)^{3/4}/N^{1/2}$ when $d=2$ with $A$ related to the geometry of $\ccalM$ \cite{garcia2020error}[Theorem 2]. 

The manifold function induced by the graph signal $\bbx_N$ over the sampled graph $\bbG$ is defined by 
\begin{equation}
    (\bbI_N \bbx_N) (x)  = \sum_{i=1}^N [\bbx]_i \mathbbm{1}_{x\in V_i}, \forall \; x \in \ccalM
\end{equation}
where we denote $\bbI_N: L^2(X_N)\rightarrow L^2(\ccalM)$ as the inducing operator.

\section{Convergence of GNN to MNN}
The convergence of GNN on sampled graphs to MNN provides the support for the generalization analysis. We first introduce the inner product over the manifold. The inner product of signals $f, g\in L^2(\ccalM)$ is defined as 
\begin{equation}\label{eqn:innerproduct}
    \langle f,g \rangle_{\ccalM}=\int_\ccalM f(x)g(x) \text{d}\mu(x), 
\end{equation}
where $\text{d}\mu(x)$ is the volume element with respect to the measure $\mu$ over $\ccalM$. Similarly, the norm of the manifold signal $f$ is
\begin{equation}\label{eqn:manifold_norm}
    \|f\|^2_{\ccalM}={\langle f,f \rangle_{\ccalM}}.
\end{equation}

\begin{proposition}
\label{prop:prob-diff}
    Let $\ccalM\subset \reals^\mathsf{M}$ be an embedded manifold with weighted Laplace operator $\ccalL_\rho$ and a bandlimited manifold signal $f$. Graph $\bbG_N$ is constructed based on a set of $N$ i.i.d. randomly sampled points $X_N=\{x_1, x_2,\cdots, x_{N}\}$ according to measure $\mu$ over $\ccalM$. A graph signal $\bbx$ is the sampled manifold function values at $X_N$. The graph Laplacian $\bbL_N$ is calculated based on \eqref{eqn:gauss_kernel} or \eqref{eqn:compact_kernel} with $\epsilon$ as the graph parameter. Let $\bm\Phi(\bbH,\ccalL_\rho,\cdot)$ be a MNN on $\ccalM$ \eqref{eqn:mnn} with $L$ layers and $F$ features in each layer. Let $\bm\Phi_\bbG(\bbH,\bbL_N,\cdot)$ be the GNN with the same architecture applied to the graph $\bbG_N$. Then, with the filters satisfy Assumption \ref{ass:low-pass} and nonlinearities as normalized Lipschitz continuous, it holds in probability at least $1-\delta$ that 
    \begin{align}
         \label{eqn:prob-diff}
    \frac{1}{N}\sum_{i=1}^N \| \bm\Phi_\bbG(\bbH,\bbL_N,\bbx) - \bbP_N {\bm\Phi}(\bbH,\ccalL_\rho,\bbI_N \bbx) \|_2 \leq   LF^{L-1}\left(C_1 \sqrt{\epsilon}  + C_2 \sqrt{\frac{\log(1/\delta)}{N}}\right)
    \end{align}
    where $C_1 $, $C_2 $ are constants defined in the following proof.
\end{proposition}

% \begin{align}
%     \frac{1}{N}\sum_{i=1}^N \| \bm\Phi(\bbH,\bbL_N,\bbx) - \bbP_N {\bm\Phi}(\bbH,\ccalL_\rho,\bbI_N \bbx) \|_2 \leq   C_1 \epsilon  +  C_2\epsilon  \theta_M^{-1} 
%      + C_3 \sqrt{\frac{\log(1/\delta)}{N}} + C_4 M^{-1}
% \end{align}
\begin{proposition}{\cite{wang2023geometric}[Proposition~2, Proposition~4]}\label{thm:converge-spectrum-sparse}
Let $\ccalM\subset \reals^\mathsf{M}$ be equipped with Laplace operator $\ccalL_\rho$, whose eigendecomposition is given by $\{\lambda_i,\bm\phi_i\}_{i=1}^\infty$. Let $\bbL_N$ be the discrete graph Laplacian of graph weights defined as \eqref{eqn:gauss_kernel} (or \eqref{eqn:compact_kernel}), with spectrum $\{\lambda_{i,N},\bm\phi_{i,N}\}_{i=1}^N$.
Fix $K\in \mathbb{N}^+$ and assume that  $\epsilon= \epsilon(N) \geq \left({\log(C/\delta)}/{N}\right)^{2/(d+4)}$ (or $\epsilon =\epsilon(N)\geq\left({\log(CN/\delta)}/{N}\right)^{2/(d+4)}$).  
Then, with probability at least $1-\delta$, we have 
\begin{equation}
    |\lambda_i-\lambda_{i,N} |\leq C_{\ccalM,1} \lambda_i\sqrt{\epsilon},  \quad 
    \|a_i\bm\phi_{i,N} -\bm\phi_i\|\leq C_{\ccalM,2} \frac{\lambda_i}{\theta_i} \sqrt{\epsilon} ,
\end{equation}
with $a_i\in\{-1,1\}$ for all $i<K$ and $\theta$ the eigengap of $\ccalL$, i.e., $\theta_i=\min\{\lambda_i-\lambda_{i-1},\lambda_{i+1}-\lambda_{i}\}$. The constants $ C_{\ccalM,1}$, $ C_{\ccalM,2}$ depend on $d$ and the volume, the injectivity radius and sectional curvature of $\ccalM$.
\end{proposition}
\begin{proof}

Because $\{x_1, x_2,\cdots,x_N\}$ is a set of randomly sampled points from $\ccalM$, based on Theorem 19 in \cite{von2008consistency} we can claim that
\begin{equation}
   \left|\langle \bbP_N f,\bbP_N \bm\phi_{i } \rangle   -\langle f,\bm\phi_i\rangle_\ccalM\right| = O\left(\sqrt{\frac{\log (1/\delta)}{N}}\right).
\end{equation}
This also indicates that 
\begin{equation}
   \left|\|\bbP_N f \|^2-\|f\|^2_\ccalM\right| = O\left(\sqrt{\frac{\log (1/\delta)}{N}}\right),
\end{equation}
which indicates $\|\bbP_N f\|=\|f\|_\ccalM + O((\log (1/\delta)/N)^{1/4})$.
We suppose that the input manifold signal is $\lambda_M$-bandlimited with $M$ spectral components. We first write out the filter representation as
 \begin{align}
    & \|\bbh(\bbL_N)\bbP_N f - \bbP_N \bbh(\ccalL_\rho) f\| = \Bigg\| \sum_{i=1}^N \hat{h}(\lambda_{i,N}) \langle \bbP_N f,\bm\phi_{i,N} \rangle \bm\phi_{i,N}  - \sum_{i=1}^M \hat{h}(\lambda_i)\langle f,\bm\phi_i\rangle_{\ccalM} \bbP_N \bm\phi_i  \Bigg\|
     \\ 
     &  \leq  \Bigg\| \sum_{i=1}^M \hat{h}(\lambda_{i,N}) \langle \bbP_N f,\bm\phi_{i,N} \rangle \bm\phi_{i,N} - \sum_{i=1}^M \hat{h}(\lambda_i) \langle f,\bm\phi_{i}  \rangle_{\ccalM} \bbP_N\bm\phi_{i} + \sum_{i=M+1}^N \hat{h}(\lambda_{i,N}) \langle \bbP_N f,\bm\phi_{i,N} \rangle \bm\phi_{i,N} \Bigg\| \\
     & \leq \Bigg\| \sum_{i=1}^M \hat{h}(\lambda_{i,N}) \langle \bbP_N f,\bm\phi_{i,N} \rangle \bm\phi_{i,N} - \sum_{i=1}^M \hat{h}(\lambda_i) \langle f,\bm\phi_{i}  \rangle_{\ccalM}\bbP_N\bm\phi_{i}\Bigg\|   + \left\|\sum_{i=M+1}^N \hat{h}(\lambda_{i,N}) \langle \bbP_N f,\bm\phi_{i,N} \rangle \bm\phi_{i,N}\right\|\label{eqn:1}
    %  & \qquad \qquad \qquad \qquad\qquad \qquad+\left\| \sum_{i=1}^M \hat{h}(\lambda_i) \langle \bbP_n f,\bm\phi_{i,n}^\epsilon \rangle_{\bbG_n} \bm\phi_{i,n}^\epsilon - \sum_{i=1}^M \hat{h}(\lambda_i) \langle f,\bm\phi_i \rangle_{\ccalM} \bbP_n \bm\phi_i \right\|.\label{eqn:conv-1}
 \end{align}
 The first part of \eqref{eqn:1} can be decomposed with the triangle inequality as 
 \begin{align}
     & \nonumber \left\| \sum_{i=1}^M \hat{h}(\lambda_{i,N}) \langle \bbP_N f,\bm\phi_{i,N}  \rangle \bm\phi_{i,N}  - \sum_{i=1}^M \hat{h}(\lambda_i)\langle f,\bm\phi_i\rangle_{\ccalM} \bbP_N \bm\phi_i  \right\|
     \\ 
     &  \leq  \left\| \sum_{i=1}^M \left(\hat{h}(\lambda_{i,N})- \hat{h}(\lambda_i) \right) \langle \bbP_Nf,\bm\phi_{i,N} \rangle \bm\phi_{i,N} \right\|   +\left\| \sum_{i=1}^M \hat{h}(\lambda_i)\left( \langle \bbP_N f,\bm\phi_{i,N} \rangle  \bm\phi_{i,N} - \langle f,\bm\phi_i \rangle_{\ccalM} \bbP_N \bm\phi_i \right)  \right\|.\label{eqn:conv-1}
 \end{align}
 In \eqref{eqn:conv-1}, the first part relies on the difference of eigenvalues and the second part depends on the eigenvector difference. %We import the spectral differences from \cite{calder2022improved} as follows.
The first term in \eqref{eqn:conv-1} is bounded with Cauchy-Schwartz inequality as 
\begin{align}
     \Bigg\| \sum_{i=1}^M (\hat{h}(\lambda_{i,n} ) - \hat{h}(\lambda_i)) \langle \bbP_N f,\bm\phi_{i,N} \rangle & \bm\phi_{i,N}  \Bigg\|   \leq \sum_{i=1}^M \left|\hat{h}(\lambda_{i,N} )-\hat{h}(\lambda_i)\right| |\langle \bbP_N f,\bm\phi_{i,N}  \rangle | \\
     &\leq \|\bbP_N f\| \sum_{i=1}^M |\hat{h}'(\lambda_i)| |\lambda_{i,N} -\lambda_i|\\
   &\leq \|\bbP_N f\|  \sum_{i=1}^M C_{\ccalM,1}C_L \sqrt{\epsilon} \lambda_i^{-d} \label{eqn:p2-1}\\
   &\leq  \|\bbP_N f\|C_L C_{\ccalM,1}\sqrt{\epsilon}  \sum_{i=1}^M  i^{-2} \label{eqn:p2}\\
   &\leq \left( \|f\|_\ccalM+ \left(\frac{\log (1/\delta)}{N}\right)^{\frac{1}{4}}\right)C_{\ccalM,1}\sqrt{\epsilon} \frac{\pi^2}{6} := A_1(N)
\end{align} 
In \eqref{eqn:p2-1}, it depends on the filter assumption in Assumption \ref{ass:low-pass}.
In \eqref{eqn:p2}, we implement Weyl's law \citep{arendt2009weyl} which indicates that eigenvalues of Laplace operator scales with the order  $\lambda_i \sim i^{2/d}$. The last inequality comes from the fact that $\sum_{i=1}^\infty i^{-2}=\frac{\pi^2}{6}$.
% The last step depends on the Lipschitz continuity of the frequency response function as well as the eigenvalue gap in Proposition \ref{thm:converge-spectrum-sparse}. 
% This indicates that 
% \begin{align}
%  &\sum_{i=1}^M  (2+C_L C_{\ccalM,1}\epsilon )\left|\hat{h}(\lambda_i)\right|\leq (2+C_LC_{\ccalM,1}\epsilon) \sum_{i=1}^M\lambda_i^{-d}
% \\& \leq (2+C_LC_{\ccalM,1}\epsilon) \sum_{i=1}^M i^{-2} \leq (2+C_LC_{\ccalM,1}\epsilon) M^{-1}:=A_1(M)
%   % &\nonumber \left\| \sum_{i=1}^M (\hat{h}
%   % (\lambda_{i,n} ) - \hat{h}(\lambda_i)) \langle \bbP_N f,\bm\phi_{i,N} \rangle  \bm\phi_{i,N}  \right\|  \\
%   %  &\leq C C_{k,1}\epsilon \left(\|f\|_\ccalM+\left(\frac{\log (1/\delta)}{N}\right)^{\frac{1}{4}}\right) \sqrt{\sum_{i=1}^M  \frac{ 1 }{\min\{\lambda_i^2,\lambda_{i,N}^2\}}}\\
%   %  &:= A_1(M,N)
% \end{align} 
The second term in \eqref{eqn:conv-1} can be bounded with the triangle inequality  as
{ 
\begin{align}
  & \nonumber \Bigg\| \sum_{i=1}^M \hat{h}(\lambda_i)\left( \langle \bbP_Nf,\bm\phi_{i,N}  \rangle \bm\phi_{i,N}- \langle f,\bm\phi_i \rangle_{\ccalM} \bbP_N \bm\phi_i\right)  \Bigg\|\\
   & \leq \nonumber \Bigg\|  \sum_{i=1}^M \hat{h}(\lambda_i)  \left(\langle \bbP_N f,\bm\phi_{i,N} \rangle \bm\phi_{i,N}   - \langle \bbP_Nf,\bm\phi_{i,N}  \rangle  \bbP_N\bm\phi_i\right)\Bigg\|\\
   &\label{eqn:term1} \qquad \qquad \qquad + \left\| \sum_{i=1}^M  \hat{h}(\lambda_i) \left(\langle \bbP_N f,\bm\phi_{i,N} \rangle  \bbP_N\bm\phi_i -\langle f,\bm\phi_i\rangle_\ccalM \bbP_N\bm\phi_i \right) \right\|
\end{align}}
The first term in \eqref{eqn:term1} can be bounded with inserting the eigenfunction convergence result in Proposition \ref{thm:converge-spectrum-sparse} as
\begin{align}
& \nonumber \left\|  \sum_{i=1}^M \hat{h}(\lambda_i) \left(\langle \bbP_N f,\bm\phi_{i,N} \rangle \bm\phi_{i,N}  - \langle \bbP_Nf,\bm\phi_{i,N} \rangle_{\ccalM} \bbP_N\bm\phi_i\right)\right\|\\
& \leq \sum_{i=1}^{M} \left|\hat{h}(\lambda_i)\right|\|\bbP_N f\|\|\bm\phi_{i,N} - \bbP_N\bm\phi_i\|\\
&\leq \sum_{i=1}^M (\lambda_i^{-d+1}) \frac{C_{\ccalM,2}\sqrt{\epsilon}}{\theta_i}\left(\|f\|_\ccalM + \left(\frac{\log (1/\delta)}{N}\right)^{\frac{1}{4}}\right)\\
&\leq \sum_{i=1}^M (\lambda_i^{-d+1}) \max_{i=1,\cdots,M}{\theta_i^{-1}}C_{\ccalM,2}\sqrt{\epsilon} \left(\|f\|_\ccalM + \left(\frac{\log (1/\delta)}{N}\right)^{\frac{1}{4}}\right)\\
&:= A_2(M,N).
\end{align}
Considering the filter assumption in Assumption \ref{ass:low-pass}, the second term in \eqref{eqn:term1} can be written as
\begin{align}
     & \nonumber \Bigg\| \sum_{i=1}^M \hat{h}(\lambda_{i,N} ) (\langle \bbP_N f,\bm\phi_{i,N}\rangle  \bbP_N\bm\phi_i -\langle f,\bm\phi_i\rangle_\ccalM \bbP_N\bm\phi_i ) \Bigg\| \\
   &\leq \sum_{i=1}^M \left|\hat{h}(\lambda_{i,N}) \right| \left|\langle \bbP_N f,\bm\phi_{i,N} \rangle   -\langle f,\bm\phi_i\rangle_\ccalM\right|\|\bbP_N\bm\phi_i\|\\
   &\leq \sum_{i=1}^M (\lambda_{i,N}^{-d})\left|\langle \bbP_N f,\bm\phi_{i,N} \rangle   -\langle f,\bm\phi_i\rangle_\ccalM\right| \left(1+\left(\frac{\log (1/\delta)}{N}\right)^{\frac{1}{4}}\right)\\
   &\label{eqn:A3}\leq 
   \sum_{i=1}^M (1 +C_{\ccalM,1}\sqrt{\epsilon})^{-d}  (\lambda_i^{-d})\left|\langle \bbP_N f,\bm\phi_{i,N} \rangle   -\langle f,\bm\phi_i\rangle_\ccalM\right| \left(1+\left(\frac{\log (1/\delta)}{N}\right)^{\frac{1}{4}}\right) \\
   &\leq \frac{\pi^2}{6}\left|\langle \bbP_N f,\bm\phi_{i,N} \rangle   -\langle f,\bm\phi_i\rangle_\ccalM\right| \left(1+\left(\frac{\log (1/\delta)}{N}\right)^{\frac{1}{4}}\right):=A_3(N)
\end{align}
The term $\left|\langle \bbP_N f,\bm\phi_{i,N} \rangle   -\langle f,\bm\phi_i\rangle_\ccalM\right| $ can be decomposed by inserting a term $\langle \bbP_N f, \bbP_N \bm\phi_i\rangle $ as
\begin{align}
    \left|\langle \bbP_N f,\bm\phi_{i,N} \rangle   -\langle f,\bm\phi_i\rangle_\ccalM\right| &\leq  \left|\langle \bbP_N f,\bm\phi_{i,N} \rangle  - \langle \bbP_N f, \bbP_N \bm\phi_i\rangle + \langle \bbP_N f, \bbP_N \bm\phi_i\rangle  -\langle f,\bm\phi_i\rangle_\ccalM\right| \\
   & \leq   \left|\langle \bbP_N f,\bm\phi_{i,N} \rangle  - \langle \bbP_N f, \bbP_N \bm\phi_i\rangle\right| + \left|\langle \bbP_N f, \bbP_N \bm\phi_i\rangle  -\langle f,\bm\phi_i\rangle_\ccalM\right|\\
   & \leq \|\bbP_N f\|\| \bm\phi_{i,N} - \bbP_N\bm\phi_i\|+ \left|\langle \bbP_N f, \bbP_N \bm\phi_i\rangle  -\langle f,\bm\phi_i\rangle_\ccalM\right|\\
   &\leq \left(\|f\|_\ccalM + \left(\frac{\log (1/\delta)}{N}\right)^{\frac{1}{4}}\right) \frac{C_{\ccalM,2}\lambda_i \sqrt{\epsilon}}{\theta_i} + \sqrt{\frac{\log (1/\delta)}{N}}
\end{align}
Then equation \eqref{eqn:A3} can be bounded as 
\begin{align}
     & \nonumber \Bigg\| \sum_{i=1}^M \hat{h}(\lambda_{i,N} ) (\langle \bbP_N f,\bm\phi_{i,N}\rangle  \bbP_N\bm\phi_i -\langle f,\bm\phi_i\rangle_\ccalM \bbP_N\bm\phi_i ) \Bigg\| \\
     &\leq 
   \sum_{i=1}^M (1 +C_{\ccalM,1}\sqrt{\epsilon})^{-d}  (\lambda_i^{-d}) \left(\left(\|f\|_\ccalM + \left(\frac{\log (1/\delta)}{N}\right)^{\frac{1}{4}}\right) \frac{C_{\ccalM,2}\lambda_i \sqrt{\epsilon}}{\theta_i} + \sqrt{\frac{\log (1/\delta)}{N}}\right) \left(1+\left(\frac{\log (1/\delta)}{N}\right)^{\frac{1}{4}}\right)
   \\&\leq \frac{\pi^2}{6} \max_{i=1,\cdots,M}\frac{C_{\ccalM,2} \sqrt{\epsilon}}{\theta_i}\left(\|f\|_\ccalM + \left(\frac{\log (1/\delta)}{N}\right)^{\frac{1}{4}}\right) + \frac{\pi^2}{6} \sqrt{\frac{\log (1/\delta)}{N}}
\end{align}

The second term in \eqref{eqn:1} can be bounded with the eigenvalue difference bound in Proposition \ref{thm:converge-spectrum-sparse} as
\begin{align}
      \left\|\sum_{i=M+1}^N \hat{h}(\lambda_{i,N}) \langle \bbP_N f,\bm\phi_{i,N} \rangle \bm\phi_{i,N}\right\|&\leq \sum_{i=M+1}^N (\lambda_{i,N}^{-d})\left(\|f\|_\ccalM+\left(\frac{\log (1/\delta)}{N}\right)^{\frac{1}{4}}\right)\\
    &\leq \sum_{i=M+1}^\infty (\lambda_{i,N}^{-d})  \|f\|_\ccalM
    \\&\leq  (1 +C_{\ccalM,1}\sqrt{\epsilon})^{-d} \sum_{i=M+1}^\infty (\lambda_i^{-d})  \|f\|_\ccalM\\
    &\leq  M^{-1}\|f\|_\ccalM:= A_4(M).
\end{align}

We note that the bound is made up by terms $A_1( N)+A_2(M,N)+A_3( N)+A_4(M)$, related to the bandwidth of manifold signal $M$ and the number of sampled points $N$. %As $\epsilon$ scales with the order $\left(\frac{\log(CN/\delta)}{N} \right)^{\frac{1}{d+4}}$. 
This makes the bound scale with the order
{ 
\begin{align}
\label{eqn:filter-bound}
    &  \|\bbh(\bbL_N)\bbP_N f - \bbP_N \bbh(\ccalL_\rho) f\|   \leq C_1' \sqrt{\epsilon} +  C_2' \sqrt{\epsilon} \theta_M^{-1}  + C_3' \sqrt{\frac{\log(1/\delta)}{N}} + C_4' M^{-1},
\end{align}}
with $C_1' = C_LC_{\ccalM,1}\frac{\pi^2}{6}\|f\|_\ccalM$, $C_2' = C_{\ccalM,2}\frac{\pi^2}{6}$, $C_3' =\frac{\pi^2}{6}$ and $C_4' = \|f\|_\ccalM$.
As $N$ goes to infinity, for every $\delta >0$, there exists some $M_0$, such that for all $M>M_0$ it holds that $A_4(M)\leq \delta/2$. There also exists $n_0$, such that for all $N>n_0$, it holds that $A_1(N)+A_2(M_0, N)+A_3(N)\leq \delta/2$. We can conclude that the summations converge as $N$ goes to infinity. We see $M$ large enough to have $M^{-1}\leq \delta'$, which makes the eigengap $\theta_M$ also bounded by some constant. We combine the first two terms as 
\begin{equation}
     \|\bbh(\bbL_N)\bbP_N f - \bbP_N \bbh(\ccalL_\rho) f\|   \leq (C_1C_L+C_2) \sqrt{\epsilon}  +\frac{\pi^2}{6} \sqrt{\frac{\log(1/\delta)}{N}},
\end{equation}
with $C_1 = C_{\ccalM,1} \frac{\pi^2}{6}\|f\|_\ccalM$ and $C_2 =  C_{\ccalM,2}\frac{\pi^2}{6} \theta^{-1}_{\delta'^{-1}}$.
To bound the output difference of MNNs, we need to write in the form of features of the final layer
 \begin{align}
     \|\bm\Phi_\bbG(\bbH,\bbL_N,\bbP_N f)-\bbP_N \bm\Phi&(\bbH,\ccalL_\rho, f))\| = \left\| \sum_{q=1}^{F }\bbx_{n,L}^q-\sum_{q=1}^{F }\bbP_N f_L^q \right\| \leq \sum_{q=1}^{F } \left\| \bbx_{n,L}^q- \bbP_N f_L^q \right\|.
 \end{align}
By inserting the definitions, we have 
 \begin{align}
    \left\| \bbx_{n,l}^p- \bbP_N f_l^p \right\| =\left\| \sigma\left(\sum_{q=1}^{F } \bbh_l^{pq}(\bbL_N) \bbx_{n,l-1}^q \right) -\bbP_N \sigma\left(\sum_{q=1}^{F } \bbh_l^{pq}(\ccalL_\rho) f_{l-1}^q\right) \right\|
 \end{align}
 with $\bbx_{n,0}=\bbP_N f$ as the input of the first layer. With a normalized point-wise Lipschitz nonlinearity, we have
  \begin{align}
    \| \bbx_{n,l}^p - \bbP_N f_l^p & \| \leq \left\|  \sum_{q=1}^{F } \bbh_l^{pq}(\bbL_N) \bbx_{n,l-1}^q    - \bbP_N \sum_{q=1}^{F } \bbh_l^{pq}(\ccalL_\rho)  f_{l-1}^q\right\|\\
    & \leq \sum_{q=1}^{F } \left\|    \bbh_l^{pq}(\bbL_N) \bbx_{n,l-1}^q    - \bbP_N   \bbh_l^{pq}(\ccalL_\rho)  f_{l-1}^q\right\|
 \end{align}
 The difference can be further decomposed as
\begin{align}
   \nonumber  & \|    \bbh_l^{pq}(\bbL_N)  \bbx_{n,l-1}^q    - \bbP_N   \bbh_l^{pq}(\ccalL_\rho)  f_{l-1}^q \| 
   \\  &\leq \|
\bbh_l^{pq}(\bbL_N) \bbx_{n,l-1}^q  - \bbh_l^{pq}(\bbL_N) \bbP_N f_{l-1}^q   +\bbh_l^{pq}(\bbL_N) \bbP_N f_{l-1}^q  - \bbP_N   \bbh_l^{pq}(\ccalL_\rho)  f_{l-1}^q
    \|\\ 
   & \leq \left\|
    \bbh_l^{pq}(\bbL_N) \bbx_{n,l-1}^q  - \bbh_l^{pq}(\bbL_N) \bbP_N f_{l-1}^q
    \right\|
   +
    \left\|
    \bbh_l^{pq}(\bbL_N) \bbP_N f_{l-1}^q  - \bbP_N   \bbh_l^{pq}(\ccalL_\rho)  f_{l-1}^q
    \right\|
\end{align}
The second term can be bounded with \eqref{eqn:filter-bound} and we denote the bound as $\Delta_N$ for simplicity. The first term can be decomposed by Cauchy-Schwartz inequality and non-amplifying of the filter functions as
 \begin{align}
 \left\| \bbx_{n,l}^p - \bbP_N f_l^p \right\| \leq \sum_{q=1}^{F } \Delta_{N}   \| \bbx_{n,l-1}^q\| + \sum_{q=1}^{F } \| \bbx_{l-1}^q - \bbP_N f_{l-1}^{q} \|.
 \end{align}
%where $C_{per}$ representing the constant in the error bound of manifold filters in \eqref{eqn:appro_filter}. 
To solve this recursion, we need to compute the bound for $\|\bbx_l^p\|$. By normalized Lipschitz continuity of $\sigma$ and the fact that $\sigma(0)=0$, we can get
 \begin{align}
  \| \bbx_l^p \|\leq \left\| \sum_{q=1}^{F } \bbh_l^{pq}(\bbL_N) \bbx_{l-1}^{q}  \right\| \leq  \sum_{q=1}^{F }  \left\| \bbh_l^{pq}(\bbL_N)\right\|  \|\bbx_{l-1}^{q}  \|   \leq   \sum_{q=1}^{F }   \|\bbx_{l-1}^{q}  \| \leq F^{l-1} \| \bbx  \|.
 \end{align}
 Insert this conclusion back to solve the recursion, we can get
 \begin{align}
 \left\| \bbx_{n,l}^p - \bbP_N f_l^p \right\| \leq l   F^{l-1} \Delta_{N}  \|\bbx \|.
 \end{align}
 Replace $l$ with $L$ we can obtain
 \begin{align}
 \|\bm\Phi_\bbG(\bbH,\bbL_N,\bbP_Nf)-\bbP_N \bm\Phi(\bbH,\ccalL_\rho, f))\|  \leq   L F^{L-1}\Delta_{N}     ,
 \end{align}
 % We have
 %  \begin{align}
 % \|\bm\Phi(\bbH,\bbL_N,\bbP_nf)-\bbP_n \bm\Phi(\bbH,\ccalL_\rho, f)) \leq LF^{L-1} \Delta_{N},
 % \end{align}
when the input graph signal is normalized. By replacing $f= \bbI_N \bbx$, we can conclude the proof.
\end{proof}
\section{Local Lipschitz continuity of MNNs}
\label{app:lipschitz-mnn}
We propose that the outputs of MNN defined in \eqref{eqn:mnn} are locally Lipschitz continuous within a certain area, which is stated explicitly as follows.

\begin{proposition}(Local Lipschitz continuity of MNNs)\label{prop:mnn-continuity}
Assume that the assumptions in Theorem 1 hold. Let MNN be $L$ layers with $F$ features in each layer, suppose the manifold filters are nonamplifying with $|\hat{h}(\lambda)|\leq 1$ and the nonlinearities normalized Lipschitz continuous, then there exists a constant $C'$ such that 
\begin{equation}
    \label{eqn:continuity-mnn}
    |\bm\Phi(\bbH,\ccalL_\rho,f)(x) - \bm\Phi(\bbH,\ccalL_\rho,f)(y)|\leq F^L C' \text{dist}(x-y),\quad \text{for all }x,y \in B_r(\ccalM),
\end{equation}
where $B_r(\ccalM)$ is a ball with radius $r$ over $\ccalM$ with respect to the geodesic distance.
\end{proposition}
\begin{proof}
The output of MNN can be written explicitly as 
\begin{align}
    &|\bm\Phi(\bbH,\ccalL_\rho,f)(x) - \bm\Phi(\bbH,\ccalL_\rho,f)(y)| = \left| \sigma\left( \sum_{q=1}^F \bbh_L^{q}(\ccalL_\rho)f_{L-1}^q(x) \right) -\sigma\left(\sum_{q=1}^F \bbh_L^{q}(\ccalL_\rho)f_{L-1}^q(y)\right)\right|\\
 & \leq   \left|   \sum_{q=1}^F \bbh_L^{q}(\ccalL_\rho)f_{L-1}^q(x)  - \sum_{q=1}^F \bbh_L^{q}(\ccalL_\rho)f_{L-1}^q(y) \right|\leq  F  \max_{q=1,\cdots, F}\left|   \bbh_L^{q}(\ccalL_\rho)f_{L-1}^q(x)  -  \bbh_L^{q}(\ccalL_\rho)f_{L-1}^q(y)\right|.
\end{align}
We have $f_{L-1}^q(x) = \sigma\left(\sum_{p=1}^F \bbh_{L-1}^p f_{L-2}^p(x)\right)$. The process can be repeated recursively by expanding $f_{L-1}^q(x)$ and $f_{L-1}^q(y)$, and finally, we can have 
\begin{align}
    |\bm\Phi(\bbH,\ccalL_\rho,f)(x) - \bm\Phi(\bbH,\ccalL_\rho,f)(y)| \leq F^L |\bbh_L(\ccalL_\rho)\cdots \bbh_1(\ccalL_\rho) f(x) -\bbh_L(\ccalL_\rho)\cdots \bbh_1(\ccalL_\rho) f(y) |.
\end{align}
With $f$ as a $\lambda$-bandlimited manifold signal, we suppose $g = \bbh_L(\ccalL_\rho)\cdots \bbh_1(\ccalL_\rho) f$. As $\langle f,\bm\phi_i\rangle = 0$ for all $i>M$, $g$ is also bandlimited and possesses $M$ spectral components. The gradient can be bounded according to \citep{shi2010gradient} combined with the non-amplifying property of the filter function as
\begin{align}
    \|\nabla g \|_\infty\leq C \sum_{\lambda_i\leq \lambda}\left|\hat{h}(\lambda_i)\right|^L\lambda_i^{\frac{d
+1}{2}}\|f\|_\ccalM \leq C\sum_{\lambda_i\leq \lambda}\lambda_i^{\frac{d
+1}{2}}\|f\|_\ccalM
\end{align}
From Theorem 4.5 in \citep{evans2018measure}, $g$ is locally Lipschitz continuous as
\begin{align}
    |g(x)-g(y)|\leq C'\text{dist}(x-y),\quad \text{with }x,y \in B_r(\ccalM), 
\end{align}
where $B_r(\ccalM)$ is a closed ball with radius $r$ with $C'$ depending on the geometry of $\ccalM$.

Combining the above, we have the continuity of the output of MNN as
\begin{equation}
     |\bm\Phi(\bbH,\ccalL_\rho,f)(x) - \bm\Phi(\bbH,\ccalL_\rho,f)(y)| \leq F^L C'\text{dist}(x-y) , \quad \text{with }x,y \in B_r(\ccalM),
\end{equation}
which concludes the proof.
\end{proof}

\newpage
\section{Proof of Theorem \ref{thm:generalization-node-gauss}}
\label{app:proof-node}
% \begin{equation}
%         GA\leq C_1 \left(\frac{\log(C/\delta)}{N} \right)^{\frac{1}{d+4}} +  C_2\left(\frac{\log(C/\delta)}{N} \right)^{\frac{1}{d+4}} \theta_M^{-1} 
%      + C_3 \sqrt{\frac{\log(1/\delta)}{N}} + C_4 M^{-1} +  L \left(\frac{\log N}{N}\right)^{\frac{1}{d}}
%     \end{equation}
    
\begin{proof}
    % Suppose we have $L^1$ loss as the loss function. Considering that the input and target graph signals $\{\bbx, \bby\}_{i=1}^N$ only have an entry in the $i$-th element. The empirical risk on the GNN can be written as 
    % \begin{equation}
    %     R_\bbG(\bbH) = \frac{1}{N} \sum_{i=1}^N \ell([\bm\Phi(\bbH,\bbL_N, \bbx]_i , [\bby]_i)).  %= \frac{1}{N} \sum_{i=1}^N \int_\ccalM \ell(\bbI_N \bm\Phi(\bbH,\bbL_N,\bbx)(x),\bbI_N \bby (x)) \text{d}\mu(x)
    %     %\int_\ccalM \sum_{i=1}^N \ell([\bm\Phi(\bbH,\bbL_N, \bbx]_i , [\bby]_i))\mathbbm{1}_{x\in V_i} \text{d}\mu(x)%=    \sum_{i=1}^N  \|\bbI_N \bm\Phi(\bbH,\bbL_N, \bbx) -\bbI_N\bby \|_{1,\ccalM},
    %  \end{equation}

To analyze the difference between the empirical risk and statistical risk, we introduce an intermediate term which is the induced version of the sampled MNN output. We define $\bbI_N$ as the inducing operator based on the decomposition $\{V_i\}_{i=1}^N$ defined in Section \ref{sec:induced}. This intermediate term is written explicitly as
    \begin{equation}
    \label{eqn:induce-manifold}
       \overline{\bm\Phi}(\bbH,\ccalL_\rho,f)(x) = \bbI_N \bbP_N \bm\Phi(\bbH,\ccalL_\rho, f) (x) = \sum_{i=1}^N \bm\Phi(\bbH,\ccalL_\rho, f) (x_i)\mathbbm{1}_{x\in V_i},\;\forall \; x\in\ccalM,
    \end{equation}
 where $x_i \in X_N$ are sampled points from the manifold. 
 
Suppose $\bbH \in \arg\min_{\bbH\in\ccalH} R_\ccalM(\bbH)$, we have 
\begin{align}
    & GA = \sup_{\bbH\in\ccalH} |R_\bbG(\bbH) - R_\ccalM(\bbH) |
    %,\\
    % &= \mathbb{E}_{X^N\sim \mu^N}\left[ \ell\left( \bm\Phi(\bbH_E,\bbx_{N}), \bby_{N}\right)\right] - \ell(\bm\Phi(\bbH_E,\bbx_{N}),\bby_{N})
\end{align}
 % While for the output manifold function, we have 
 %    \begin{equation}
 %        \overline g (x) = \bbI_N \bbP_N g (x)= \sum_{i=1}^N g(x_i) \mathbbm{1}_{x\in V_i},\;\forall \; x\in\ccalM
 %    \end{equation}
    The difference between $R_\bbG(\bbH)$ and $R_\ccalM(\bbH)$ can be decomposed as 
    \begin{align}
       &\nonumber \left|R_\bbG(\bbH) -  R_\ccalM(\bbH) \right|\\
       &= \left|\frac{1}{N} \sum_{i=1}^N \ell([\bm\Phi_\bbG(\bbH,\bbL_N, \bbx)]_i , [\bby]_i) - \int_\ccalM \ell\left( {\bm\Phi}(\bbH,\ccalL_\rho, f)(x), {g}(x) \right)\text{d}\mu(x) \right|\\
     &\nonumber  = \Bigg|\frac{1}{N} \sum_{i=1}^N \ell([\bm\Phi_\bbG(\bbH,\bbL_N, \bbx)]_i , [\bby]_i) - \int_\ccalM \ell\left( \overline{\bm\Phi}(\bbH,\ccalL_\rho, f)(x),  {g}(x) \right)\text{d}\mu(x) 
\\& \qquad \qquad    
+ \int_\ccalM \ell\left( \overline{\bm\Phi}(\bbH,\ccalL_\rho, f)(x),  {g}(x) \right)\text{d}\mu(x)  - \int_\ccalM \ell\left( {\bm\Phi}(\bbH,\ccalL_\rho, f)(x), {g}(x) \right)\text{d}\mu(x)\Bigg|\\
& \nonumber  \leq    \left| \frac{1}{N} \sum_{i=1}^N \ell([\bm\Phi_\bbG(\bbH,\bbL_N, \bbx)]_i , [\bby]_i) - \int_\ccalM \ell\left( \overline{\bm\Phi}(\bbH,\ccalL_\rho, f)(x),  {g}(x) \right)\text{d}\mu(x) \right| 
\\& \qquad \qquad +\left|\int_\ccalM \ell\left( \overline{\bm\Phi}(\bbH,\ccalL_\rho, f)(x),  {g}(x) \right)\text{d}\mu(x)  - \int_\ccalM \ell\left( {\bm\Phi}(\bbH,\ccalL_\rho, f)(x), {g}(x) \right)\text{d}\mu(x)\right|\label{eqn:decomp1}
    \end{align}
We analyze the two terms in \eqref{eqn:decomp1} separately, with the first term bounded based on the convergence of GNN to MNN and the second term bounded with the smoothness of manifold functions.

The first term in \eqref{eqn:decomp1} can be written as 
\begin{align}
    & \label{eqn:eq1}\left| \frac{1}{N} \sum_{i=1}^N \ell([\bm\Phi_\bbG(\bbH,\bbL_N, \bbx )]_i , [\bby ]_i) - \int_\ccalM \ell\left( \overline{\bm\Phi}(\bbH,\ccalL_\rho, f)(x),  {g}(x) \right)\text{d}\mu(x) \right| \\
     & \label{eqn:eq2} = \frac{1}{N} \left|   \sum_{i=1}^N \ell([\bm\Phi_\bbG(\bbH,\bbL_N, \bbx)]_i , [\bby]_i) - \sum_{i=1}^N  \ell(\bm\Phi(\bbH, \ccalL_\rho, f)(x_i), g (x_i) )  \right|\\
     &\label{eqn:eq3}\leq \frac{1}{N} \sum_{i=1}^N  \left|\ell([\bm\Phi_\bbG(\bbH,\bbL_N, \bbx)]_i , [\bby]_i) -  \ell(\bm\Phi(\bbH, \ccalL_\rho, f)(x_i), g (x_i) ) \right|  \\
     % &\label{eqn:eq4}\leq \frac{1}{N} \sum_{i=1}^N \int_\ccalM \Big||\bbI_N \bm\Phi(\bbH,\bbL_N,\bbx)(x) - \bbI_N \bby (x)| -| \bm\Phi(\bbH, \ccalL_\rho, \bbI_N \bbx)(x) - \bbI_N \bby (x)|\Big| \text{d}\mu(x) \\
     &\label{eqn:eq4}\leq \frac{1}{N} \sum_{i=1}^N  \Big| [\bm\Phi_\bbG(\bbH,\bbL_N, \bbx)]_i  - \bm\Phi(\bbH, \ccalL_\rho, f)(x_i) \Big|  \\
     &\label{eqn:eq5}\leq \frac{1}{N} \| \bm\Phi_\bbG(\bbH,\bbL_N,\bbx) - \bbP_N {\bm\Phi}(\bbH,\ccalL_\rho,\bbI_N \bbx) \|_1\\
     &\leq \frac{1}{\sqrt{N}}L F^{L-1}\left((C_1C_L+C_2) \sqrt{\epsilon}  +  \frac{\pi^2}{6}\sqrt{\frac{\log(1/\delta)}{N}} \right)\label{eqn:part1}
\end{align}
From \eqref{eqn:eq1} to \eqref{eqn:eq2}, we use the definition of induced manifold signal defined in \eqref{eqn:induce-manifold}. We utilize the Lipschitz continuity assumption on loss function from \eqref{eqn:eq3} to \eqref{eqn:eq4}. From \eqref{eqn:eq4} to \eqref{eqn:eq5}, it depends on the fact that $\bbx$ is a single-entry vector and that $[\bby]_i$ is the value sampled from target manifold function $g$ evaluated on $x_i$. Finally the bound depends on the convergence of GNN on the sampled graph to the MNN as stated in Proposition \ref{prop:prob-diff}.

The second term is decomposed as 
\begin{align}
&\label{eqn:eq21}\left|\int_\ccalM \ell\left( \overline{\bm\Phi}(\bbH,\ccalL_\rho, f)(x),  {g}(x) \right)\text{d}\mu(x)  - \int_\ccalM \ell\left( {\bm\Phi}(\bbH,\ccalL_\rho, f)(x), {g}(x) \right)\text{d}\mu(x)\right| \\
    &\label{eqn:eq22}\leq \left| \sum_{i=1}^N \int_{V_i} \ell\left( \overline{\bm\Phi}(\bbH,\ccalL_\rho, f)(x),  {g}(x) \right) \text{d}\mu(x)  - \sum_{i=1}^N \int_{V_i} \ell\left( {\bm\Phi}(\bbH,\ccalL_\rho, f)(x), {g}(x) \right)\text{d}\mu(x)\right|\\
    &\label{eqn:eq24}\leq \sum_{i=1}^N \int_{V_i} \left| \ell\left( \overline{\bm\Phi}(\bbH,\ccalL_\rho, f)(x),  {g}(x) \right) - \ell\left( {\bm\Phi}(\bbH,\ccalL_\rho, f)(x), {g}(x) \right)\right|  \text{d}\mu(x) \\
    % &\nonumber \leq \sum_{i=1}^N \int_{V_i} \Big| \ell\left( \overline{\bm\Phi}(\bbH,\ccalL_\rho, f)(x), \overline{g}(x) \right) - \ell(\overline{g}(x),{g}(x))  \\
    % &\qquad \qquad \qquad \qquad \qquad\qquad + \ell(g(x),\overline{g}(x))- \ell\left( {\bm\Phi}(\bbH,\ccalL_\rho, f)(x), {g}(x) \right)\Big| \mathbbm{1}_{x\in V_i}\text{d}\mu(x)\\
    &\label{eqn:eq23}\leq \sum_{i=1}^N \int_{V_i} \left| \overline{\bm\Phi}(\bbH,\ccalL_\rho, f)(x)   -  {\bm\Phi}(\bbH,\ccalL_\rho, f)(x)  \right| \text{d}\mu(x) \\
    &\label{eqn:eq231}\leq \sum_{i=1}^N \int_{V_i} \left|\bm\Phi(\bbH,\ccalL_\rho,f)(x_i) - \bm\Phi(\bbH,\ccalL_\rho,f)(x)  \right| \text{d}\mu(x) 
    %+\sum_{i=1}^N \int_{V_i} |g(x)-g(x_i)|\mathbbm{1}_{x\in V_i}\text{d}\mu(x) 
\end{align}
From \eqref{eqn:eq21} to \eqref{eqn:eq22}, it relies on the decomposition of the MNN output over $\{V_i\}_{i=1}^N$. From \eqref{eqn:eq24} to \eqref{eqn:eq23}, we use the Lipschitz continuity of loss function. From \eqref{eqn:eq23} to \eqref{eqn:eq231}, we use the definition of $\overline{\bm\Phi}(\bbH,\ccalL_\rho, f)$. Proposition \ref{prop:mnn-continuity} indicates that the MNN outputs are Lipschitz continuous within a certain range, which leads to 
\begin{align}
&\nonumber \sum_{i=1}^N \int_{V_i} \left|\bm\Phi(\bbH,\ccalL_\rho,f)(x_i) - \bm\Phi(\bbH,\ccalL_\rho,f)(x)  \right| \text{d}\mu(x)  \\ &\leq  \sum_{i=1}^N \int_{V_i} F^L C_3 \left(\frac{\log N}{N}\right)^{\frac{1}{d}} \text{d}\mu(x)
\\& = F^L C_3 \left(\frac{\log N}{N}\right)^{\frac{1}{d}} \sum_{i=1}^N \int_{V_i} \text{d}\mu(x)
\\&\leq F^L C_3 \left(\frac{\log N}{N}\right)^{\frac{1}{d}},\label{eqn:part2}
\end{align}
when $d\geq 3$. If $d=2$, the bound would be 
\begin{equation}
    \sum_{i=1}^N \int_{V_i} \left|\bm\Phi(\bbH,\ccalL_\rho,f)(x_i) - \bm\Phi(\bbH,\ccalL_\rho,f)(x)  \right| \text{d}\mu(x) \leq F^L C_3 \frac{(\log N)^\frac{3}{4}}{N^{\frac{1}{2}}}.\label{eqn:part2-1}
\end{equation}
Combining \eqref{eqn:part1} and \eqref{eqn:part2} (or \eqref{eqn:part2-1}), we can conclude the proof.
\end{proof}
\newpage
\section{Corollary of Theorem 1}
\begin{corollary}
    Suppose the GNN with filters satisfying Assumption \ref{ass:low-pass} have $L$ layers with $F$ features in each layer and the input signal is bandlimited (Definition \ref{def:band})). Suppose graphs $\bbG_1$ with $N_1$ nodes and $\bbG_2$ with $N_2$ nodes are sampled from the same underlying manifold $\ccalM$. Under Assumptions \ref{ass:activation} and \ref{ass:loss} it holds in probability at least $1-\delta$ that 
    \begin{align}
       \nonumber  &\sup_{\bbH\in\ccalH} \left|\frac{1}{N_1}\sum_{i=1}^{N_1}\ell \left(  [\bm\Phi_{\bbG_1}(\bbH, \bbL_{N_1}, \bbx_1 )]_i  ,  [\bby_1 ]_i \right) - \frac{1}{N_2}\sum_{i=1}^{N_2}\ell \left(  [\bm\Phi_{\bbG_2}(\bbH, \bbL_{N_2}, \bbx_2 )]_i  ,  [\bby_2 ]_i \right)\right|\leq\\
       &   L F^{L-1}\left((C_1C_L+C_2) \frac{\sqrt{\epsilon}(\sqrt{N_1}+\sqrt{N_2}) }{\sqrt{N_1 N_2}}
     +  \frac{\pi^2(N_1+N_2)\sqrt{\log(1/\delta)}}{6N_1N_2}  \right) +  F^L C_3 \left(\frac{\log N_1}{N_1}\right)^{\frac{1}{d}}+F^L C_3 \left(\frac{\log N_2}{N_2}\right)^{\frac{1}{d}},
    \end{align}
    with $C_1$, $C_2$, and $C_3$ depending on the geometry of $\ccalM$, $C_L$ is the spectral continuity constant in Assumption \ref{ass:low-pass}.
\end{corollary}
\begin{proof}
By importing the statistical risk over the manifold $R_\ccalM(\bbH)$ in \eqref{eqn:statistical-manifoldloss-node}, the bound can be derived with a triangle inequality as
    \begin{align}
      \nonumber & \sup_{\bbH\in\ccalH} \left|\frac{1}{N_1}\sum_{i=1}^{N_1}\ell \left(  [\bm\Phi_{\bbG_1}(\bbH, \bbL_{N_1}, \bbx_1 )]_i  ,  [\bby_1 ]_i \right) - \frac{1}{N_2}\sum_{i=1}^{N_2}\ell \left(  [\bm\Phi_{\bbG_2}(\bbH, \bbL_{N_2}, \bbx_2 )]_i  ,  [\bby_2 ]_i \right)\right|\\
        & =  \sup_{\bbH\in\ccalH}\left|\frac{1}{N_1}\sum_{i=1}^{N_1}\ell \left(  [\bm\Phi_{\bbG_1}(\bbH, \bbL_{N_1}, \bbx_1 )]_i  ,  [\bby_1 ]_i \right) -R_\ccalM(\bbH) + R_\ccalM(\bbH)-  \frac{1}{N_2}\sum_{i=1}^{N_2}\ell \left(  [\bm\Phi_{\bbG_2}(\bbH, \bbL_{N_2}, \bbx_2 )]_i  ,  [\bby_2 ]_i \right)\right|\\
        &\leq \sup_{\bbH\in\ccalH}\left|\frac{1}{N_1}\sum_{i=1}^{N_1}\ell \left(  [\bm\Phi_{\bbG_1}(\bbH, \bbL_{N_1}, \bbx_1 )]_i  ,  [\bby_1 ]_i \right) -R_\ccalM(\bbH)\right| + \sup_{\bbH\in\ccalH}\left|\frac{1}{N_2}\sum_{i=1}^{N_2}\ell \left(  [\bm\Phi_{\bbG_2}(\bbH, \bbL_{N_2}, \bbx_2 )]_i  ,  [\bby_2 ]_i \right) -R_\ccalM(\bbH)\right|.
    \end{align}
    Inserting the result in Theorem \ref{thm:generalization-node-gauss} concludes the proof.
\end{proof}

\newpage
\section{Proof of Theorem \ref{thm:generalization-graph-gauss}}
\label{app:proof-graph}
    % \begin{equation}
    %     GA\leq \frac{C_1}{\sqrt{N}}\left(\frac{\log(C/\delta)}{N} \right)^{\frac{1}{d+4}}  +  \frac{C_2\theta_M^{-1}}{\sqrt{N}}   \left(\frac{\log(C/\delta)}{N} \right)^{\frac{1}{d+4}}
    %  + C_3 \frac{\sqrt{\log(1/\delta)}}{N} + C_4 \frac{M^{-1}}{\sqrt{N}} + L \left(\frac{\log N}{N}\right)^{\frac{1}{d}}
    % \end{equation}
\begin{proof}
% Suppose $\bbH \in \arg\min_{\bbH\in\ccalH} R_\ccalM(\bbH)$, we have 
% \begin{align}
%     & GA = \min_{\bbH\in\ccalH} R_\bbG(\bbH) - \min_{\bbH\in\ccalH} R_\ccalM(\bbH) \leq R_\bbG(\bbH) - R_\ccalM(\bbH)
% \end{align}
 We can write the difference as 
\begin{align}
    \nonumber & \left| R_\bbG(\bbH) - R_\ccalM(\bbH)\right| \\
    &\leq \sum_{k=1}^K \left|  \ell\left(\frac{1}{N} \sum_{i=1}^N [\bm\Phi_\bbG(\bbH, \bbL_{N,k},\bbx_k)]_i, y_k \right) -\ell\left(\int_{\ccalM_k}\bm\Phi(\bbH,\ccalL_{\rho,k}, f_k) \text{d}\mu_k(x), y_k\right) \right|
\end{align}
Based on the property of absolute value inequality and the Lipschitz continuity assumption of loss function (Assumption \ref{ass:loss}), we have 
\begin{align}
   &\nonumber\left| \ell\left(\frac{1}{N} \sum_{i=1}^N [\bm\Phi_\bbG(\bbH, \bbL_{N,k},\bbx_k)]_i, y_k \right) -\ell\left(\int_{\ccalM_k}\bm\Phi(\bbH,\ccalL_{\rho,k}, f_k) \text{d}\mu_k(x), y_k\right)\right|\\
   % &\leq \left|\left| \frac{1}{N} \sum_{i=1}^N [\bm\Phi(\bbH, \bbL_{N,k},\bbx_k)]_i - y_k\right| - \left|\int_{\ccalM_k}\bm\Phi(\bbH,\ccalL_{\rho,k}, f_k) \text{d}\mu_k(x) - y_k\right|\right|\\
   &\leq \left|\frac{1}{N} \sum_{i=1}^N [\bm\Phi_\bbG(\bbH, \bbL_{N,k},\bbx_k)]_i - \int_{\ccalM_k}\bm\Phi(\bbH,\ccalL_{\rho,k}, f_k) \text{d}\mu_k(x) \right|
\end{align}
We insert an intermediate term $\bm\Phi(\bbH, \ccalL_{\rho,k},f_k)(x_i)$ as the value evaluated on the sampled point $x_i$, which leads to 
\begin{align}
    &\left|\frac{1}{N} \sum_{i=1}^N [\bm\Phi_\bbG(\bbH, \bbL_{N,k},\bbx_k)]_i - \int_{\ccalM_k}\bm\Phi(\bbH,\ccalL_{\rho,k}, f_k) \text{d}\mu_k(x) \right| \\
   &\label{eqn:eq31}\nonumber \leq  \left|\frac{1}{N} \sum_{i=1}^N [\bm\Phi_\bbG(\bbH, \bbL_{N,k},\bbx_k)]_i - \frac{1}{N}\sum_{i=1}^N \bm\Phi(\bbH,\ccalL_{\rho,k},f_k)(x_i)\right| + \\
   &\qquad \qquad \qquad \qquad \qquad \qquad \left|\frac{1}{N}\sum_{i=1}^N \bm\Phi(\bbH,\ccalL_{\rho,k},f_k)(x_i) - \int_{\ccalM_k}\bm\Phi(\bbH,\ccalL_{\rho,k}, f_k) \text{d}\mu_k(x)\right|
\end{align}
The first term in \eqref{eqn:eq31} can be bounded similarly as \eqref{eqn:eq5}, which is explicitly written as 
\begin{align}
   & \left|\frac{1}{N} \sum_{i=1}^N [\bm\Phi_\bbG(\bbH, \bbL_{N,k},\bbx_k)]_i - \frac{1}{N}\sum_{i=1}^N \bm\Phi(\bbH,\ccalL_{\rho,k},f_k)(x_i)\right| \\
    &\leq \frac{1}{N} \| \bm\Phi_\bbG(\bbH,\bbL_N,\bbx_k) - \bbP_N {\bm\Phi}(\bbH,\ccalL_\rho,f_k) \|_1\\ &\leq \frac{1}{\sqrt{N}} \| \bm\Phi_\bbG(\bbH,\bbL_N,\bbx_k) - \bbP_N {\bm\Phi}(\bbH,\ccalL_\rho,f_k) \|_2\\
    &\leq \frac{1}{\sqrt{N}}  \left((C_1 C_L+C_2)\sqrt{\epsilon} +\frac{\pi^2}{6}\sqrt{\frac{\log(1/\delta)}{N}}\right)
\end{align}
The second term is 
\begin{align}
   & \left|\frac{1}{N}\sum_{i=1}^N \bm\Phi(\bbH,\ccalL_{\rho,k},f_k)(x_i) - \int_{\ccalM_k}\bm\Phi(\bbH,\ccalL_{\rho,k}, f_k) \text{d}\mu_k(x)\right|\\
    & =\left|\sum_{i=1}^N \int_{V_i} \bm\Phi(\bbH,\ccalL_{\rho,k},f_k)(x_i) \text{d}\mu_k(x) - \sum_{i=1}^N \int_{V_i} \bm\Phi(\bbH,\ccalL_{\rho,k}, f_k)(x)\text{d}\mu_k(x)\right|\\
    &\leq \sum_{i=1}^N \int_{V_i} \left|\bm\Phi(\bbH,\ccalL_{\rho,k},f_k)(x_i) - \bm\Phi(\bbH,\ccalL_{\rho,k}, f_k)(x)\right| \text{d}\mu_k(x)\\
    &\leq F^L C_3 \left(\frac{\log N}{N}\right)^{\frac{1}{d}}
\end{align}
This depends on the Lipschitz continuity of the output manifold function in Proposition \ref{prop:mnn-continuity}.
\end{proof}
\newpage

\section{Further references}
\paragraph{Graphon theory}
Different from the manifold model we are using, some research constructs graphs derived from graphons, which can be viewed as a random limit graph model. This research has focused on their convergence, stability, as well as transferability \citep{ruiz2020graphon, maskey2023transferability, keriven2020convergence}. In \citep{parada2023graphon}, a graphon is used as a pooling tool in GNNs. Despite its utility, the graphon presents several limitations compared to the manifold model we use. Firstly, the graphon model assumes an infinite degree at every node \citep{lovasz2012large}, which is not the case in the manifold model. 
Additionally, graphons offer limited insight into the underlying model; visualizing a graphon is challenging, except in the stochastic block model case. Manifolds, however, are more interpretable, especially when based on familiar shapes like spheres and 3D models (see Figures \ref{fig:gauss-kernel-graph} and \ref{fig:epsilon-graph}). Finally, the manifold model supports a wider range of characterizable models, making it a more realistic choice.

% \paragraph{Manifold Theory} 
% Some studies have used graph samples to infer properties of the underlying manifold itself. These properties include the validity of the manifold assumption \cite{fefferman2016testing}, the manifold dimension \cite{farahmand2007manifold} and the complexity of these inferences \cite{narayanan2009sample,aamari2021statistical}. Other research has focused on prediction and classification using manifolds and manifold data, proposing various algorithms and methods. Impressive examples include the Isomap algorithm \cite{choi2004kernel,wu2004extended,yang2016multi} and other manifold learning techniques \cite{talwalkar2008large}. These techniques aim to infer manifold properties without analyzing the generalization capabilities of GNNs operated on the sampled manifold.

\paragraph{Transferability of GNNs} 
The transferability of GNNs has been extensively studied by examining the differences in GNN outputs across graphs of varying sizes as they converge to a limit model. This analysis, however, often lacks statistical generalization. Several studies have explored GNN transferability with graphon models, proving bounds on the differences in GNN outputs \citep{ruiz2023transferability,ruiz2020graphon, maskey2023transferability}. Other research has demonstrated how increasing graph size during GNN training can improve generalization to large-scale graphs \citep{cervino2023learning}. The transferability of GNNs has also been investigated in the context of graphs generated from general topological spaces \citep{levie2021transferability} and manifolds \citep{wang2023geometric}. Furthermore, a novel graphop operator has been proposed as a limit model for both dense and sparse graphs, with proven transferability results \citep{le2024limits}. Further research has focused on transfer learning for GNNs by measuring distances between graphs without assuming a limit model \citep{lee2017transfer, zhu2021transfer}.  Finally, a transferable graph transformer has been proposed and empirically validated \citep{he2023network}.

\section{Filter Assumption}
In the main results, we assume that the filters in GNN and MNN satisfy Assumption \ref{ass:low-pass}. This may lead to limited discriminability in high-frequency spectrum. While this is a reasonable assumption, high-frequency signals on graphs or manifolds can fluctuate significantly between adjacent entries, leading to instability and learning challenges. We expect a degree of local homogeneity, which translates to low-frequency signals. This assumption is supported by empirical evidence in various domains, including opinion dynamics, econometrics, and graph signal processing \citep{LPF_consensus, billio2012econometric,ramakrishna2020user}. 
Moreover, several other effective learning techniques, such as Principal Component Analysis (PCA) and Isomap, implicitly employ low-pass filtering. Therefore, we believe that the filter assumption is not restrictive and is well-supported by both practical applications and theoretical considerations.

\section{Manifold Assumption}
In this paper, we considered the case in which graphs are sampled from manifolds. This is an assumption that has been widely used in practice. From dynamical systems \citep{talmon2015manifold} to images \citep{peyre2009manifold,osher2017low}, assuming an underlying low dimensional manifold is a common practice. 
Real-world graphs, like the ones considered in the node prediction experiments, can be assumed to be sampled from $d$-dimensional manifolds. To support this argument, in Figure \ref{fig:manifold_assumption}, we plot the $100$ largest eigenvalues of the Laplacian matrix associated with each graph. By doing this, we show a fast decay in the values of the eigenvalues progress. This decay shows that the information is mostly supported on a subset of the eigenvalues thus reinforcing the idea that it comes from a low dimensional manifold.

\begin{figure*}[ht]
\centering
\begin{subfigure}{0.24\textwidth}
  \centering
  \includegraphics[width=\linewidth]{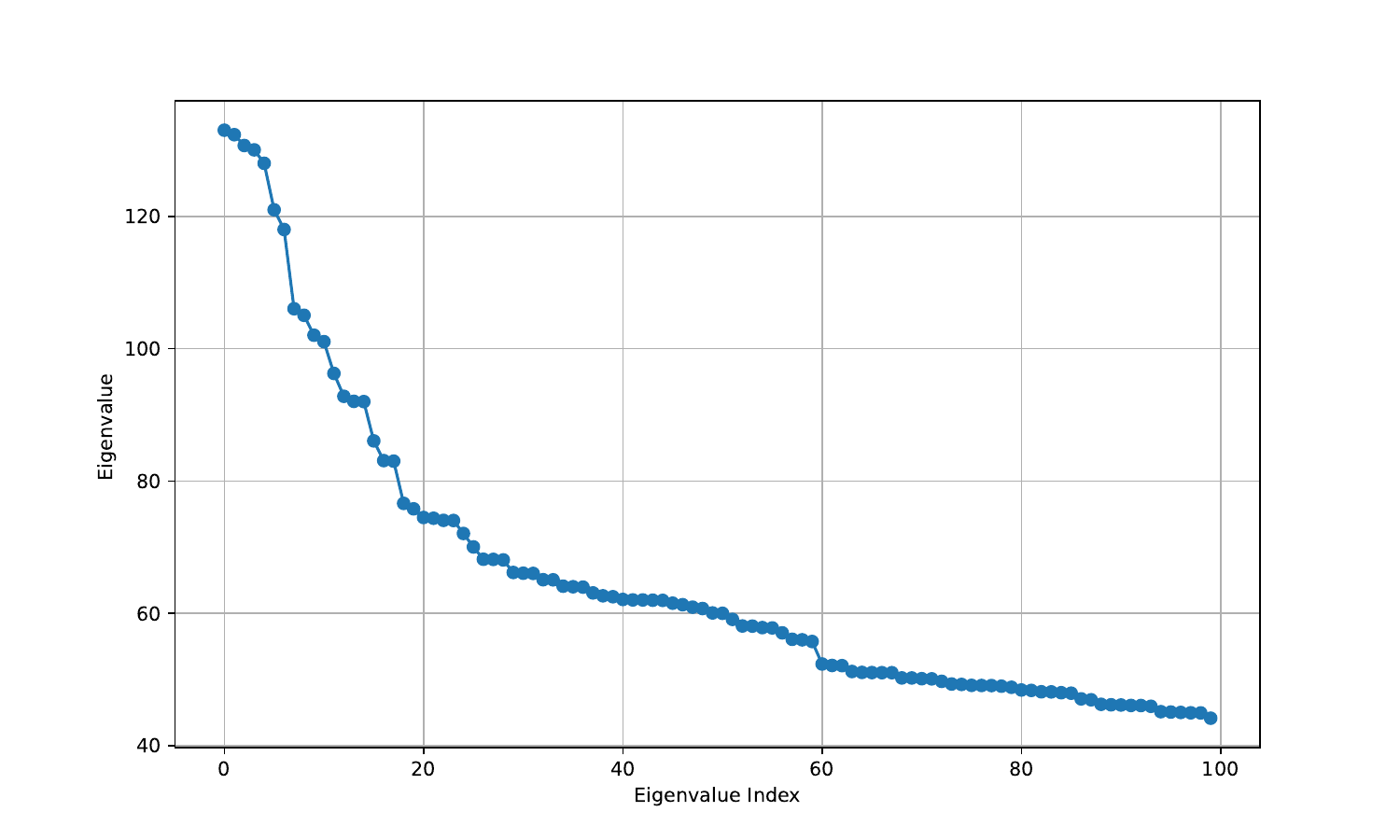}
  \caption{Amazon}
\end{subfigure}
\begin{subfigure}{0.24\textwidth}
  \centering
  \includegraphics[width=\linewidth]{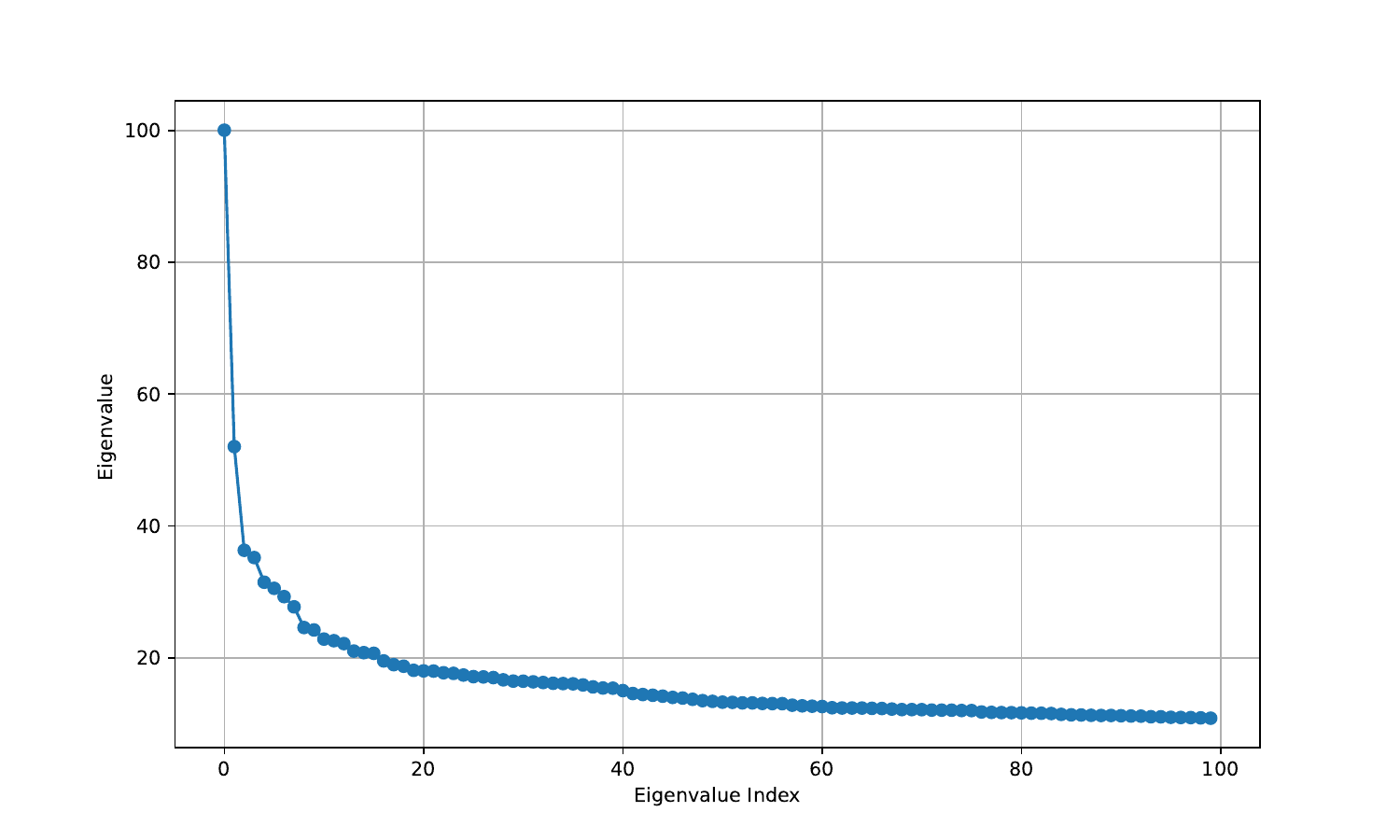}
  \caption{CiteSeer}
\end{subfigure}
\begin{subfigure}{0.24\textwidth}
  \centering
  \includegraphics[width=\linewidth]{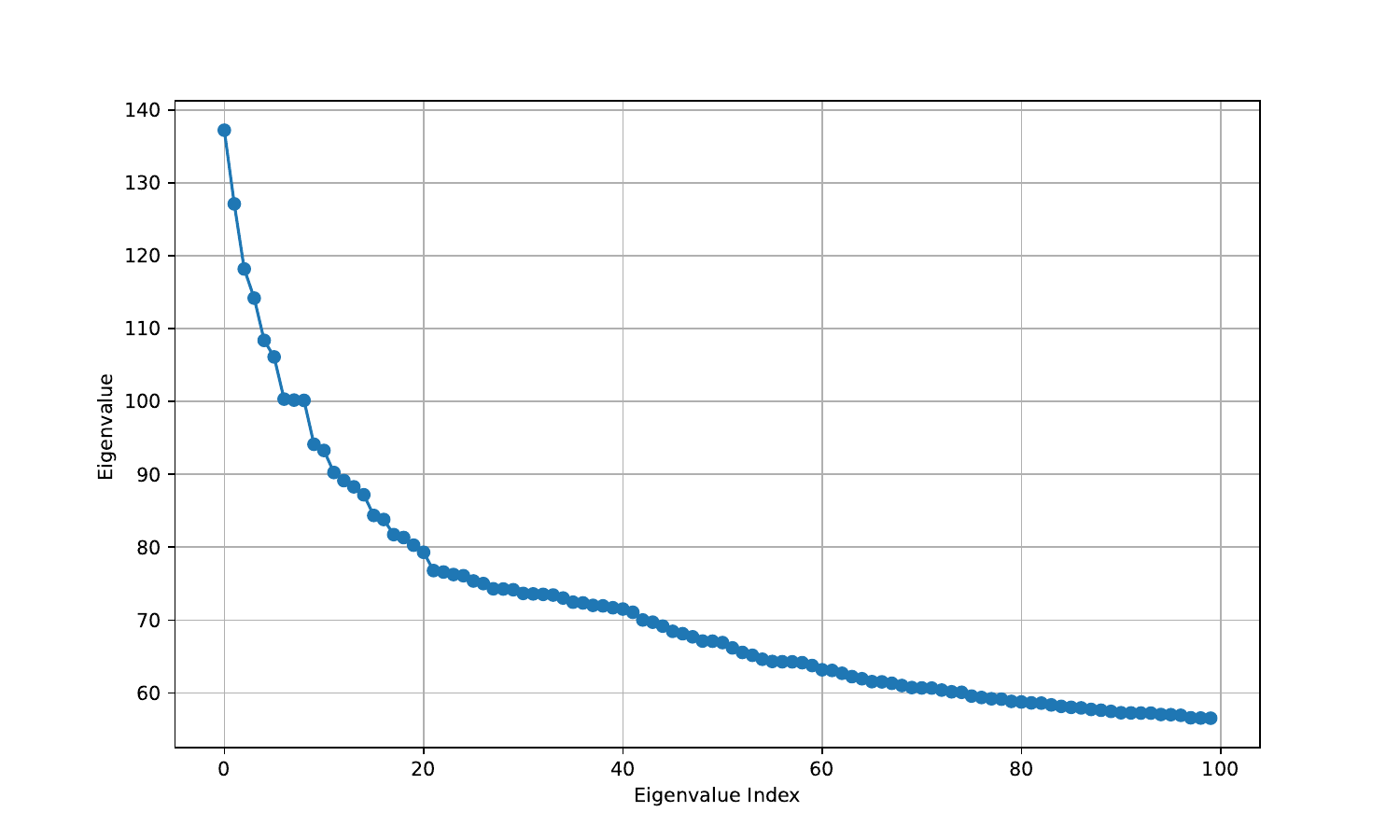}
  \caption{CS}
\end{subfigure}
\begin{subfigure}{0.24\textwidth}
  \centering
  \includegraphics[width=\linewidth]{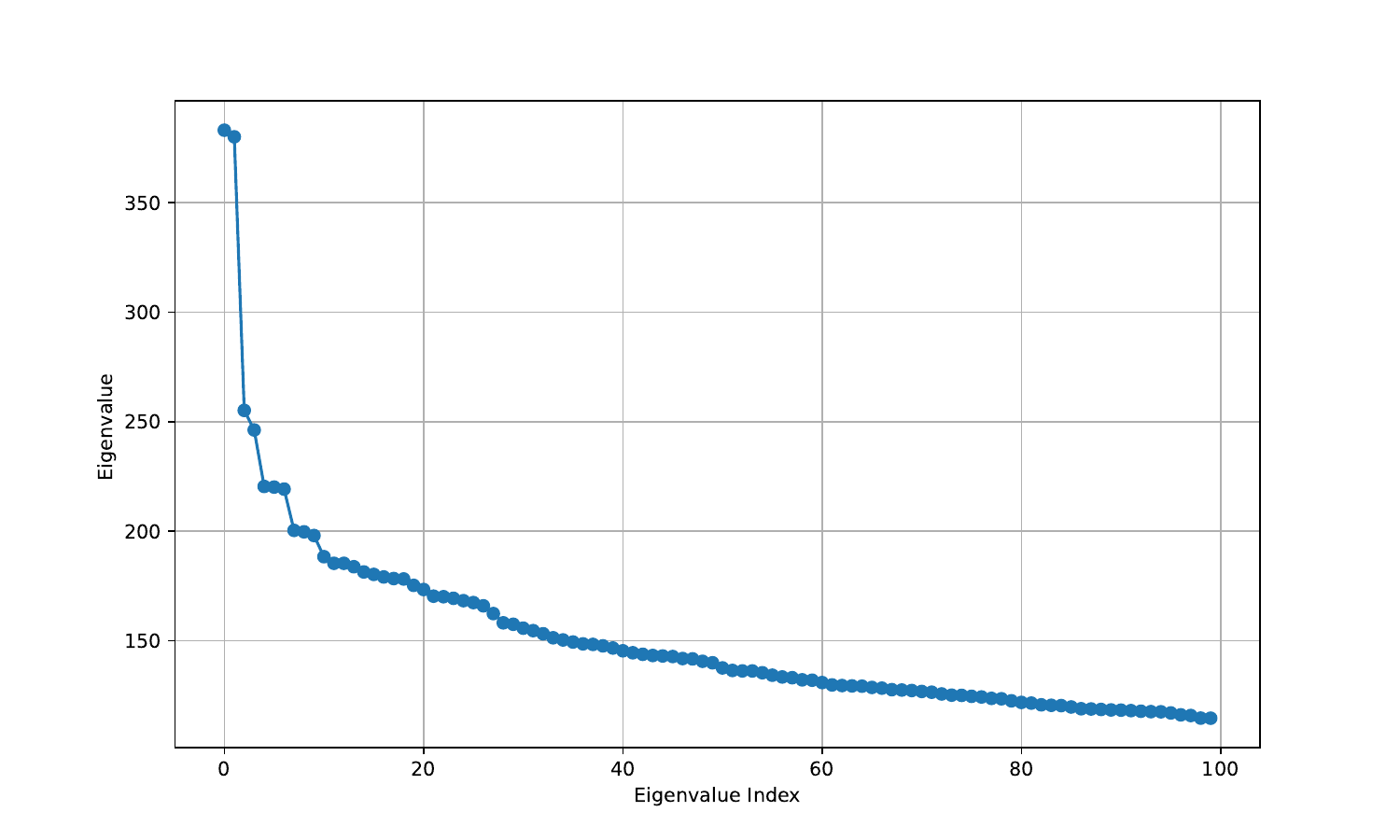}
  \caption{Physics}
\end{subfigure}
\begin{subfigure}{0.24\textwidth}
  \centering
  \includegraphics[width=\linewidth]{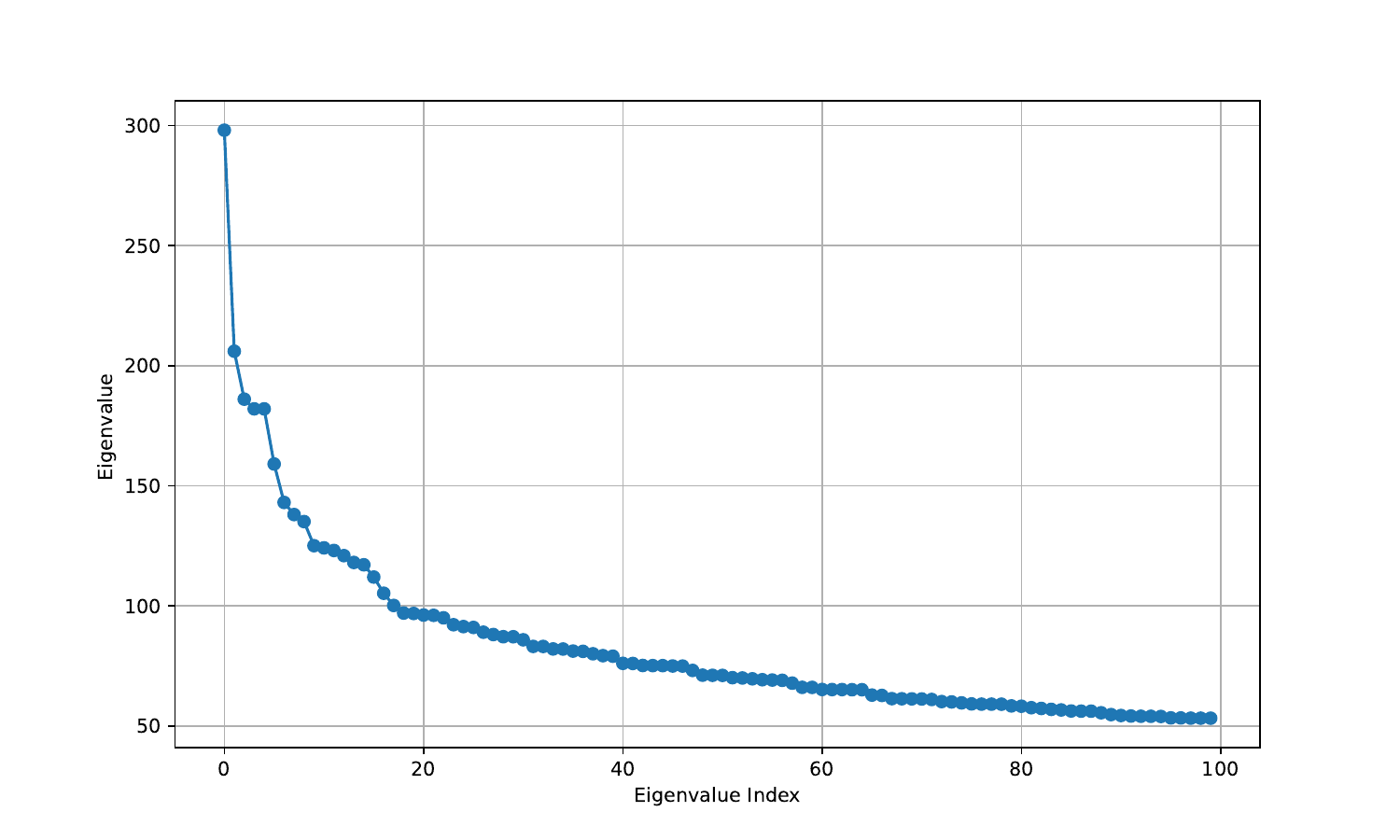}
  \caption{Cora}
\end{subfigure}
\begin{subfigure}{0.24\textwidth}
  \centering
  \includegraphics[width=\linewidth]{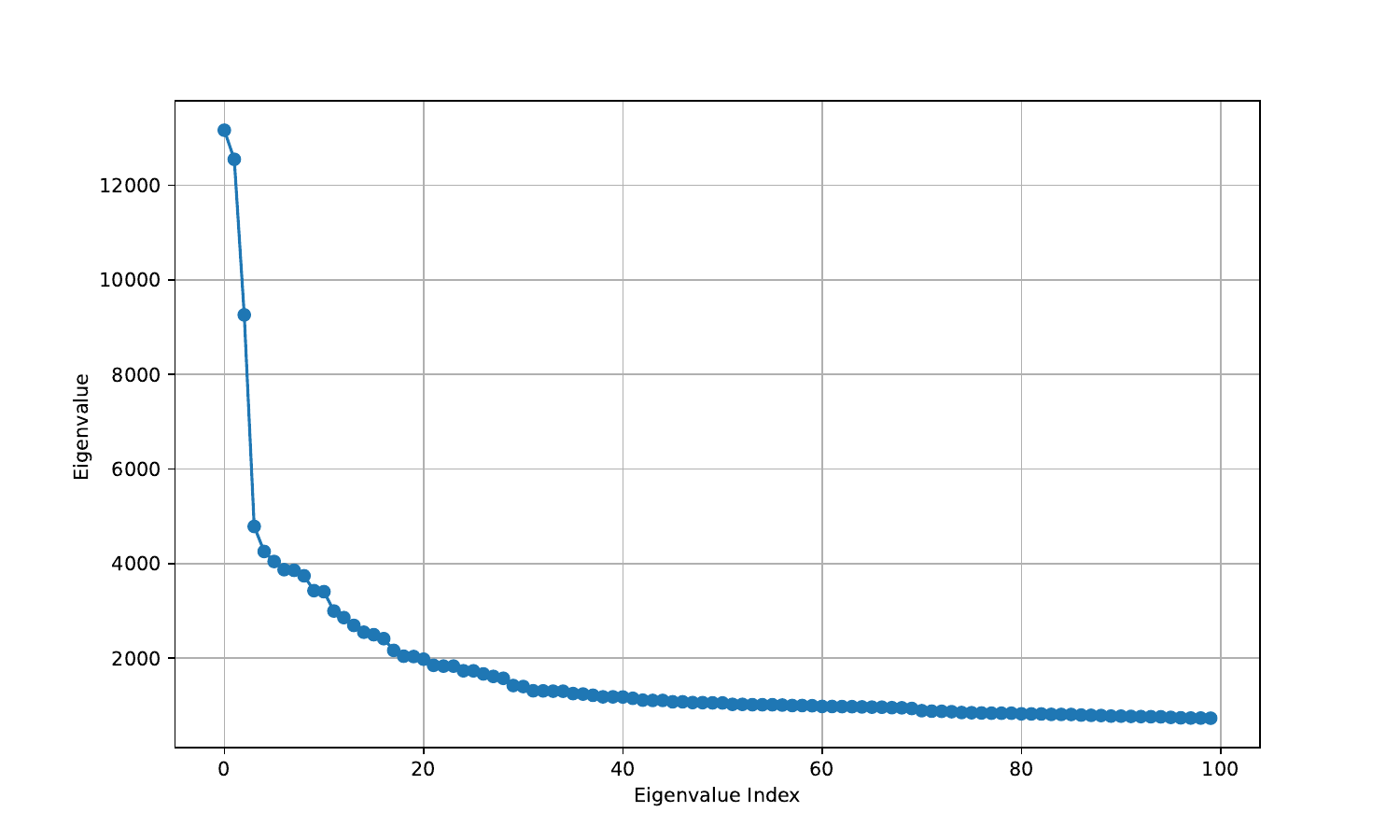}
  \caption{OBGN-Arxiv}
\end{subfigure}
\begin{subfigure}{0.24\textwidth}
  \centering
  \includegraphics[width=\linewidth]{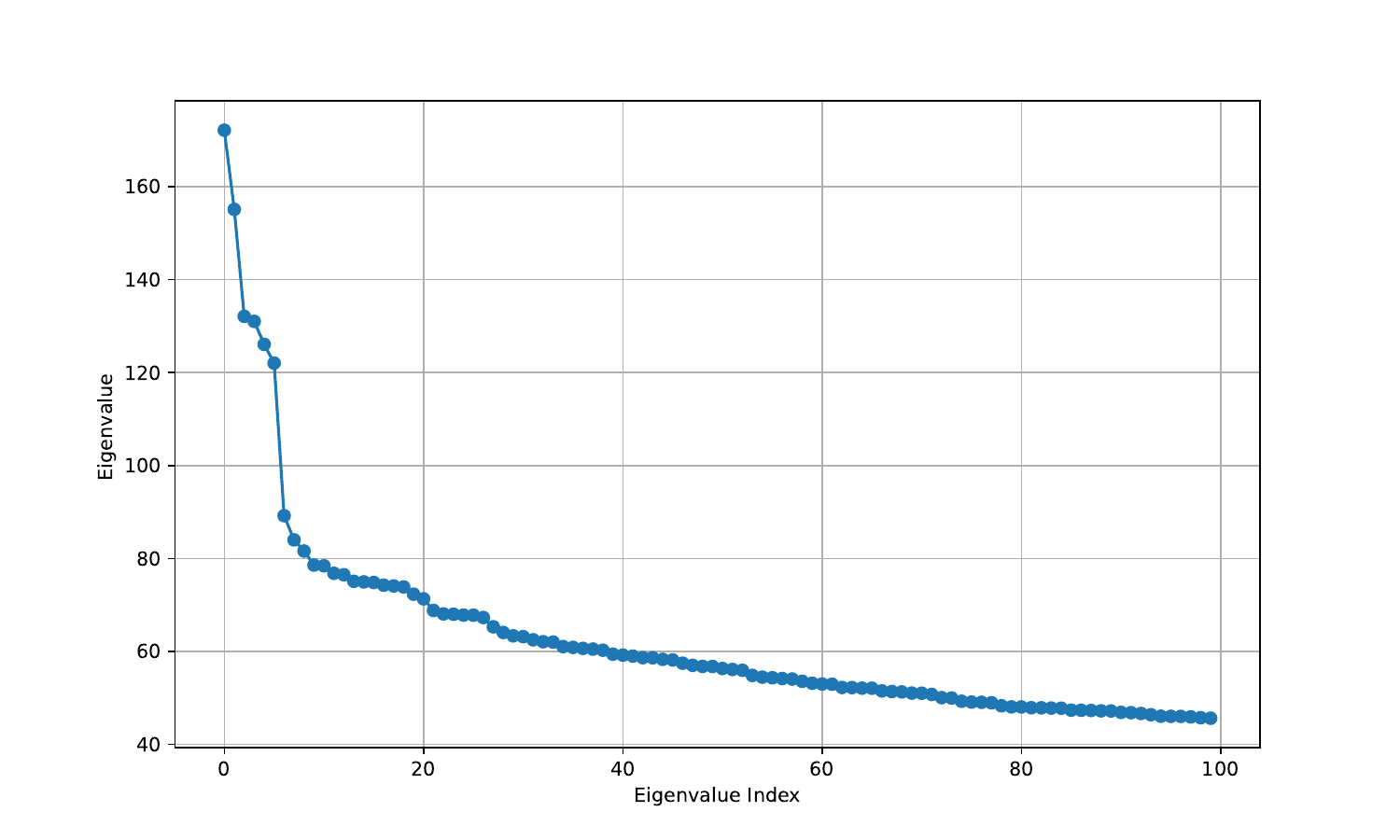}
  \caption{PudMed}
\end{subfigure}
\begin{subfigure}{0.24\textwidth}
  \centering
  \includegraphics[width=\linewidth]{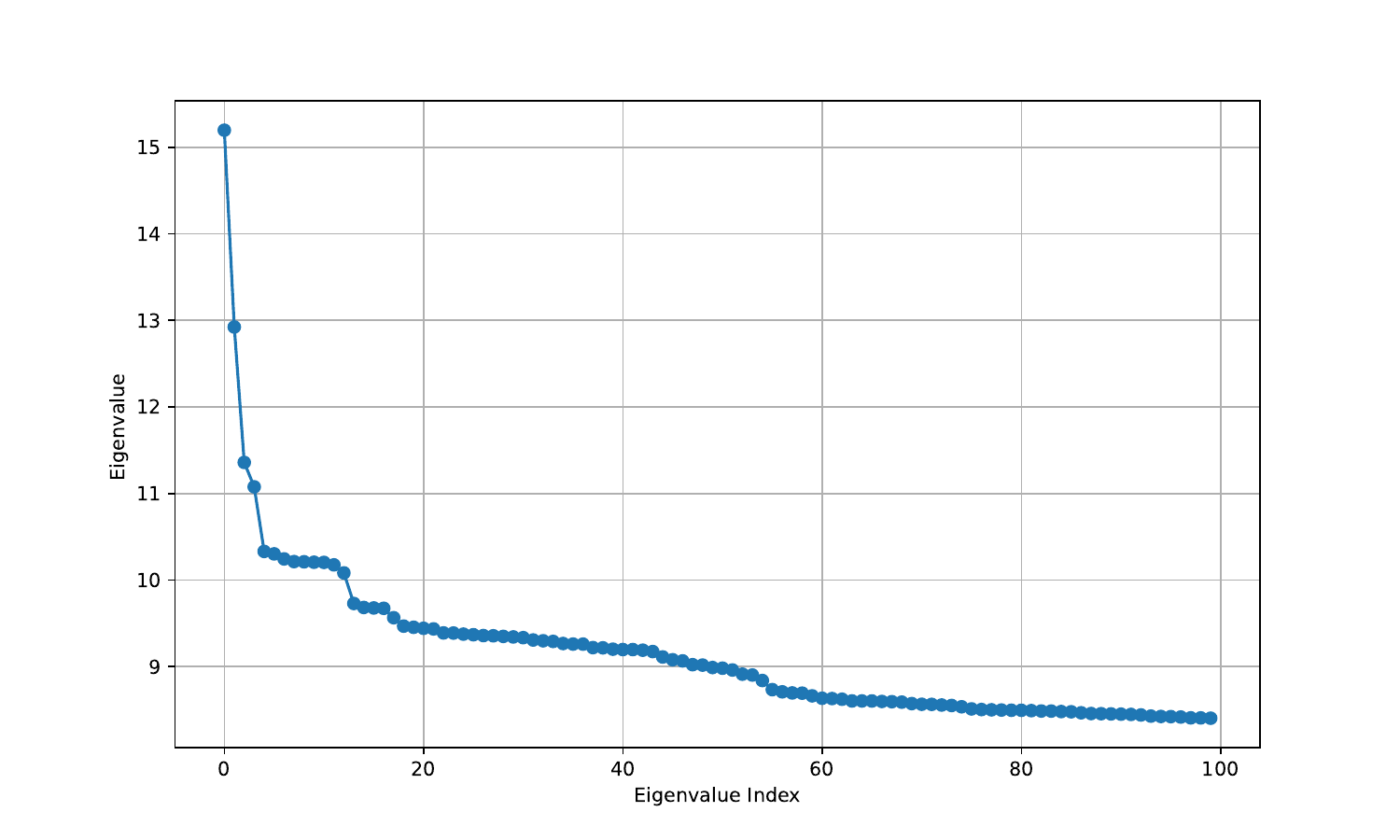}
  \caption{Roman}
\end{subfigure}
\caption{ Top $100$ eigenvalues of the graph for each dataset considered in the node classification problem. }
\label{fig:manifold_assumption}
\end{figure*}

\section{Experiment details and further experiments}
\label{sec:Appendix_Experiments}
We consider the following datasets: \textit{OGBN-Arxiv} \citep{wang2020microsoft,mikolov2013distributed}, \textit{Cora} \citep{yang2016revisiting}, \textit{CiteSeer} \citep{yang2016revisiting}, \textit{PubMed} \citep{yang2016revisiting}, \textit{Coauthors CS} \citep{shchur2018pitfalls}, \textit{Coauthors Physics} \citep{shchur2018pitfalls}, \textit{Amazon-rating} \citep{platonov2023critical}, and \textit{Roman-Empire} \citep{platonov2023critical}, details of the datasets can be found in Table \ref{tab:datasets_information}. 

All experiments were done using a \textit{NVIDIA GeForce RTX 3090}, and each set of experiments took at most $10$ hours to complete. In total, we run $10$ datasets, which amounts for around $100$ hours of GPU use. All datasets used in this paper are public, and free to use. They can be downloaded using the $pytorch$ package (\url{https://pytorch-geometric.readthedocs.io/en/latest/modules/datasets.html}), the \textit{ogb} package (\url{https://ogb.stanford.edu/docs/nodeprop/}) and the Princeton ModelNet project (\url{https://modelnet.cs.princeton.edu/}). In total, the datasets occupy around $5$ gb. However, they do not need to be all stored at the same time, as the experiments that we run can be done in series. 

\subsection{ModelNet10 and ModelNet40 graph classification tasks}\label{subsec:model}
ModelNet10 dataset \citep{wu20153d} includes 3,991 meshed CAD models from 10 categories for training and 908 models for testing as Figure \ref{fig:points-10} shows.  ModelNet40 dataset includes 38,400 training and 9,600 testing models as Figure \ref{fig:points} shows. In each model, $N$ points are uniformly randomly selected to construct graphs to approximate the underlying model, such as chairs, tables.

\begin{figure*}[ht!]
\centering
\includegraphics[trim= 80 50 80  0,clip,width=0.2  \textwidth]{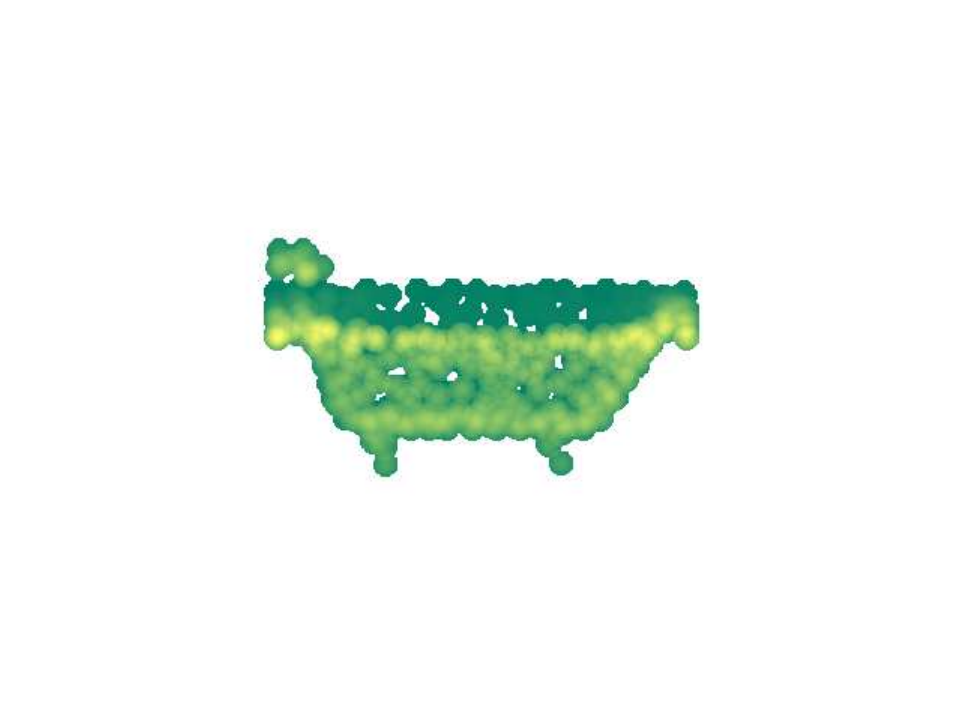}
\includegraphics[trim=90 0 100 0,width=0.15\textwidth]{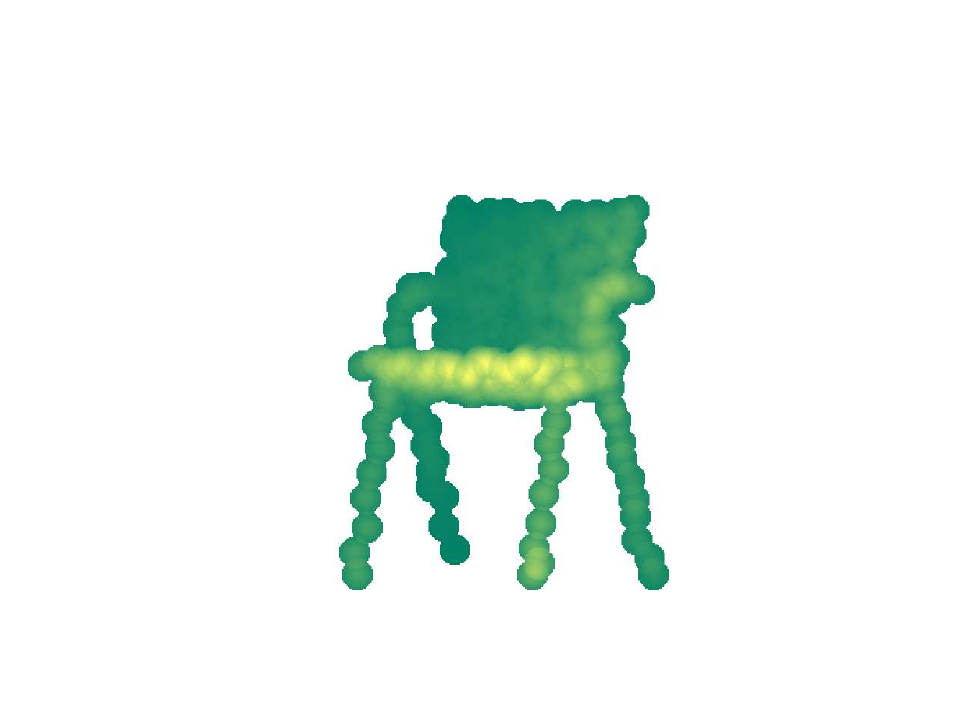}  
\includegraphics[trim=60 0 100 0,width=0.15\textwidth]{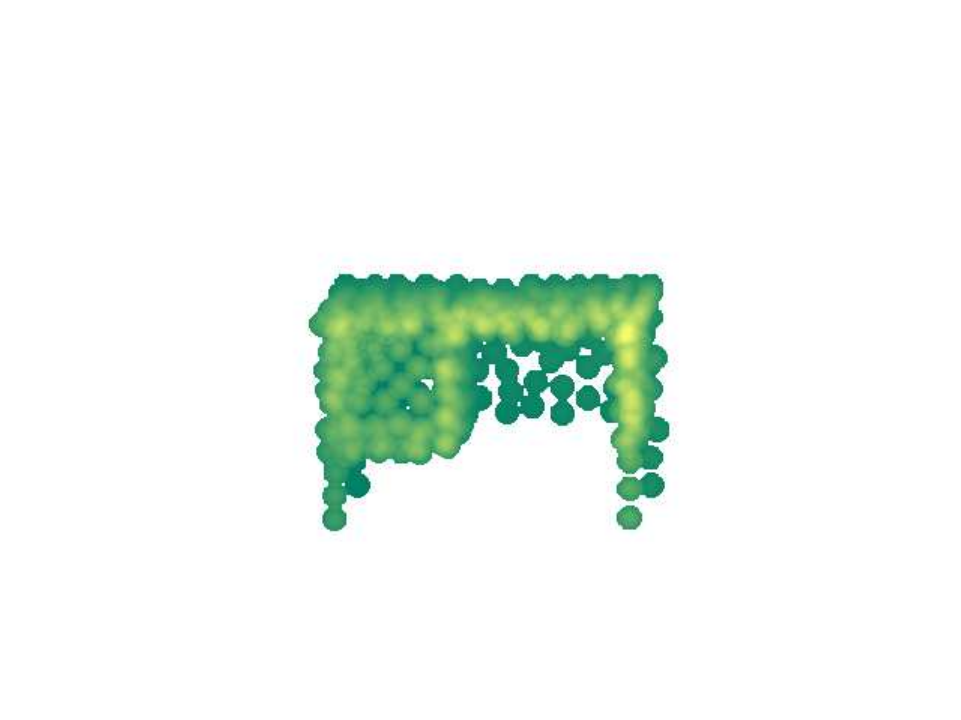} 
\includegraphics[trim=60 0 50 0,width=0.15\textwidth]{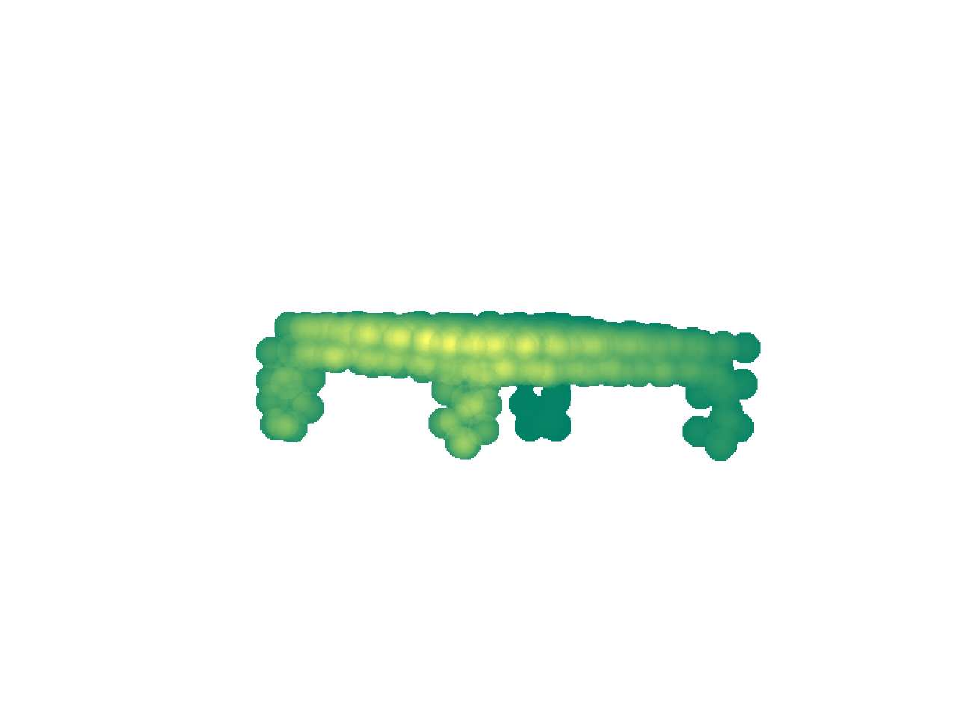}  
\includegraphics[trim=60 0 100 0,width=0.15\textwidth]{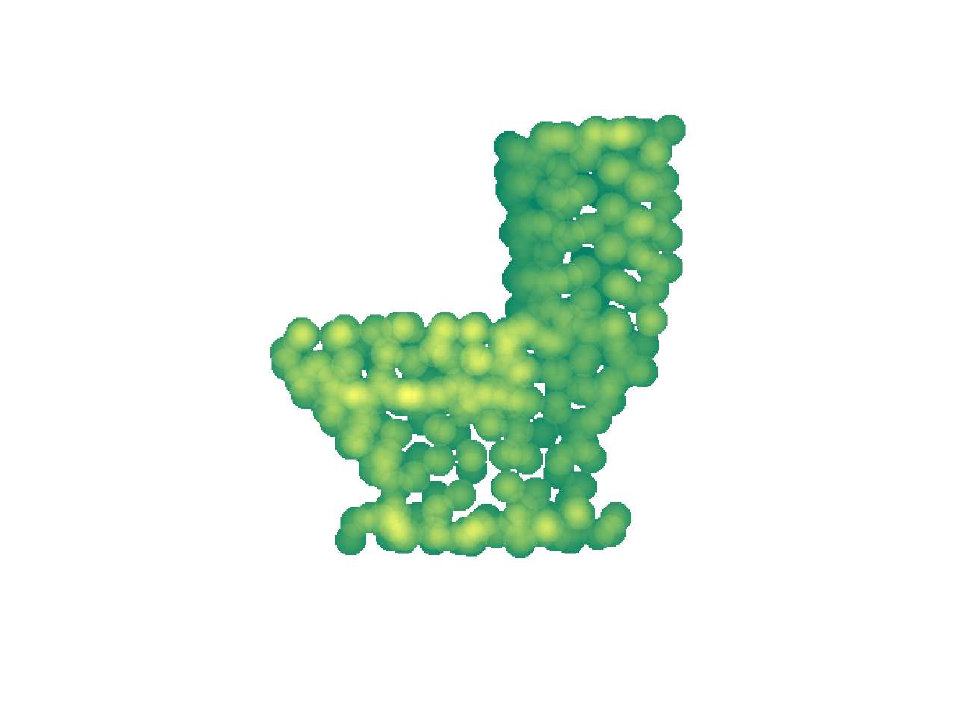}  
\includegraphics[trim=60 0 100 0,width=0.15\textwidth]{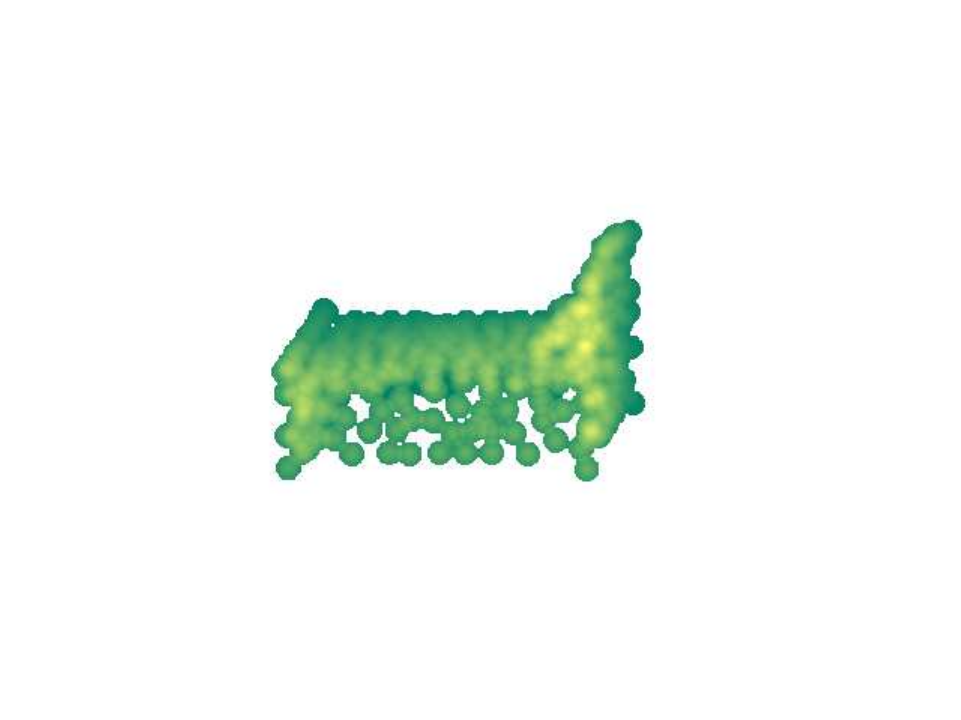} 

\caption{Point cloud models in ModelNet10 with $N=300$ sampled points in each model, corresponding to bathtub, chair, desk, table, toiler, and bed.}
\label{fig:points-10}
\end{figure*}

\begin{figure*}[ht!]
\centering
\includegraphics[ width=0.2  \textwidth]{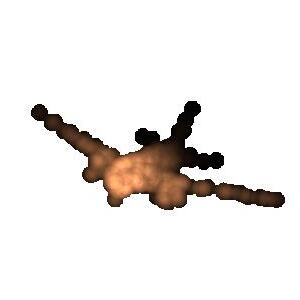}
\includegraphics[ width=0.16\textwidth]{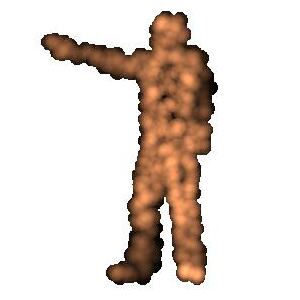}  
\includegraphics[ width=0.16\textwidth]{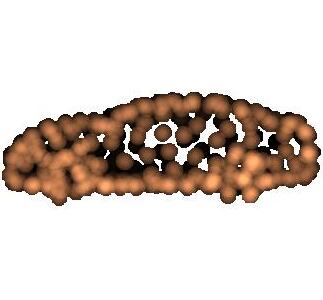} 
\includegraphics[ width=0.15\textwidth]{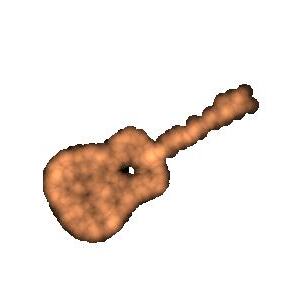}  
\includegraphics[ width=0.15\textwidth]{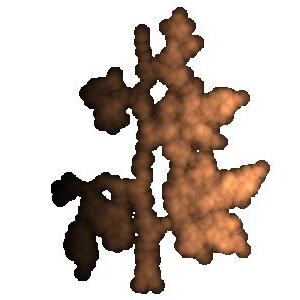}  
\includegraphics[ width=0.15\textwidth]{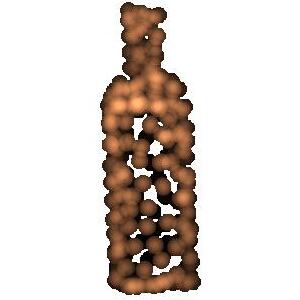} 

\caption{Point cloud models from ModelNet40 with $N=300$ sampled points in each model, corresponding to airplane, person, car, guitar, plant, and bottle. }
\label{fig:points}
\end{figure*}

\begin{figure*}[ht!]
\begin{subfigure}[b]{0.5\textwidth} 
\centering
    \includegraphics[ width= \textwidth]{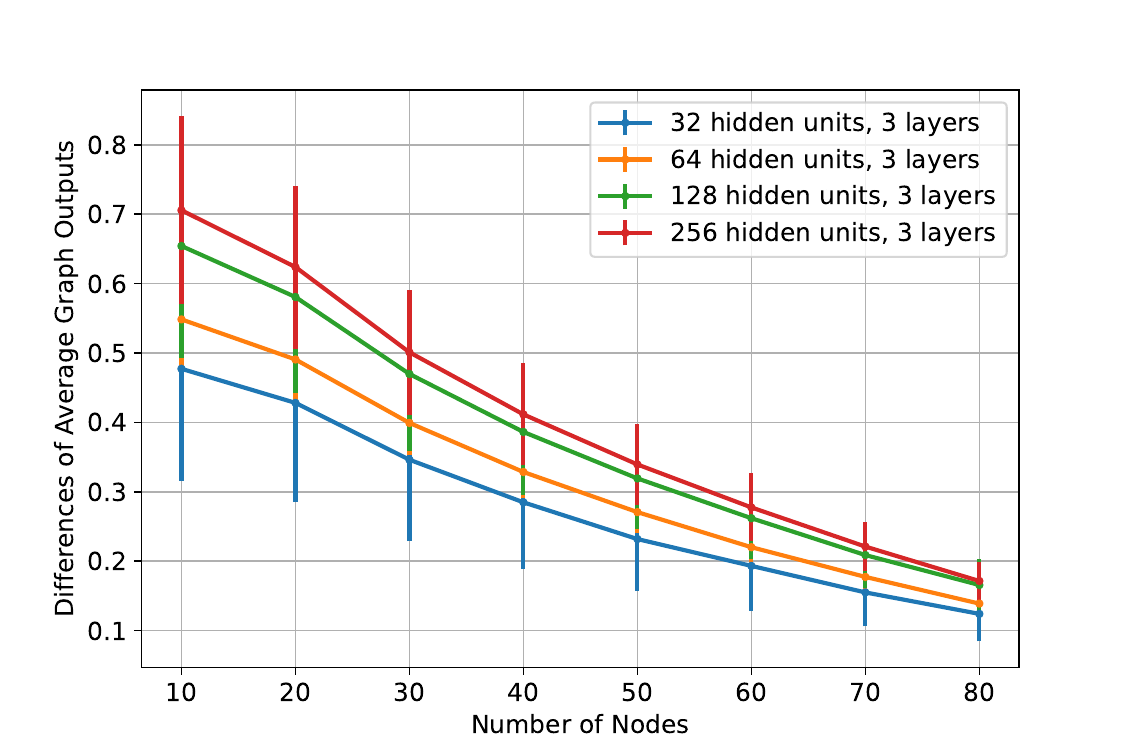}
    \caption{Differences of the outputs of 3-layer GNNs. }
    % \label{fig:diff_3layer}
\end{subfigure}
\hfill
\begin{subfigure}[b]{0.5\textwidth} 
\centering

\includegraphics[ width= \textwidth]{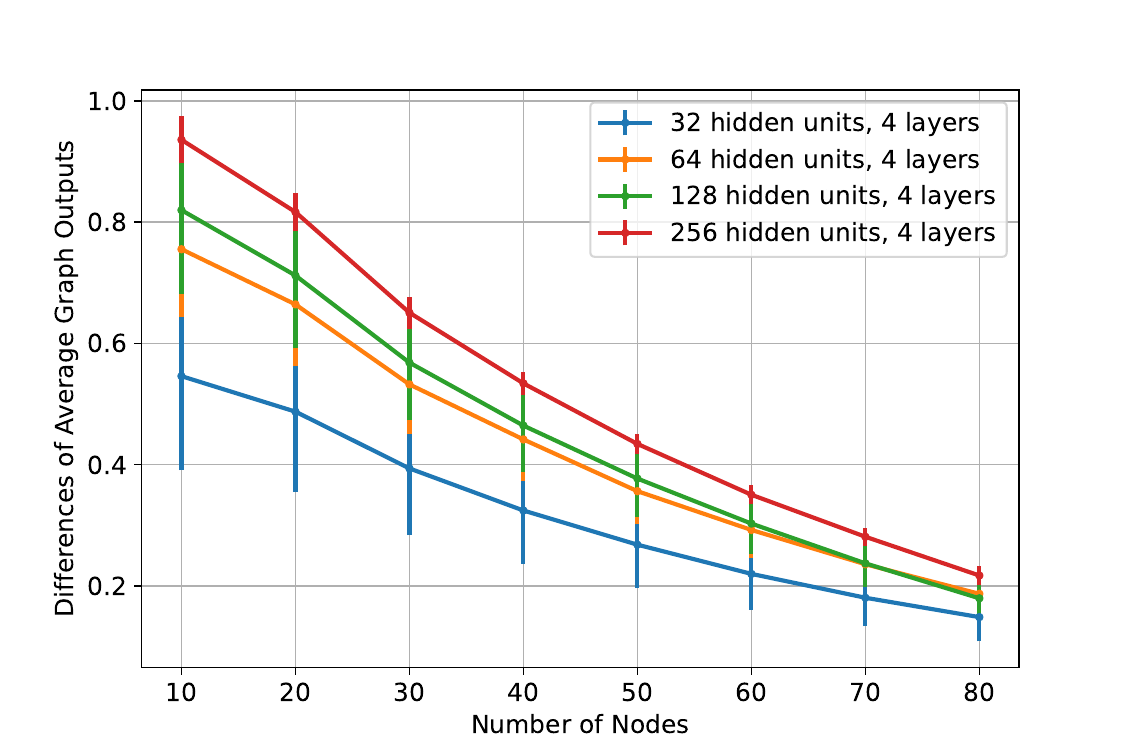}  
\caption{Differences of the outputs of 4-layer GNNs. }
    % \label{fig:diff_4layer}
\end{subfigure}
\caption{Graph outputs differences of GNNs with different architectures on ModelNet40 dataset.}
\end{figure*}
The weight function of the constructed graph is determined as \eqref{eqn:gauss_kernel} with $\epsilon = 0.1$. 
%Similarly, the weight function of an $\epsilon$-graph is calculated as \eqref{eqn:compact_kernel} with $\epsilon = 0.001$ as the threshold. 
We calculate the Laplacian matrix for each graph as the input graph shift operator. In this experiment, we implement GNNs with different numbers of layers and hidden units with $K=5$ filters in each layer. All the GNN architectures are trained by minimizing the cross-entropy loss. We implement an ADAM optimizer with the learning rate set as 0.005 along with the forgetting factors 0.9 and 0.999. We carry out the training for 40 epochs with the size of batches set as 10. We run 5 random dataset partitions and show the average performances and the standard deviation across these partitions. 

\subsection{Node classification training details and datasets}

In this section, we present the results for node classification. In this paragraph we present the common details for all datasets, we will next delve into each specific detail inside the dataset subsection that follows. 
\input{tables/experiment_details}

In all datasets, we used the graph convolutional layer \texttt{GCN}, and trained for $1000$ epochs. For the optimizer, we used $\texttt{AdamW}$, with using a learning rate of $0.01$, and $0$ weight decay. We trained using the graph convolutional layer, with a varying number of layers and hidden units. For dropout, we used $0.5$. We trained using the cross-entropy loss. In all cases, we trained $2$ and $3$ layered GNNs. 

To compute the linear approximation in the plots, we used the mean squared error estimator of the form
\begin{align}
    \mathbf{y} = s * \log(\mathbf{n}) +p.
\end{align}
Where $s$ is the slope, $p$ is the point, and $\mathbf{n}$ is the vector with the nodes in the training set for each experiment. Note that we repeated each experiment for $10$ independent runs. In all experiments, we compute the value of $s$ and $p$ that minimize the mean square error over the mean of the experiment runs, and we compute the Pearson correlation index over those values. 

Our experiment shows that our bound shows the same rate dependency as the experiments. That is to say, in the logarithmic scale, the generalization gap of GNNs is linear with respect to the logarithm of the number of nodes. In most cases, the Pearson correlation index is above $0.9$ in absolute value, which indicates a strong linear relationship. We noticed that the linear relationship changes the slope in the overfitting regime, and in the non-overfitting regime. That is to say, when the GNN is overfitting the training set, the generalization gap decreases at a much slower rate than it does with the GNN does not have the capacity to do so. Therefore, in the case in which the GNN overfits the training set for all nodes when computed $s$ using all the samples in the experiment. On the other hand, when the number of nodes is large enough that the GNN cannot overfit the training set, then we computed the $s$ and $p$ with the nodes in the non overfitting regime. 

\subsection{Spectral Continuity Constant Regularizer} \label{appendix_subsection:SCCR}
We add a regularization term to the loss to better control the value of the spectral continuity constant (defined in Assumption \ref{ass:low-pass}) while training. To do so, given a convolutional filter $\bbh\in \reals^{K}$, its associated spectral continuity constant is
\begin{align}\label{eqn:regulizer}
    R(\bbh) =\sum_{k=0}^{K-1}  k |h_{k}|\lambda_{max}^{k-1},
\end{align}
Where $\lambda_{max}$ is the largest eigenvalue of the graph $\bbG$.

\subsubsection{Arxiv dataset}
For this datasets, we trained $2,3,4$ layered GNN. We also used a learning rate scheduler $\texttt{ReduceLROnPlateau}$ with mode min, factor $0.5$, patience $100$ and a minimum learning rate of $0.001$. 
\begin{figure*}[!ht]
    \centering
    \begin{subfigure}{0.33\textwidth}
        \centering
        \includegraphics[width=\linewidth]{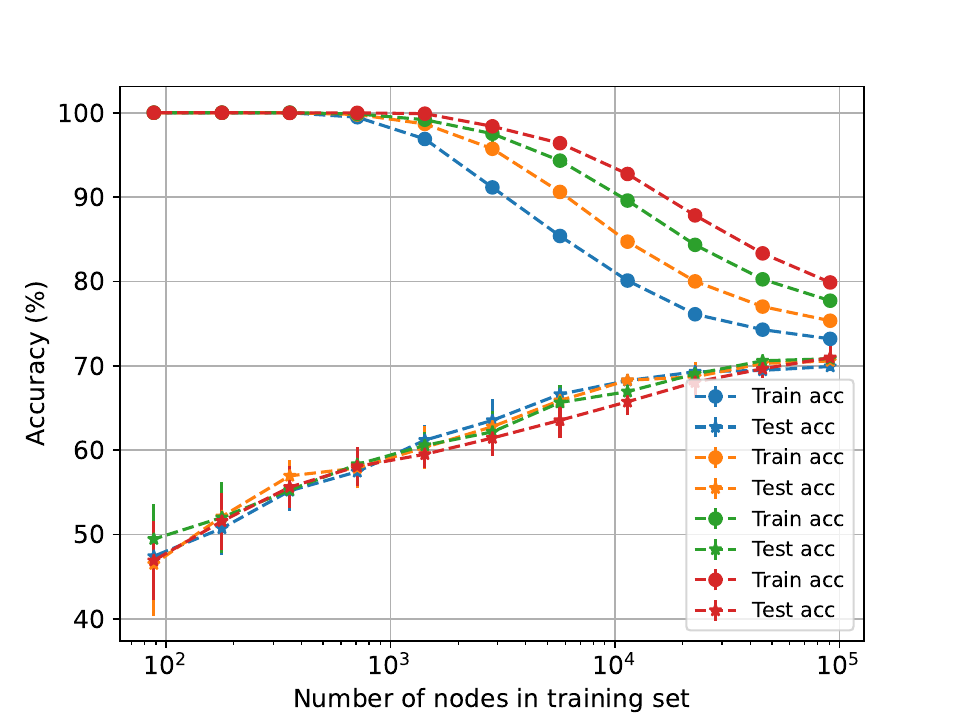}
        % \caption{Two Layers}
        % \label{fig:arxiv_acc_1}
    \end{subfigure}%
    \begin{subfigure}{0.33\textwidth}
        \centering
        \includegraphics[width=\linewidth]{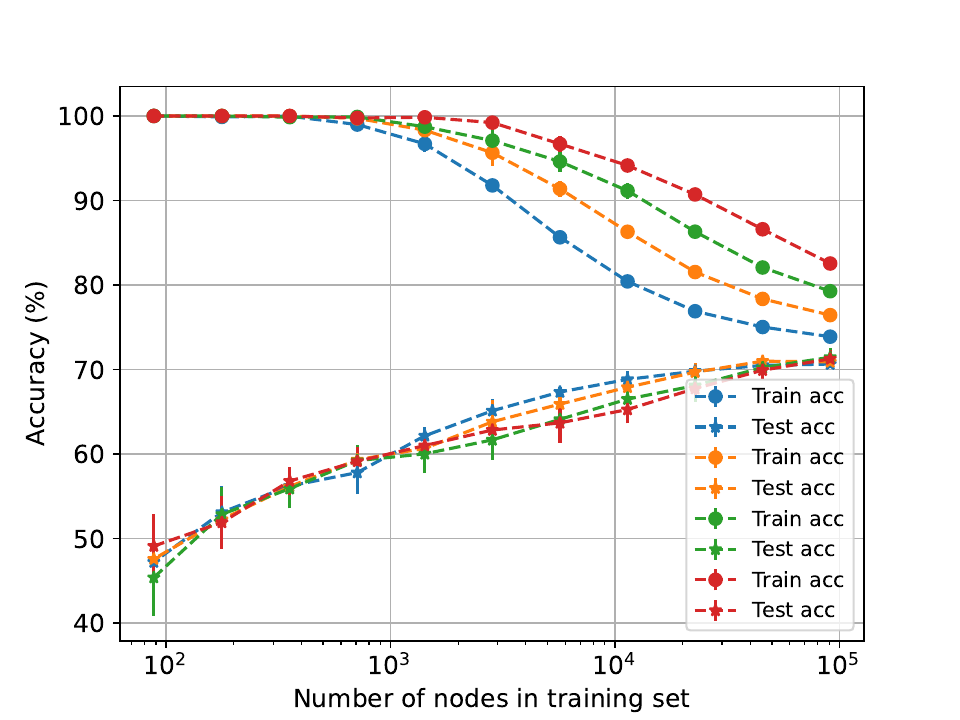}
        % \caption{Three Layers}
        % \label{fig:arxiv_acc_3}
    \end{subfigure}%
    \begin{subfigure}{0.33\textwidth}
        \centering
        \includegraphics[width=\linewidth]{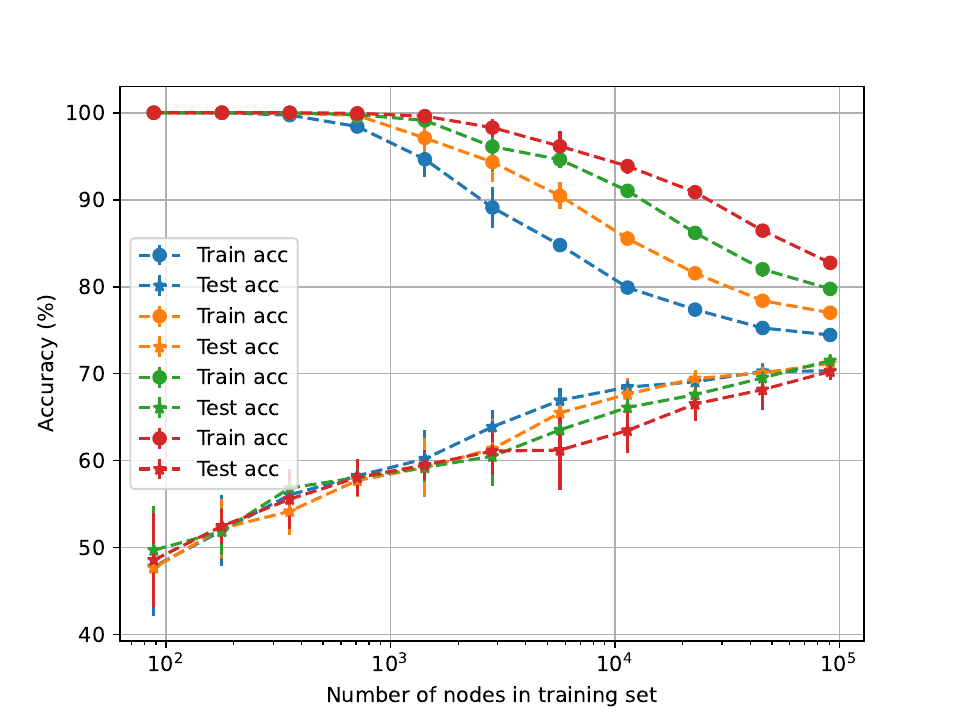}
        % \caption{Four Layers}
        % \label{fig:arxiv_gen_gap}
    \end{subfigure}
    \begin{subfigure}{0.33\textwidth}
        \centering
        \includegraphics[width=\linewidth]{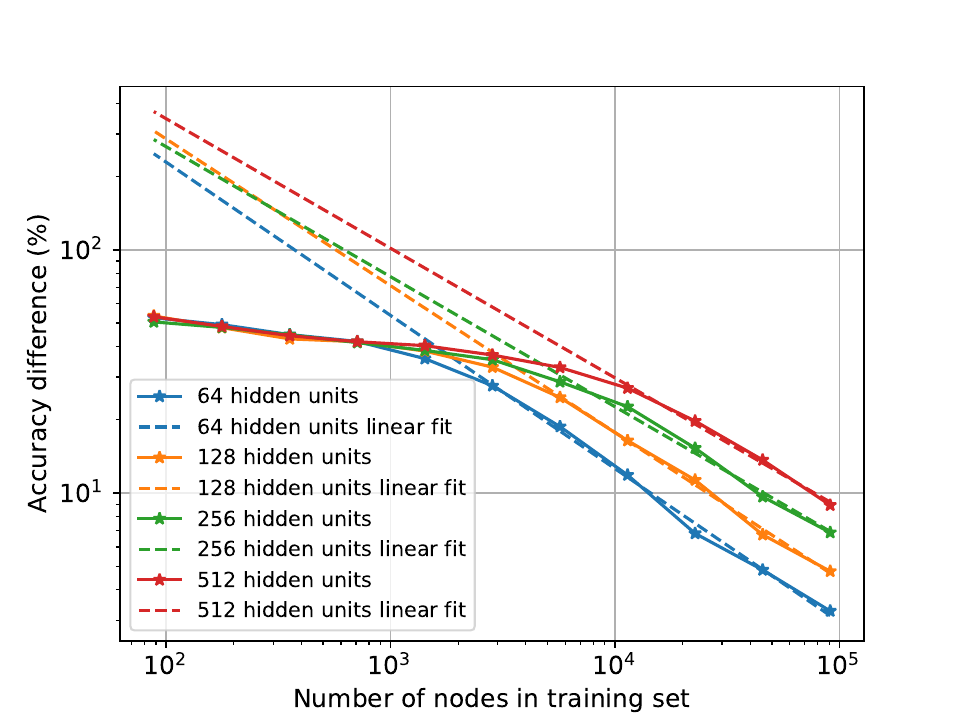}
        \caption{Two Layers}
        % \label{fig:arxiv_acc_1}
    \end{subfigure}%
    \begin{subfigure}{0.33\textwidth}
        \centering
        \includegraphics[width=\linewidth]{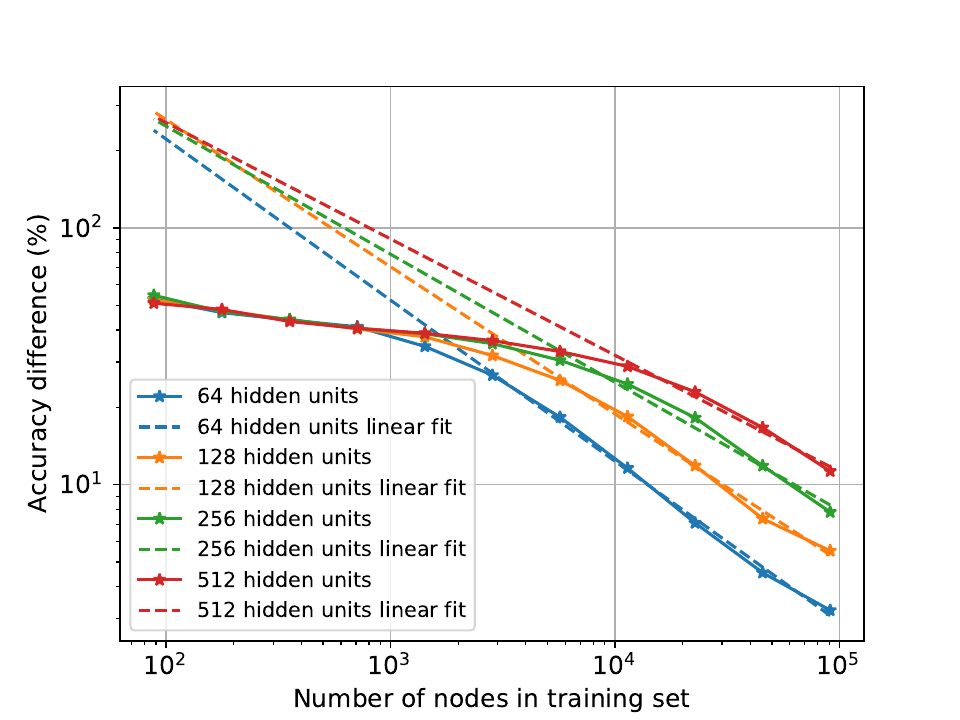}
        \caption{Three Layers}
        % \label{fig:arxiv_acc_3}
    \end{subfigure}%
    \begin{subfigure}{0.33\textwidth}
        \centering
        \includegraphics[width=\linewidth]{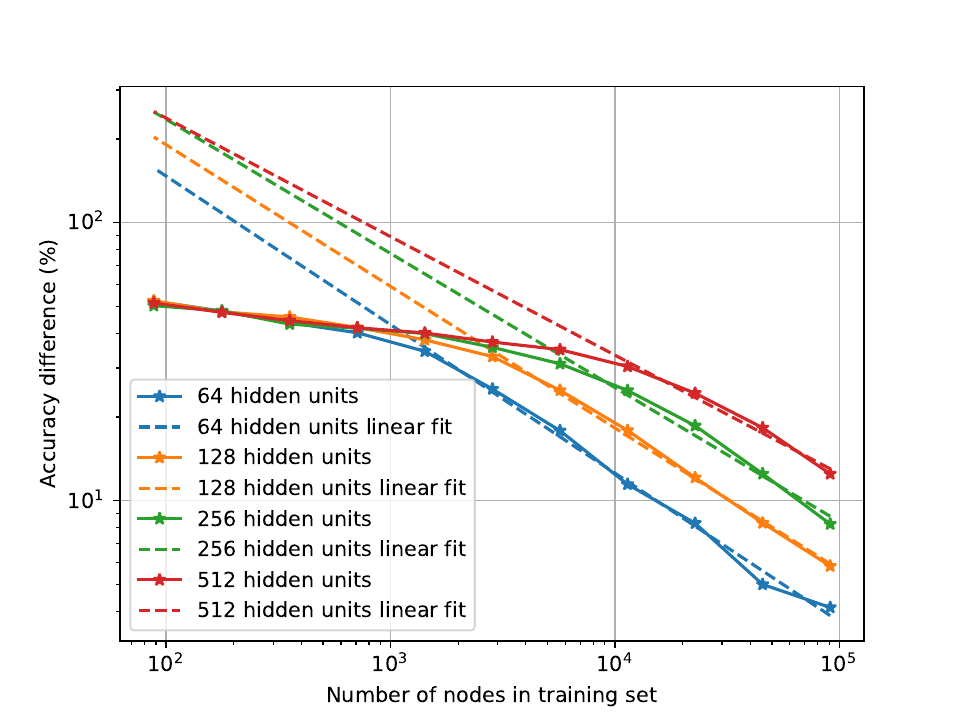}
        \caption{Four Layers}
        % \label{fig:arxiv_gen_gap}
    \end{subfigure}
    \caption{Generalization gap for the OGBN-Arxiv dataset on the accuracy as a function of the number of nodes in the training set.  }
    \label{fig:arxiv_loss}
\end{figure*}

\begin{figure*}
    \centering
    \begin{subfigure}{0.33\textwidth}
        \centering
        \includegraphics[width=\linewidth]{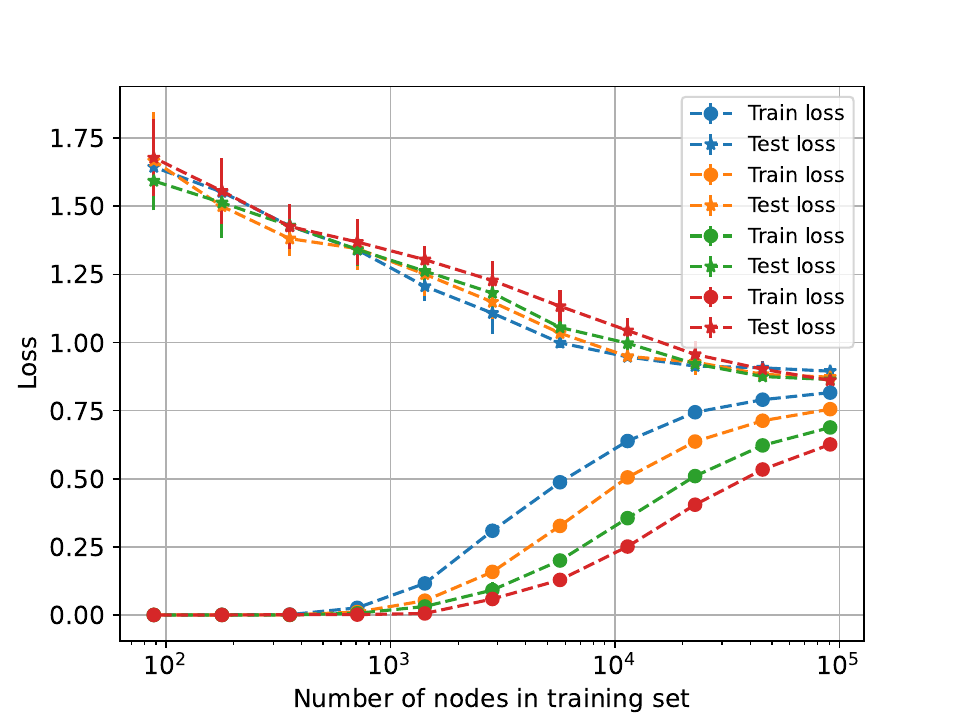}
        % \caption{Two Layers}
        % \label{fig:arxiv_acc_1}
    \end{subfigure}%
    \begin{subfigure}{0.33\textwidth}
        \centering
        \includegraphics[width=\linewidth]{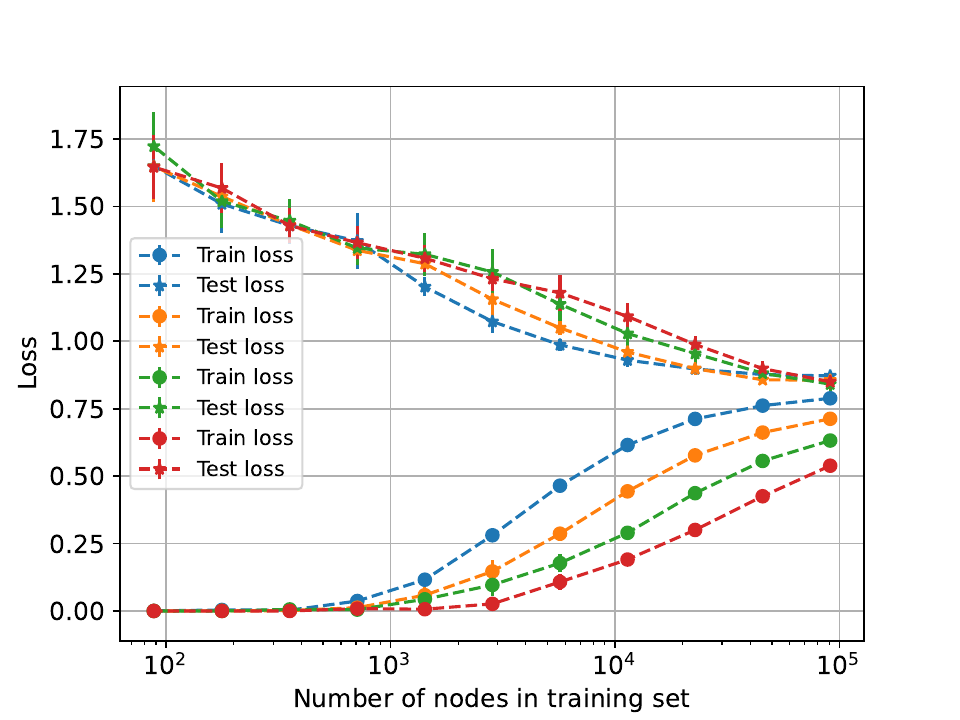}
        % \caption{Three Layers}
        % \label{fig:arxiv_acc_3}
    \end{subfigure}%
    \begin{subfigure}{0.33\textwidth}
        \centering
        \includegraphics[width=\linewidth]{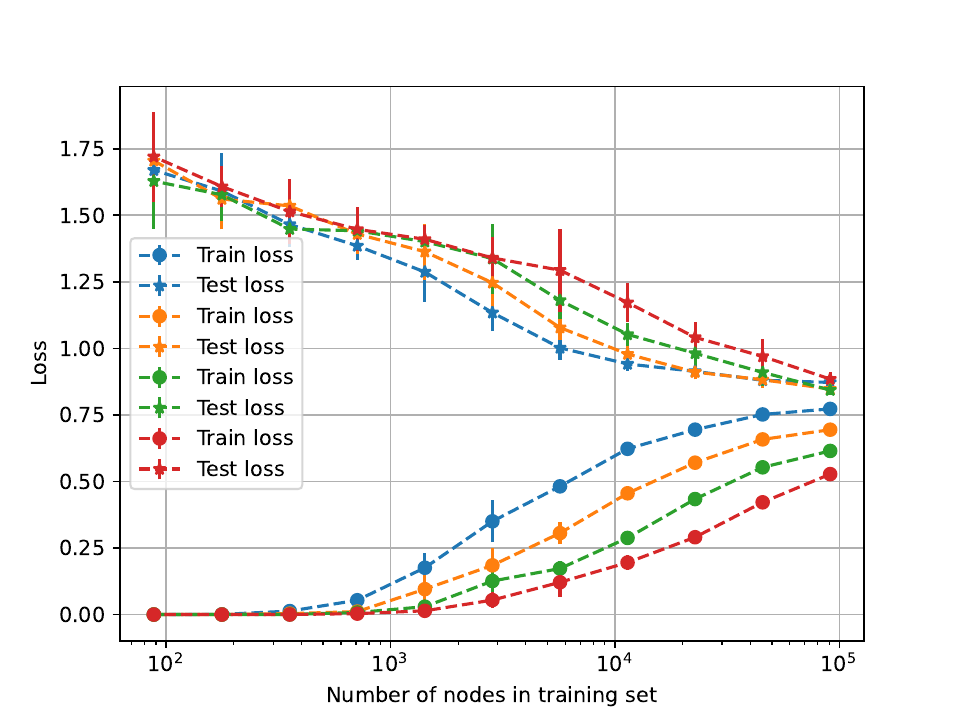}
        % \caption{Four Layers}
        % \label{fig:arxiv_gen_gap}
    \end{subfigure}
    \begin{subfigure}{0.33\textwidth}
        \centering
        \includegraphics[width=\linewidth]{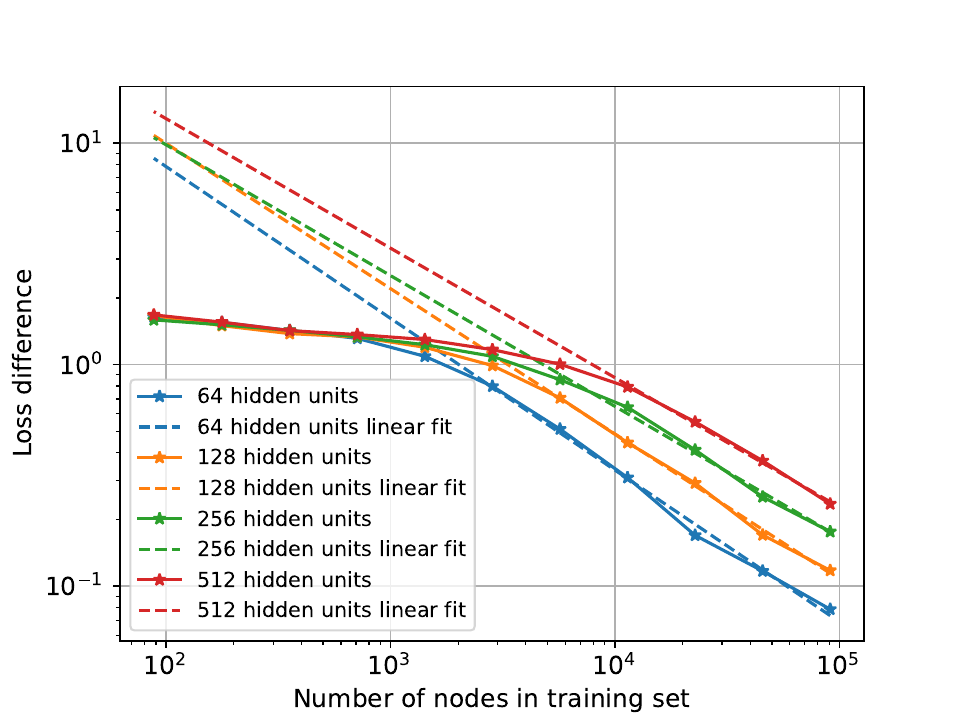}
        \caption{Two Layers}
        % \label{fig:arxiv_acc_1}
    \end{subfigure}%
    \begin{subfigure}{0.33\textwidth}
        \centering
        \includegraphics[width=\linewidth]{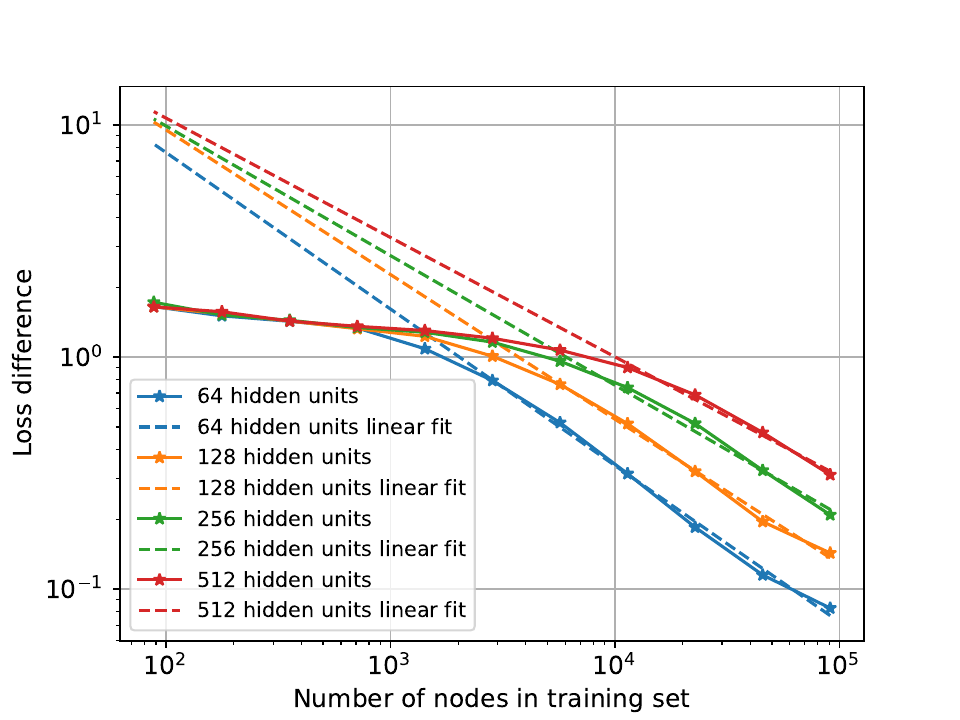}
        \caption{Three Layers}
        % \label{fig:arxiv_acc_3}
    \end{subfigure}%
    \begin{subfigure}{0.33\textwidth}
        \centering
        \includegraphics[width=\linewidth]{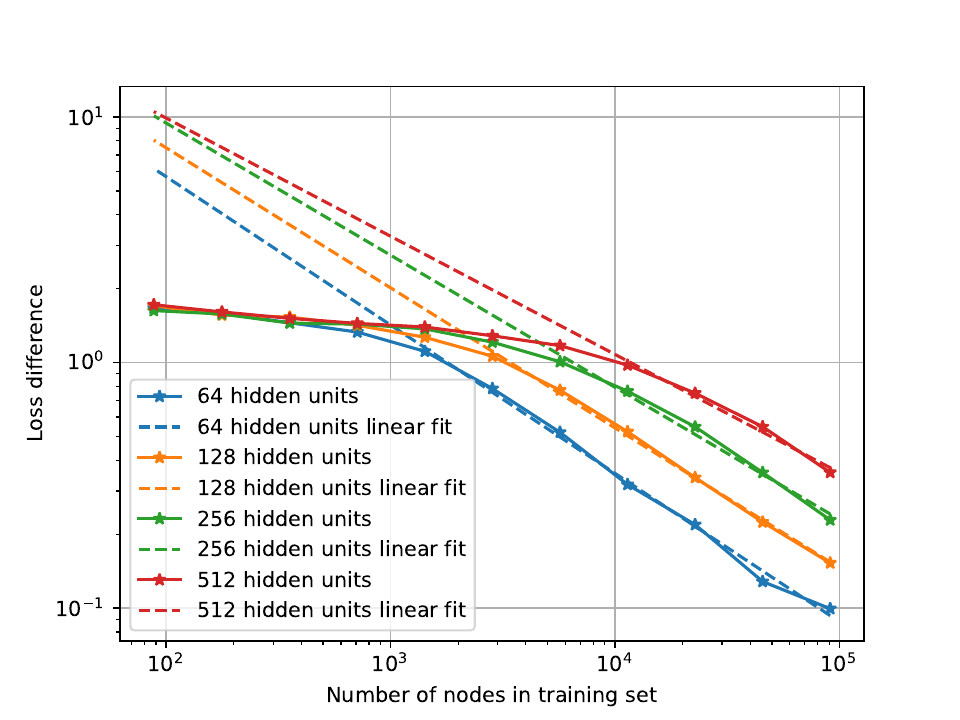}
        \caption{Four Layers}
        % \label{fig:arxiv_gen_gap}
    \end{subfigure}
    \caption{Generalization gap for the OGBN-arxiv dataset on the loss (cross-entropy) as a function of the number of nodes in the training set.  }
    \label{fig:arxiv_acc}
\end{figure*}

\input{tables/arxiv}

\begin{figure*}
    \centering
    \includegraphics[width=\textwidth]{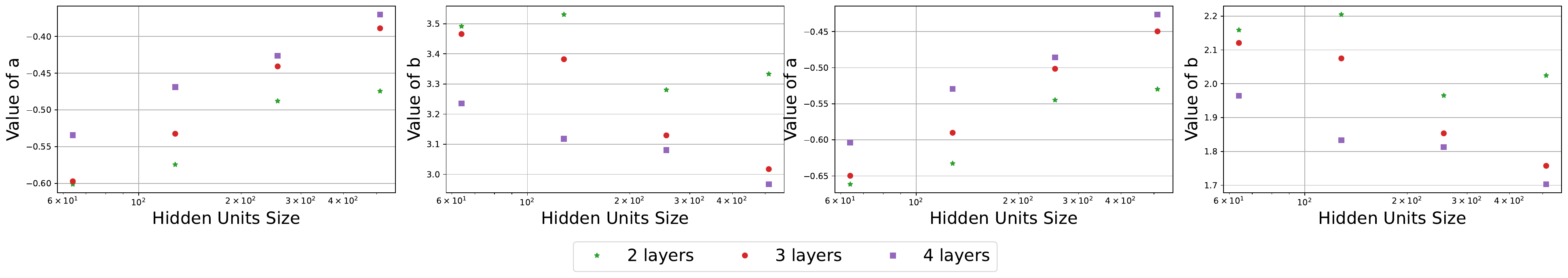}
    \begin{subfigure}{0.49\textwidth}
        \caption{Accuracy}
    \end{subfigure}
    \begin{subfigure}{0.49\textwidth}
        \caption{Loss}
    \end{subfigure}
    \caption{Values of slope (a) and point (b) corresponding to the linear fit ($a*\texttt{log}(N)+b$) of Figures \ref{fig:arxiv_acc} and \ref{fig:arxiv_loss}.  } 
    \label{fig:a_b_arxiv}
\end{figure*}

\subsubsection{Cora dataset}
For the Cora dataset, we used the standard one, which can be obtained running $\texttt{torch\_geometric.datasets.Planetoid(root="./data",name='Cora')}$. 

\begin{figure*}[ht!]
    \centering
    \begin{subfigure}{0.32\textwidth}
        \centering
        \includegraphics[width=\linewidth]{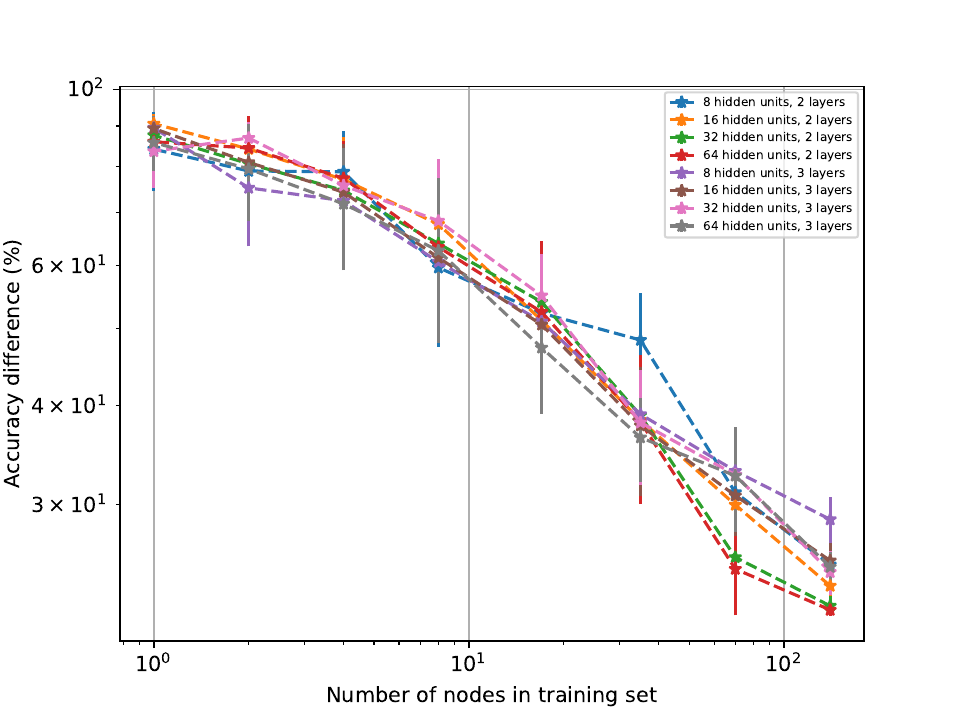}
        % \caption{Accuracy Difference in Cora Dataset.}
        % \label{fig:cora_acc_1}
    \end{subfigure}%
    \begin{subfigure}{0.32\textwidth}
        \centering
        \includegraphics[width=\linewidth]{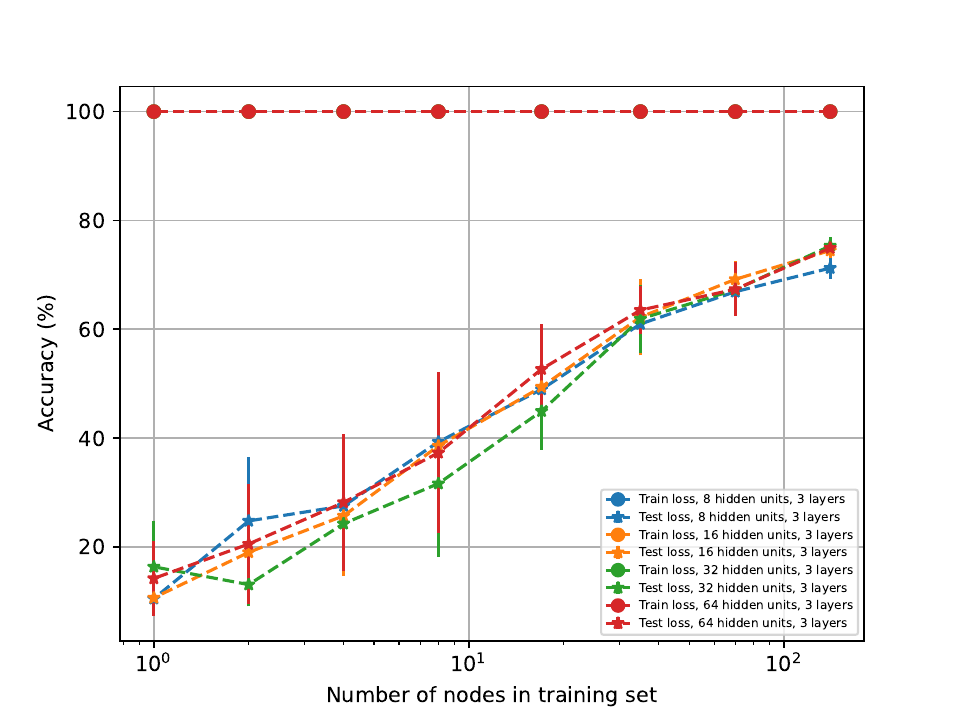}
        % \caption{Accuracy in Train/Test in Cora Dataset for a $3$ Layers GNN.}
        % \label{fig:cora_acc_layer_3}
    \end{subfigure}%
    \begin{subfigure}{0.32\textwidth}
        \centering
        \includegraphics[width=\linewidth]{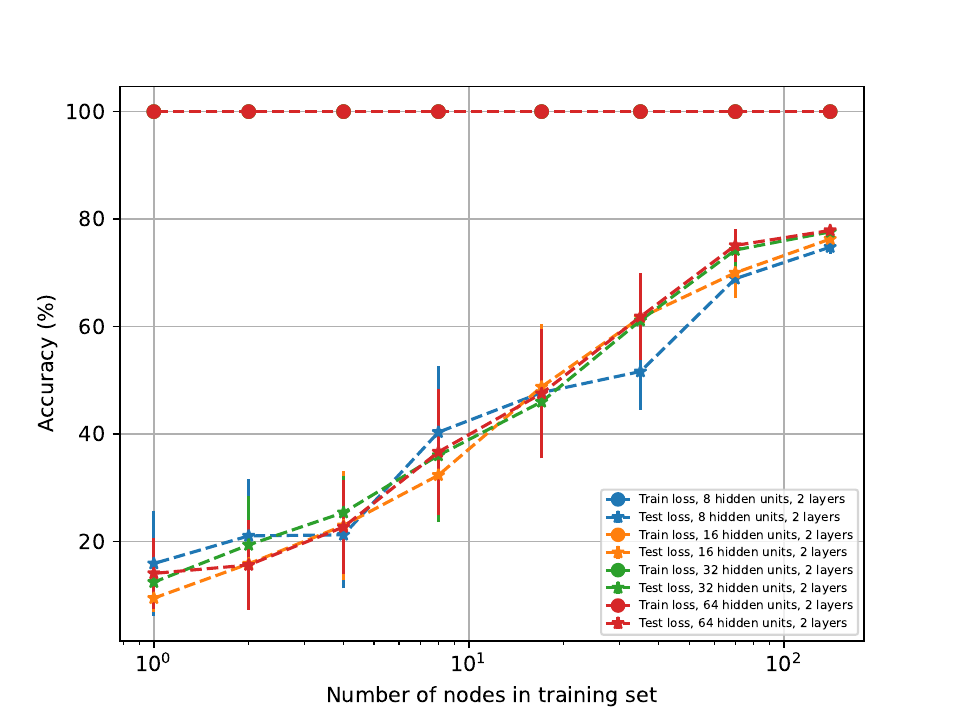}
        % \caption{Accuracy in Train/Test in Cora Dataset for a $2$ Layers GNN.}
        % \label{fig:cora_acc_layer_2}
    \end{subfigure}
        \begin{subfigure}{0.32\textwidth}
        \centering
        \includegraphics[width=\linewidth]{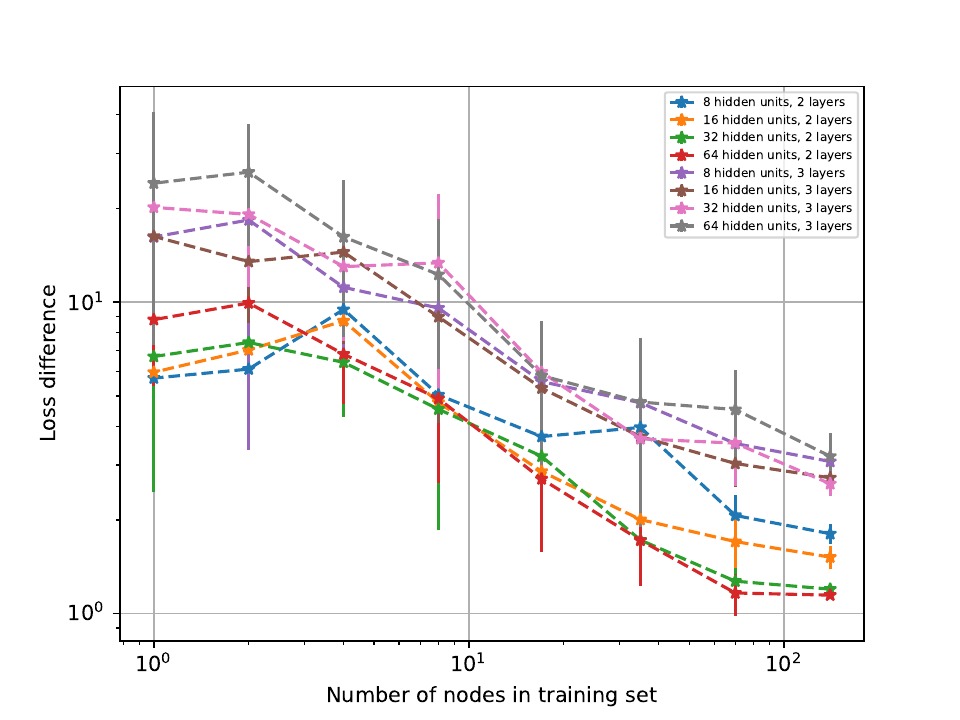}
        \caption{Generalization Gap.}
        \label{fig:cora_acc_3}
    \end{subfigure}
    \begin{subfigure}{0.32\textwidth}
        \centering
        \includegraphics[width=\linewidth]{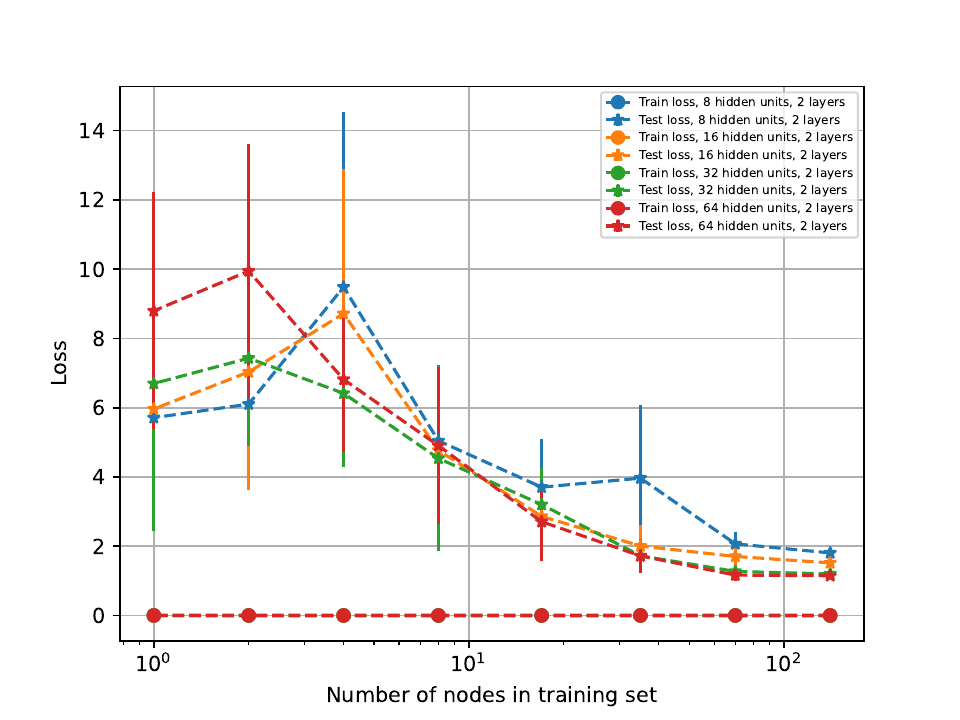}
        \caption{$2$ Layers}
        \label{fig:cora_loss_layer_2}
    \end{subfigure}
    \begin{subfigure}{0.32\textwidth}
        \centering
        \includegraphics[width=\linewidth]{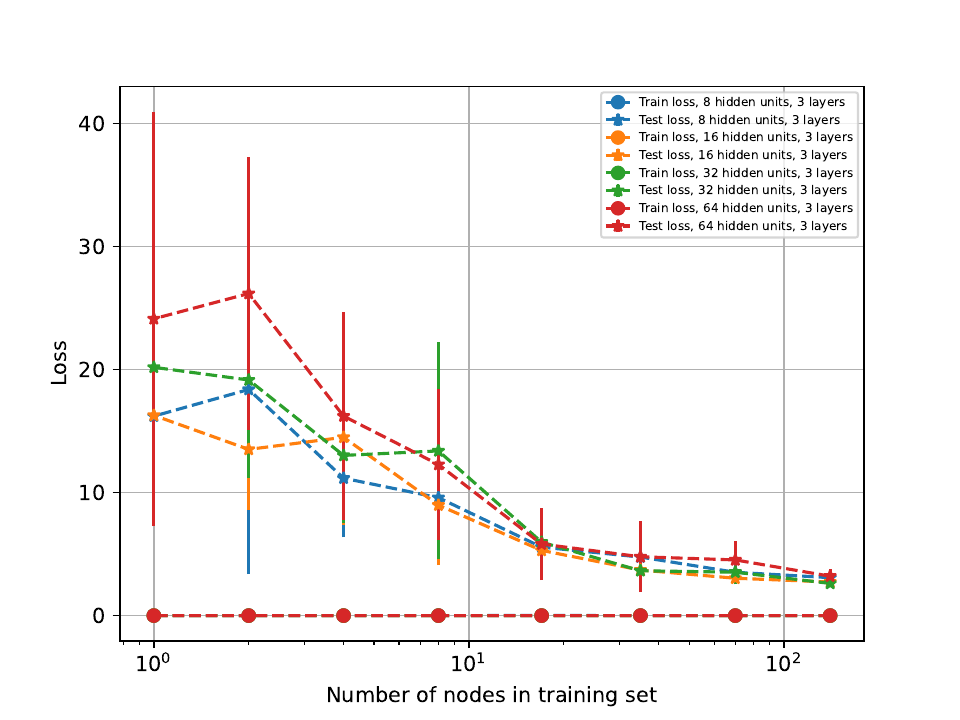}
        \caption{$3$ Layers}
        % \label{fig:cora_loss_layer_3}
    \end{subfigure}
    \caption{Generalization gap, testing, and training losses with respect to the number of nodes in the Cora dataset. The top row is in accuracy, and the bottom row is the cross-entropy loss. }
\end{figure*}

\input{tables/cora}

\subsubsection{CiteSeer dataset}
For the CiteSeer dataset, we used the standard one, which can be obtained running $\texttt{torch\_geometric.datasets.Planetoid(root="./data",name='CiteSeer')}$.

\begin{figure*}[ht!]
    \centering
    \begin{subfigure}{0.32\textwidth}
        \centering
        \includegraphics[width=\linewidth]{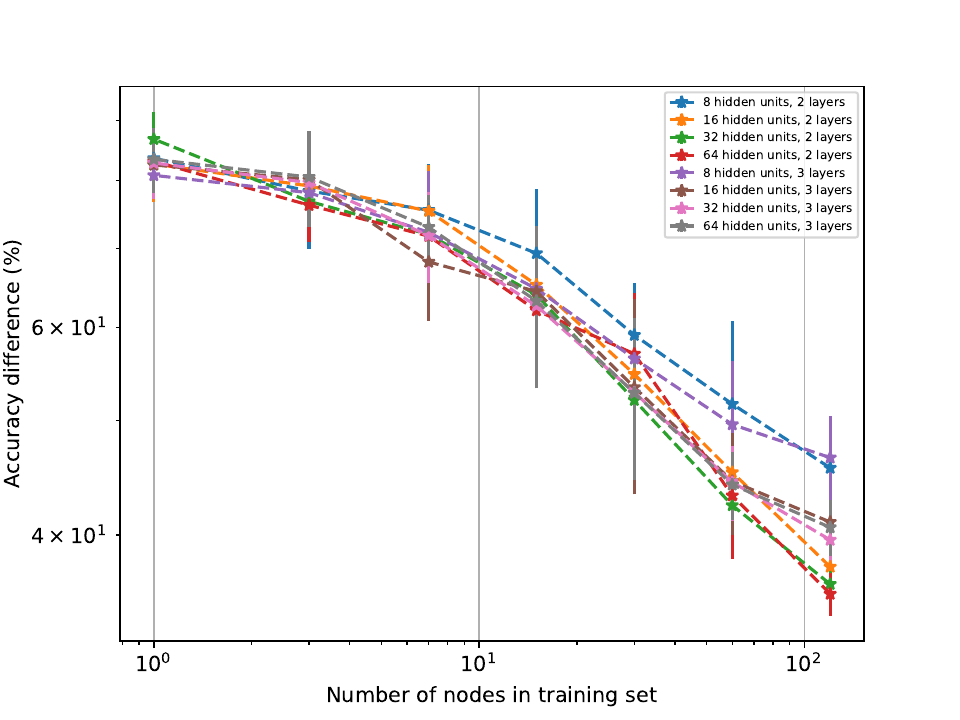}
        % \caption{Accuracy Difference in CiteSeer Dataset.}
        % \label{fig:CiteSeer_acc_1}
    \end{subfigure}%
    \begin{subfigure}{0.32\textwidth}
        \centering
        \includegraphics[width=\linewidth]{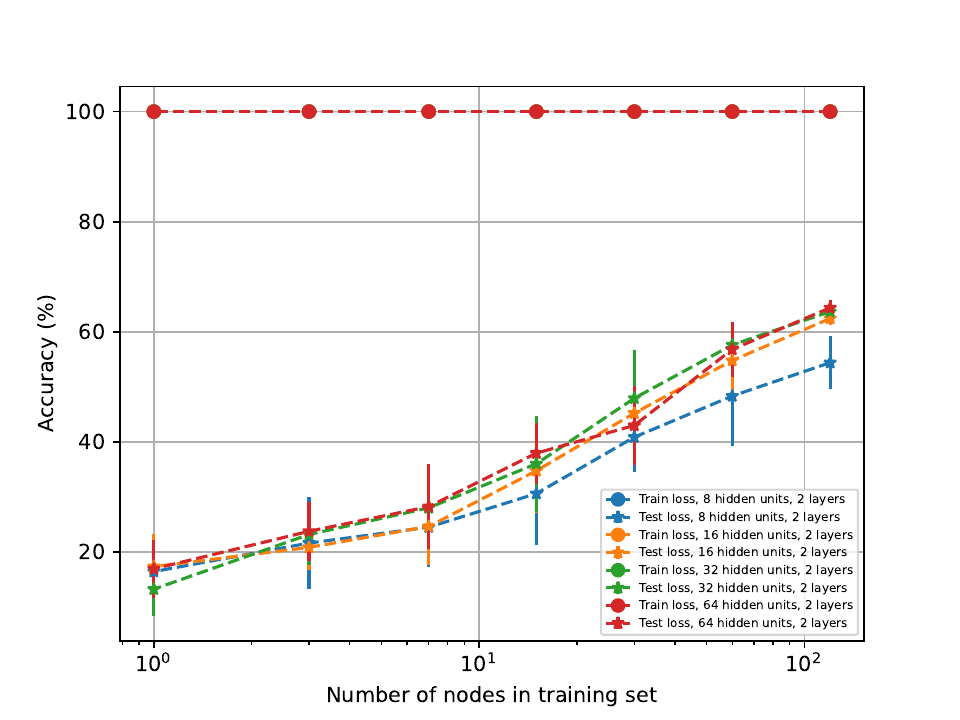}
        % \caption{Accuracy in Train/Test in CiteSeer Dataset for a $2$ Layers GNN.}
        % \label{fig:CiteSeer_acc_layer_2}
    \end{subfigure}%
        \begin{subfigure}{0.32\textwidth}
        \centering
        \includegraphics[width=\linewidth]{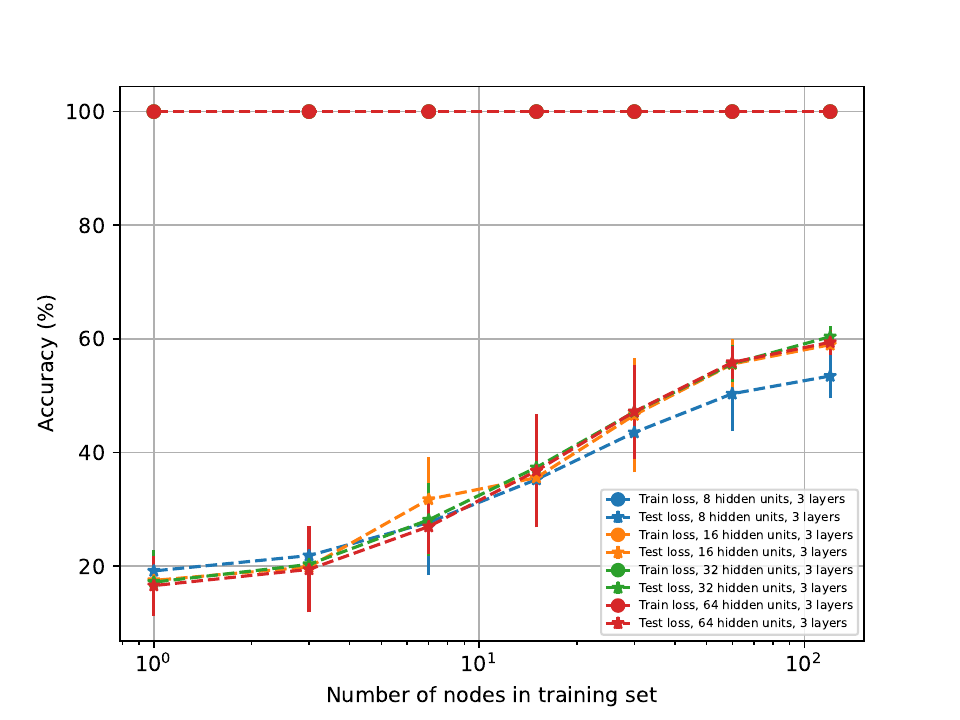}
        % \caption{Accuracy in Train/Test in CiteSeer Dataset for a $3$ Layers GNN.}
        % \label{fig:CiteSeer_acc_layer_3}
    \end{subfigure}\\
    \begin{subfigure}{0.32\textwidth}
        \centering
        \includegraphics[width=\linewidth]{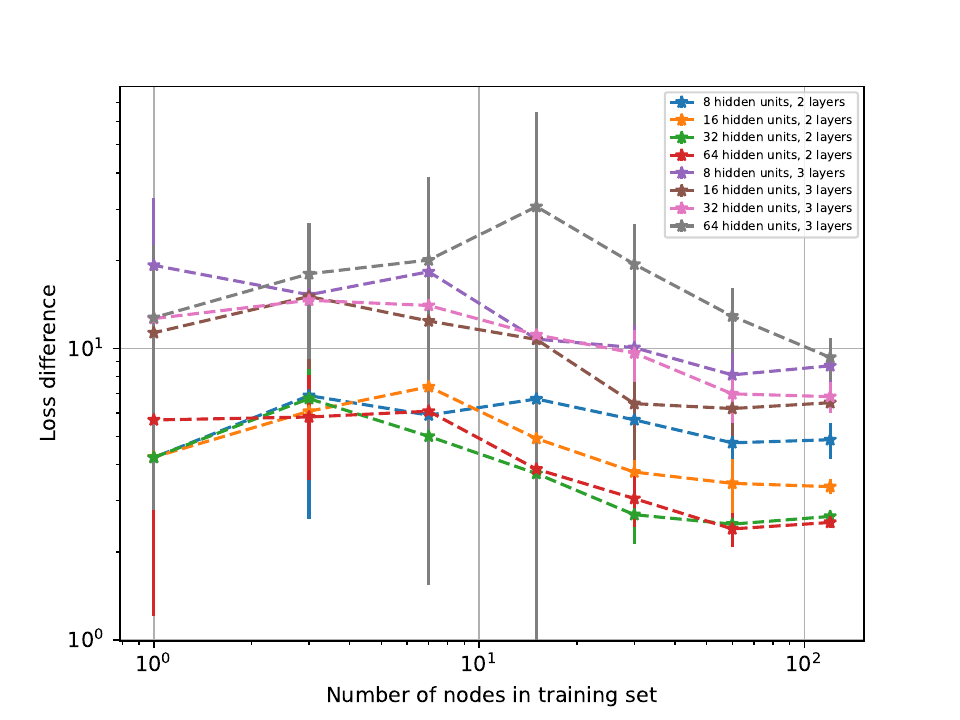}
        % \caption{Loss Difference in CiteSeer Dataset.}
        \caption{Generalization Gap}
    \end{subfigure}
    \begin{subfigure}{0.32\textwidth}
        \centering
        \includegraphics[width=\linewidth]{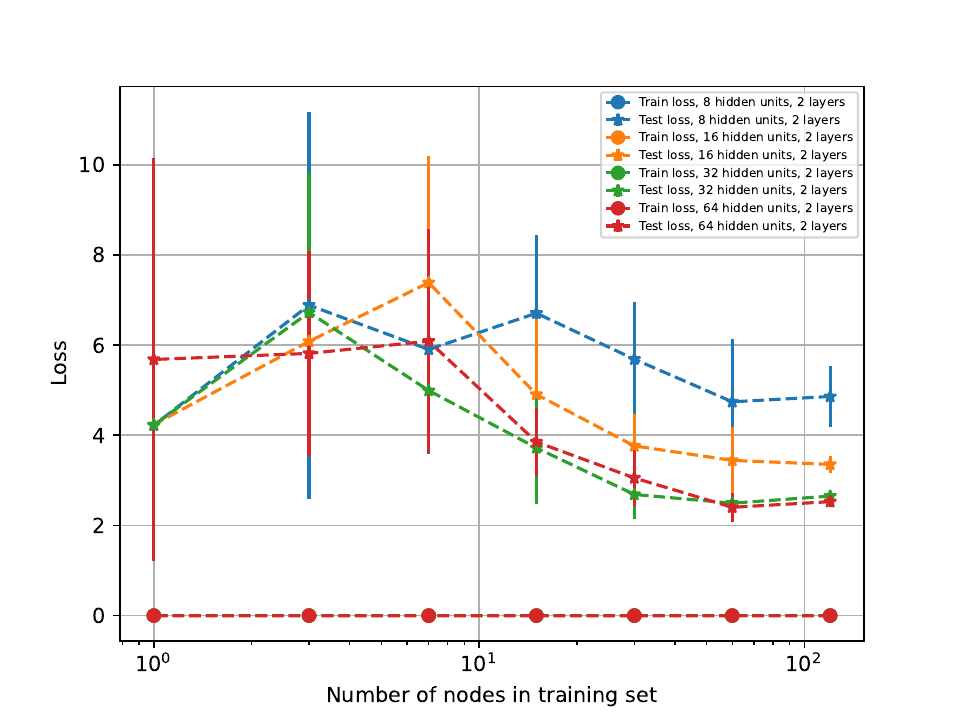}
        % \caption{Loss in Train/Test in CiteSeer Dataset for a $2$ Layers GNN.}
        \caption{$2$ Layers}
    \end{subfigure}
    \begin{subfigure}{0.32\textwidth}
        \centering
        \includegraphics[width=\linewidth]{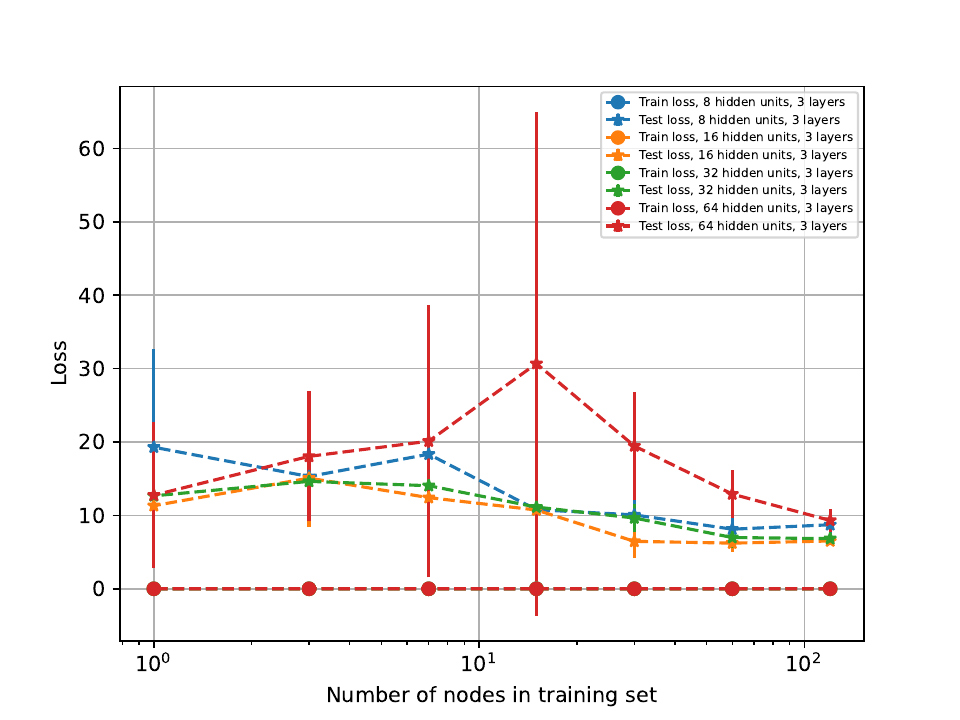}
        % \caption{Loss in Train/Test in CiteSeer Dataset for a $3$ Layers GNN.}
        \caption{$3$ Layers}
    \end{subfigure}
    \caption{Generalization gap, testing, and training losses with respect to the number of nodes in the CiteSeer dataset. The top row is in accuracy, and the bottom row is the cross-entropy loss. }
\end{figure*}

\input{tables/citeseer}

\subsubsection{PubMed dataset}
For the PubMed dataset, we used the standard one, which can be obtained running $\texttt{torch\_geometric.datasets.Planetoid(root="./data",name='PubMed')}$.

\begin{figure*}[ht!]
    \centering
    \begin{subfigure}{0.32\textwidth}
        \centering
        \includegraphics[width=\linewidth]{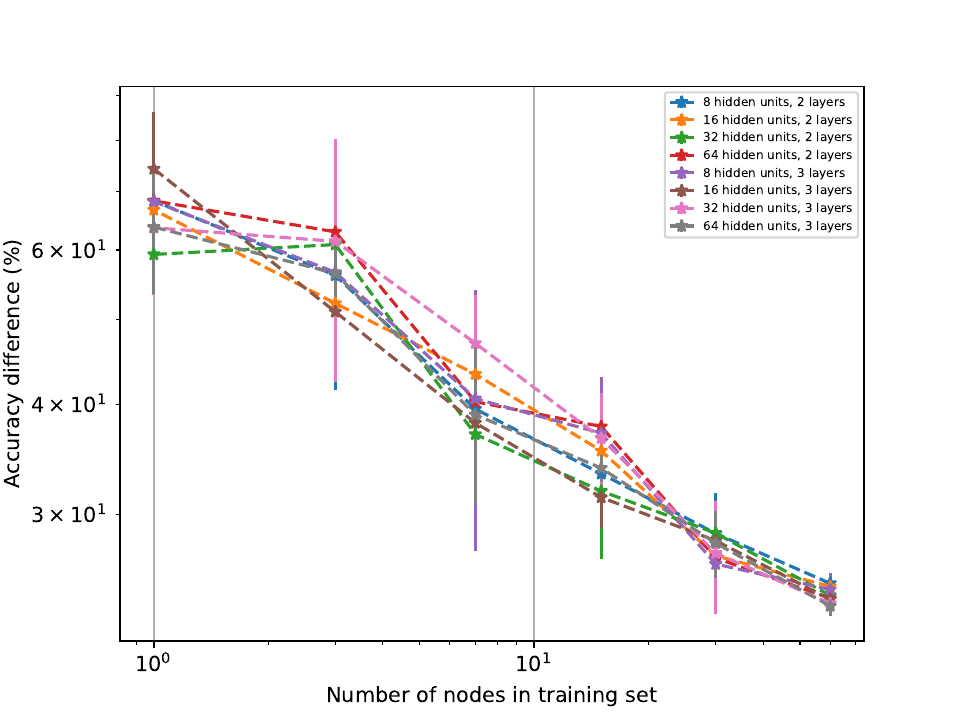}
        % \caption{Accuracy Difference in PubMed Dataset.}
        % \label{fig:PubMed_acc_1}
    \end{subfigure}%
    \begin{subfigure}{0.32\textwidth}
        \centering
        \includegraphics[width=\linewidth]{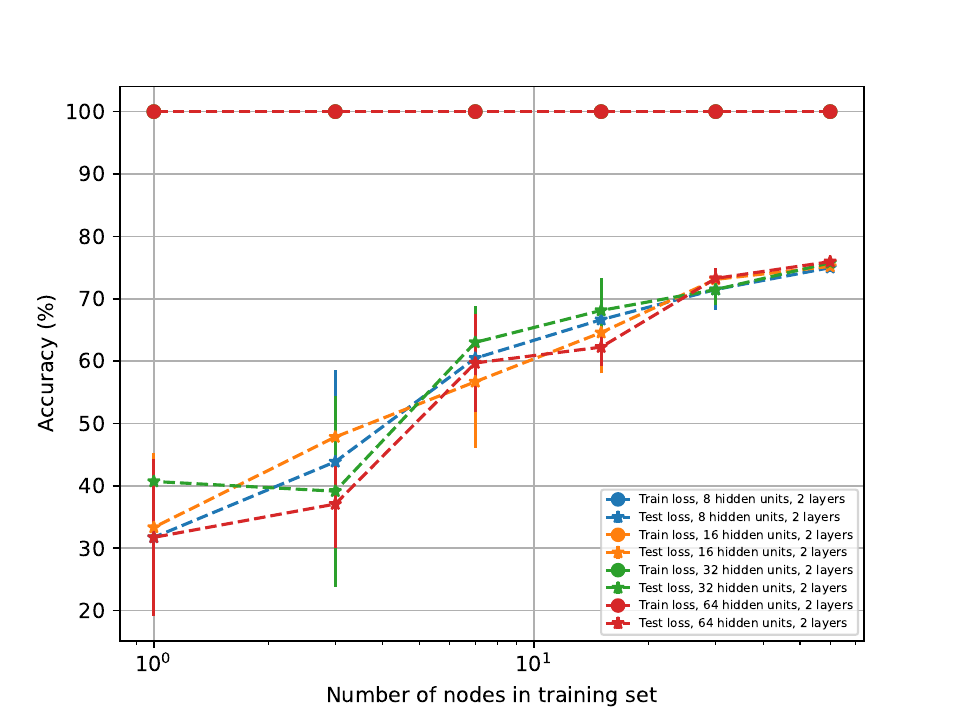}
        % \caption{Accuracy in Train/Test in PubMed Dataset for a $2$ Layers GNN.}
        % \label{fig:PubMed_acc_layer_2}
    \end{subfigure}%
    \begin{subfigure}{0.32\textwidth}
        \centering
        \includegraphics[width=\linewidth]{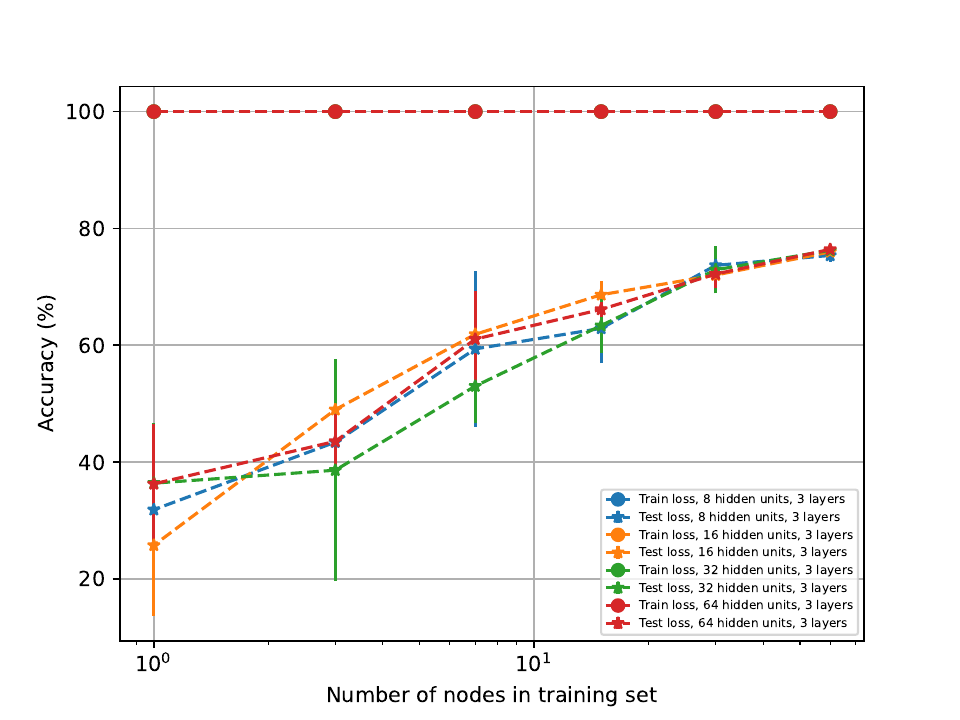}
        % \caption{Accuracy in Train/Test in PubMed Dataset for a $3$ Layers GNN.}
        % \label{fig:PubMed_acc_layer_3}
    \end{subfigure}\\
    \begin{subfigure}{0.32\textwidth}
        \centering
        \includegraphics[width=\linewidth]{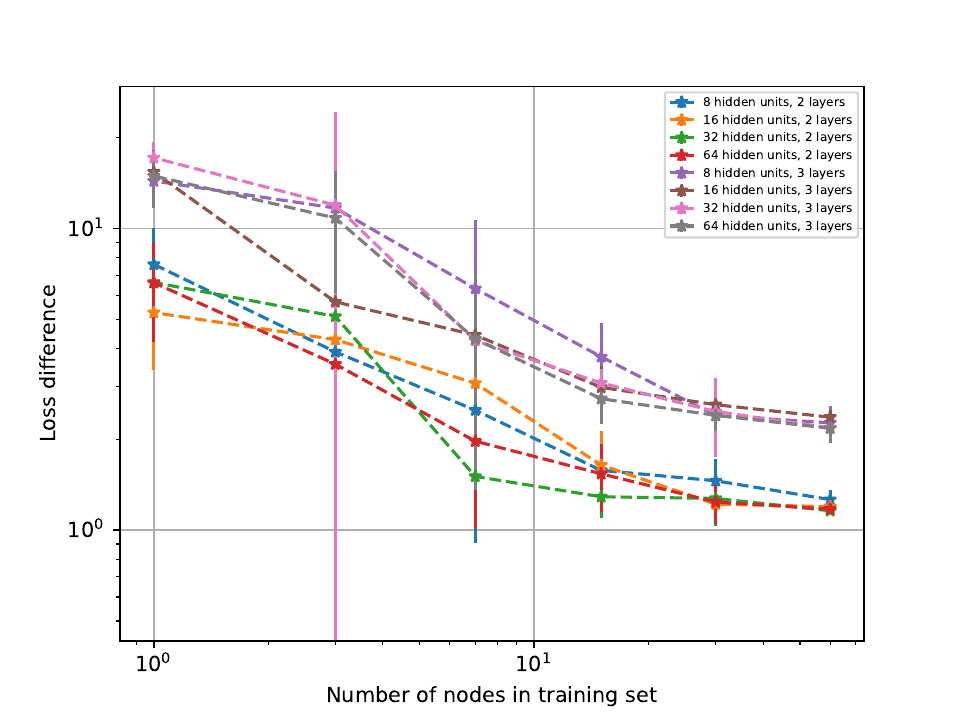}
        % \caption{Loss Difference in PubMed Dataset.}
        % \label{fig:PubMed_acc_3}
            \caption{Generalization Gap}
    \end{subfigure}
    \begin{subfigure}{0.32\textwidth}
        \centering
        \includegraphics[width=\linewidth]{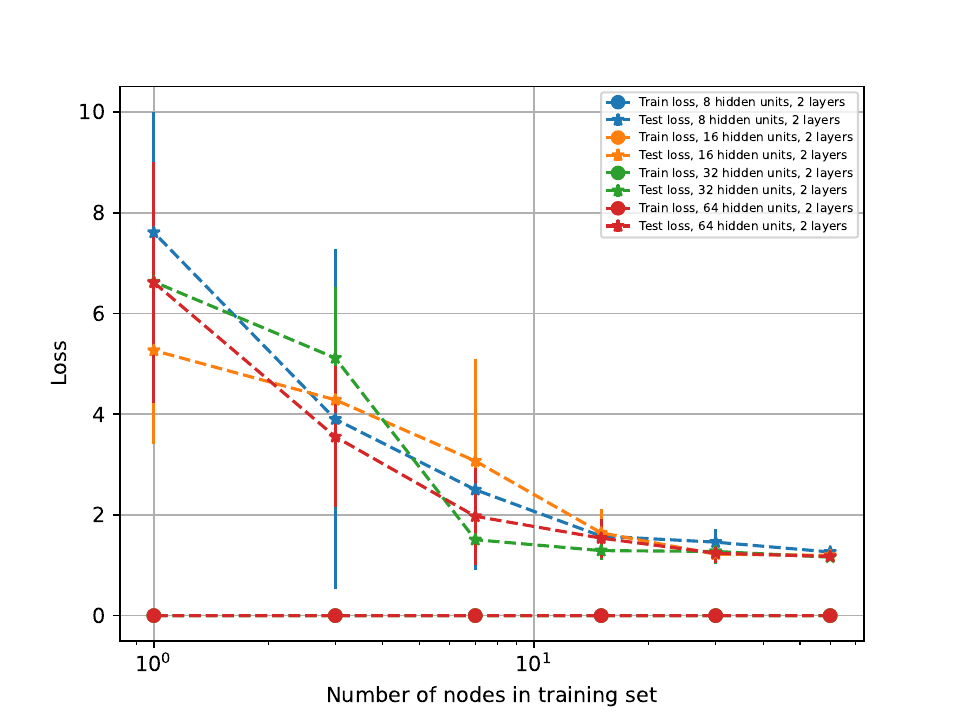}
        % \caption{Loss in Train/Test in PubMed Dataset for a $2$ Layers GNN.}
        \caption{$2$ Layers}
    \end{subfigure}
    \begin{subfigure}{0.32\textwidth}
        \centering
        \includegraphics[width=\linewidth]{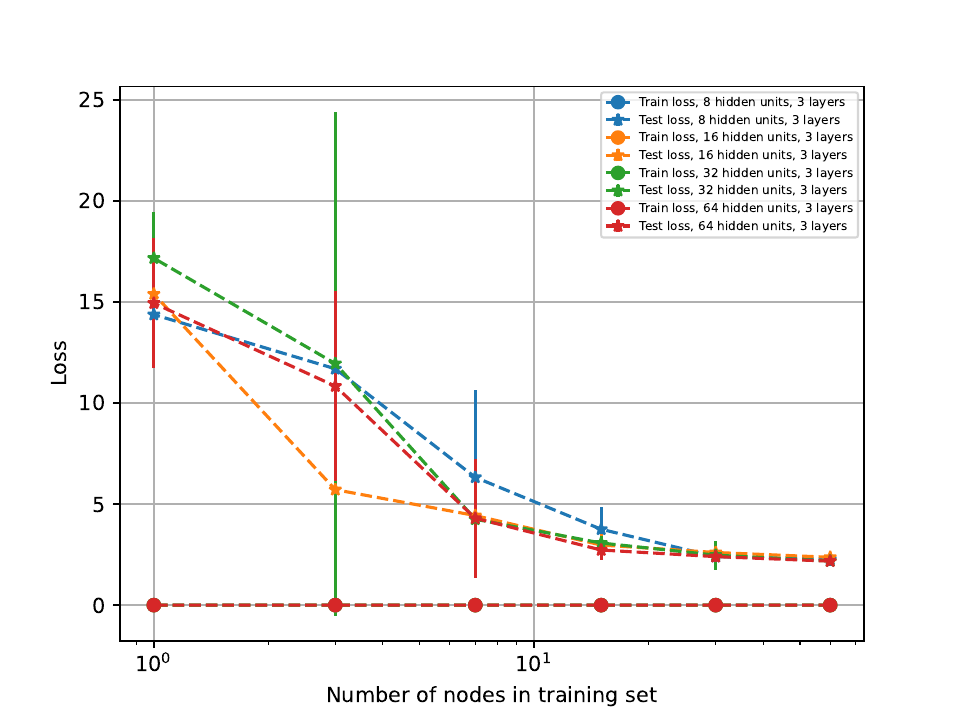}
        % \caption{Loss in Train/Test in PubMed Dataset for a $3$ Layers GNN.}
        \caption{$3$ Layers}
    \end{subfigure}
    \caption{Generalization gap, testing, and training losses with respect to the number of nodes in the PubMed dataset. The top row is in accuracy, and the bottom row is the cross-entropy loss. }
\end{figure*}
\input{tables/pubmed}

\subsubsection{Coauthors CS dataset}
For the CS dataset, we used the standard one, which can be obtained running $\texttt{torch\_geometric.datasets.Coauthor(root="./data", name='CS')}$. In this case, given that there are no training and testing sets, we randomly partitioned the datasets and used $90\%$ of the samples for training and the remaining $10\%$ for testing.

\begin{figure*}[ht!]
    \centering
    \begin{subfigure}{0.32\textwidth}
        \centering
        \includegraphics[width=\linewidth]{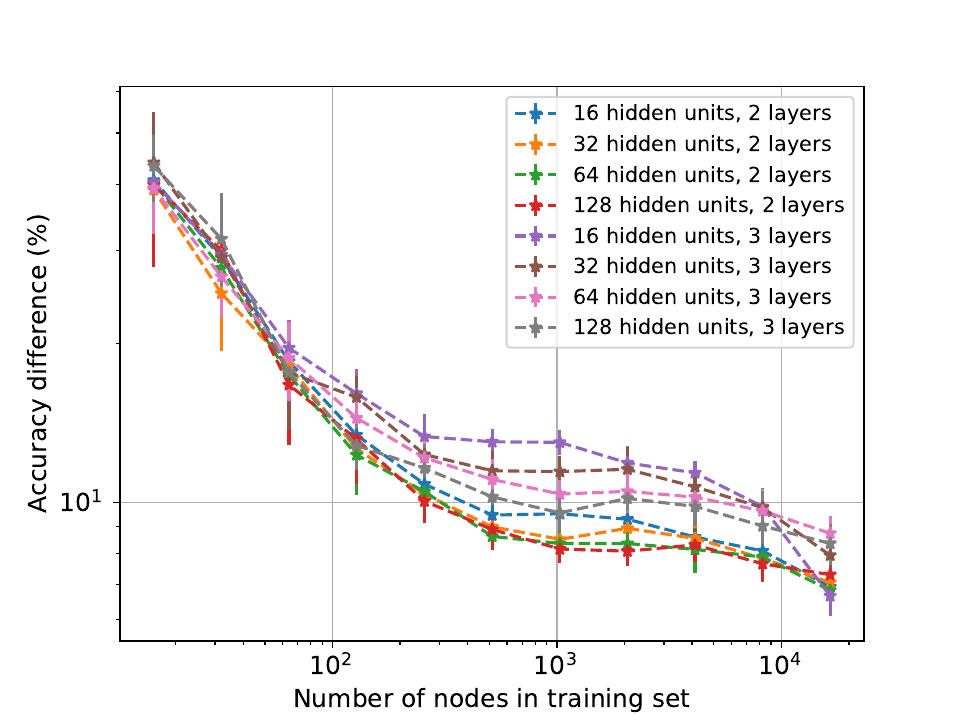}
        % \caption{Accuracy Difference in CS Dataset.}
        % \label{fig:CS_acc_1}
    \end{subfigure}%
        \begin{subfigure}{0.32\textwidth}
        \centering
        \includegraphics[width=\linewidth]{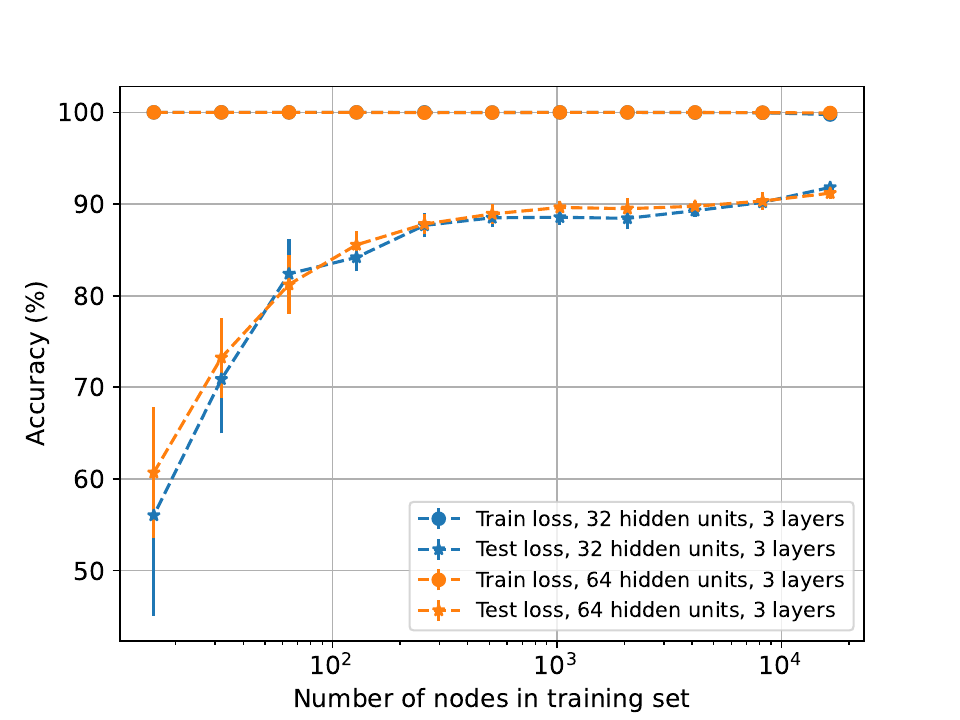}
        % \caption{Accuracy in Train/Test in CS Dataset for a $3$ Layers GNN.}
        % \label{fig:CS_acc_layer_3}
    \end{subfigure}%
    \begin{subfigure}{0.32\textwidth}
        \centering
        \includegraphics[width=\linewidth]{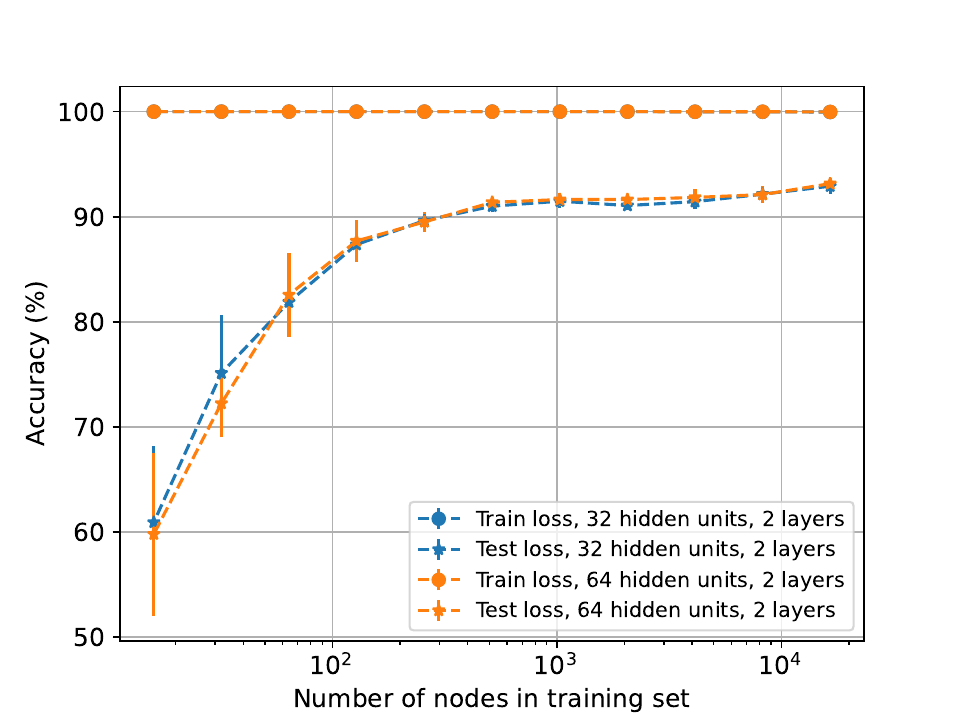}
        % \caption{Accuracy in Train/Test in CS Dataset for a $2$ Layers GNN.}
        % \label{fig:CS_acc_layer_2}
    \end{subfigure}\\
    \begin{subfigure}{0.32\textwidth}
        \centering
        \includegraphics[width=\linewidth]{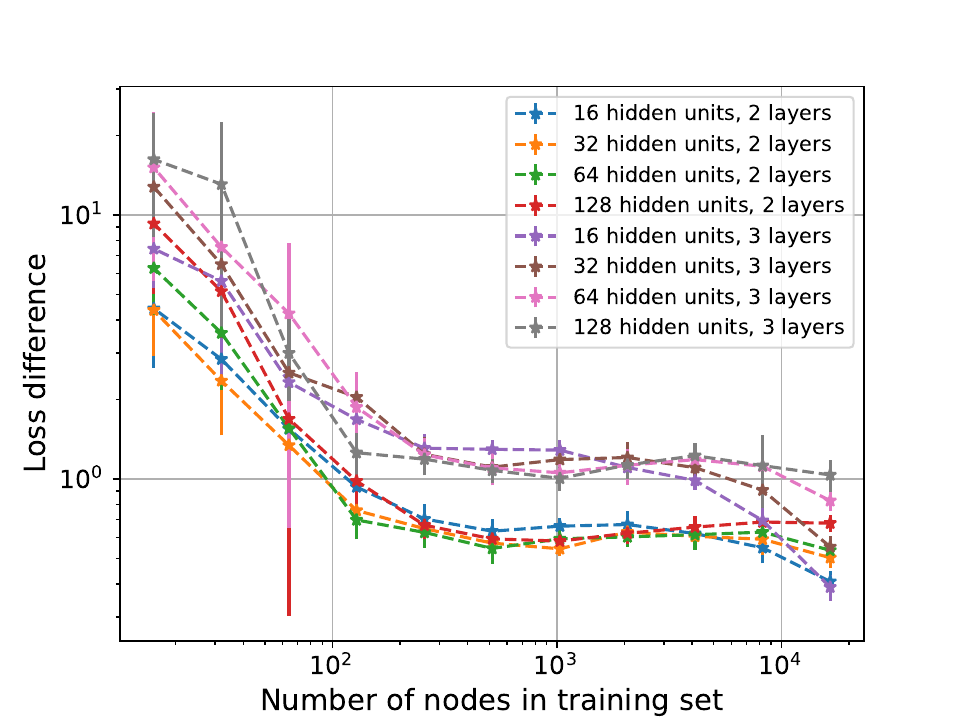}
        % \caption{Loss Difference in CS Dataset.}
            \caption{Generalization Gap}
    \end{subfigure}
    \begin{subfigure}{0.32\textwidth}
        \centering
        \includegraphics[width=\linewidth]{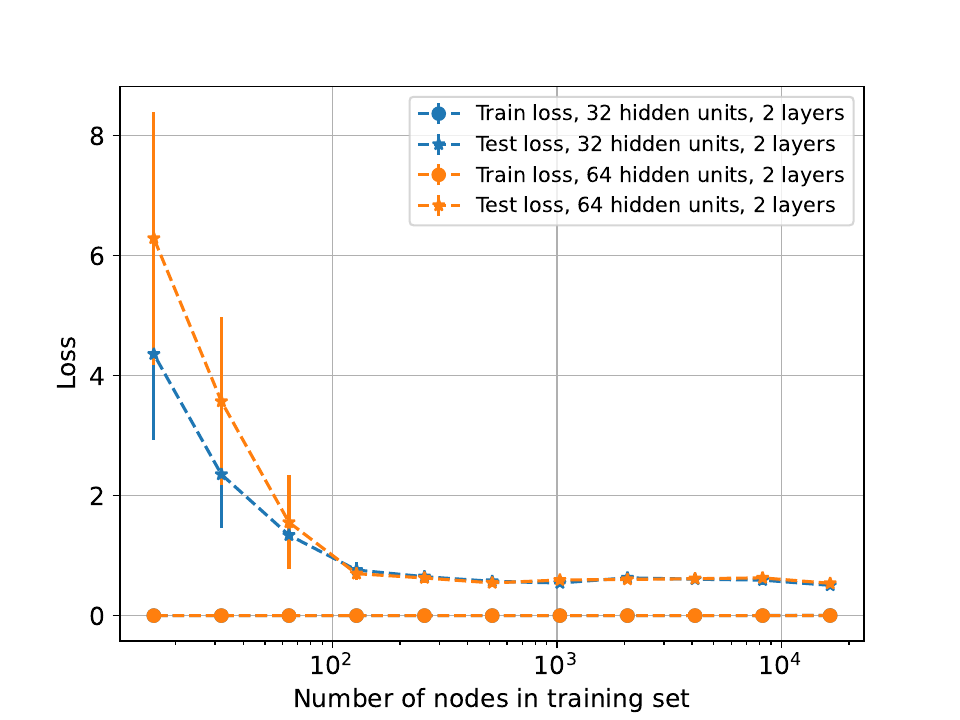}
        % \caption{Loss in Train/Test in CS Dataset for a $2$ Layers GNN.}
        \caption{$2$ Layers}
    \end{subfigure}
    \begin{subfigure}{0.32\textwidth}
        \centering
        \includegraphics[width=\linewidth]{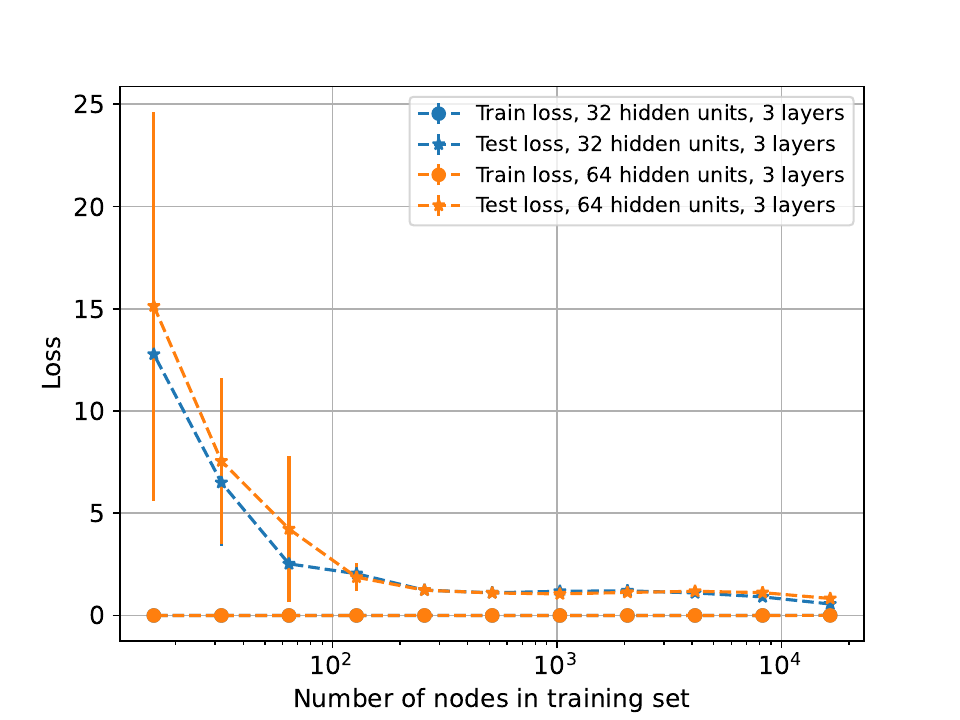}
        % \caption{Loss in Train/Test in CS Dataset for a $3$ Layers GNN.}
        \caption{$3$ Layers}
    \end{subfigure}
    \caption{Generalization gap, testing, and training losses with respect to the number of nodes in the CS dataset. The top row is in accuracy, and the bottom row is the cross-entropy loss. }
\end{figure*}

\begin{figure*}[ht!]
    \centering
    \begin{subfigure}{0.5\textwidth}
        \centering
        \includegraphics[width=\linewidth]{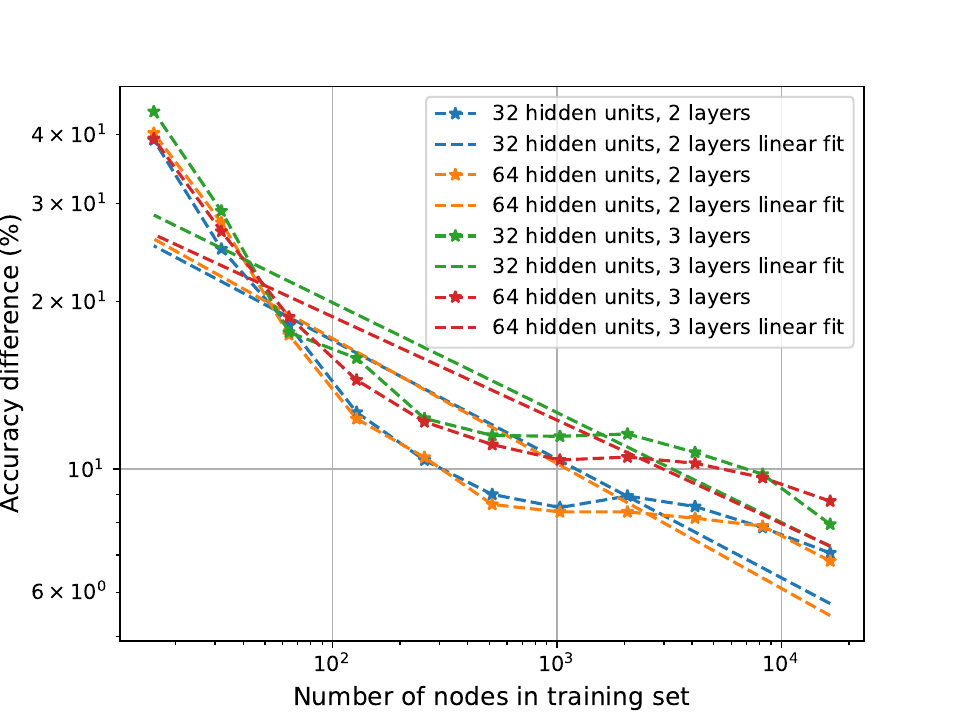}
        \caption{Linear fit for accuracy generalization gap}
        % \label{fig:CS_acc_1}
    \end{subfigure}%
        \begin{subfigure}{0.5\textwidth}
        \centering
        \includegraphics[width=\linewidth]{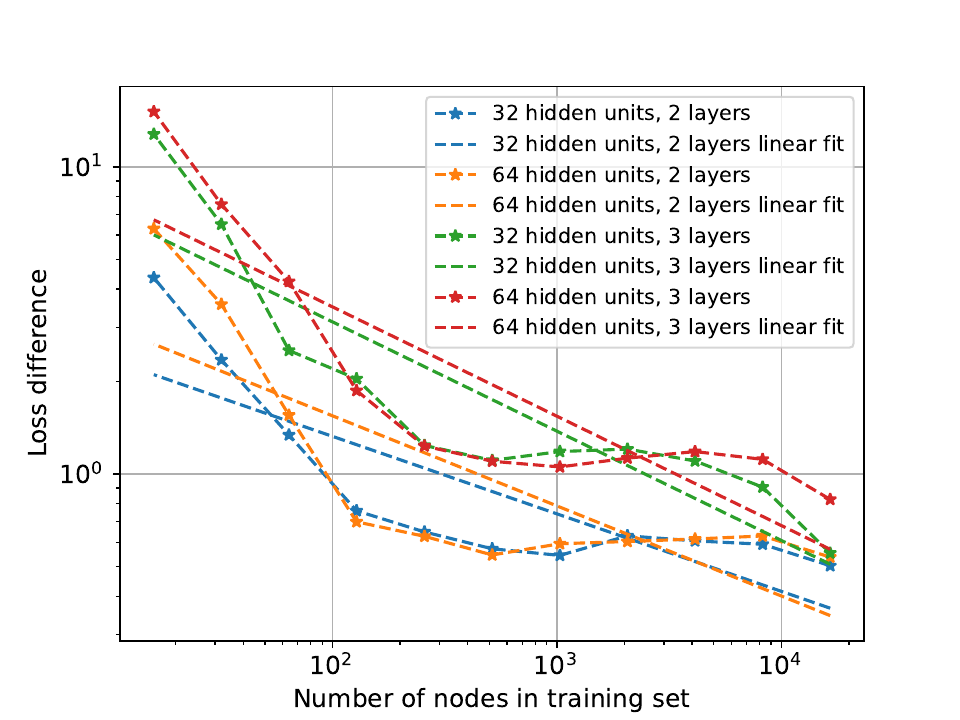}
        % \caption{Accuracy in Train/Test in CS Dataset for a $3$ Layers GNN.}
        \caption{Linear fit for loss generalization gap}
    \end{subfigure}%
\caption{Generalization gaps as a function of the number of nodes in the training set in the CS dataset. }
\end{figure*}
\input{tables/CS}

\subsubsection{Coauthors Physics dataset}
For the Physics dataset, we used the standard one, which can be obtained running $\texttt{torch\_geometric.datasets.Coauthor(root="./data", name='Physics')}$. In this case, given that there are no training and testing sets, we randomly partitioned the datasets and used $90\%$ of the samples for training and the remaining $10\%$ for testing. 

\begin{figure*}[ht!]
    \centering
    \begin{subfigure}{0.32\textwidth}
        \centering
        \includegraphics[width=\linewidth]{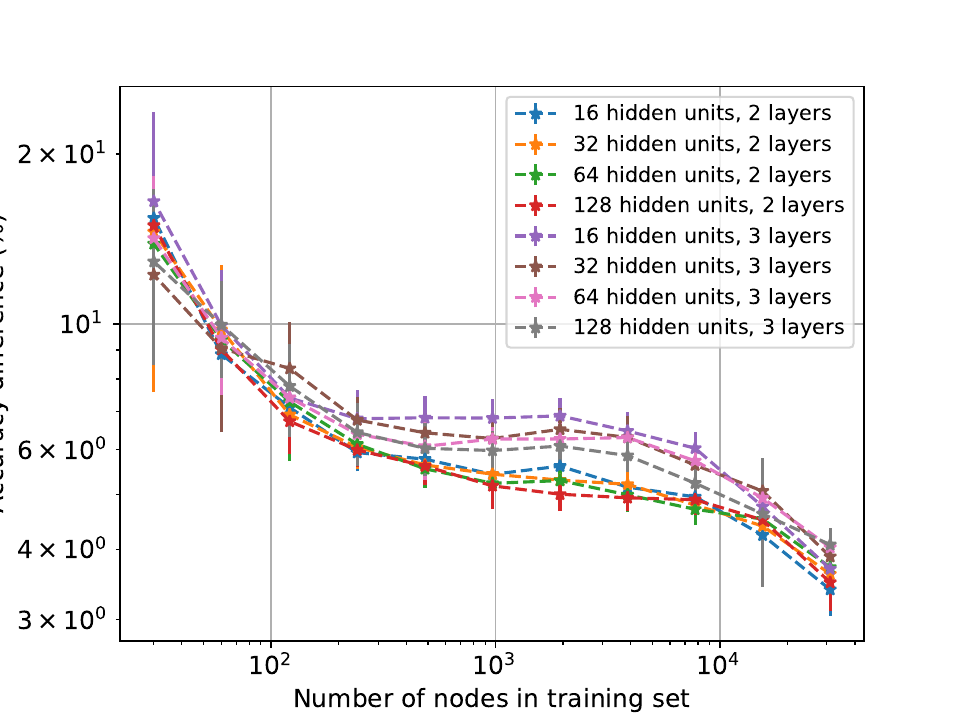}
        % \caption{Accuracy Difference in Physics Dataset.}
        % \label{fig:Physics_acc_1}
    \end{subfigure}%
    \begin{subfigure}{0.32\textwidth}
        \centering
        \includegraphics[width=\linewidth]{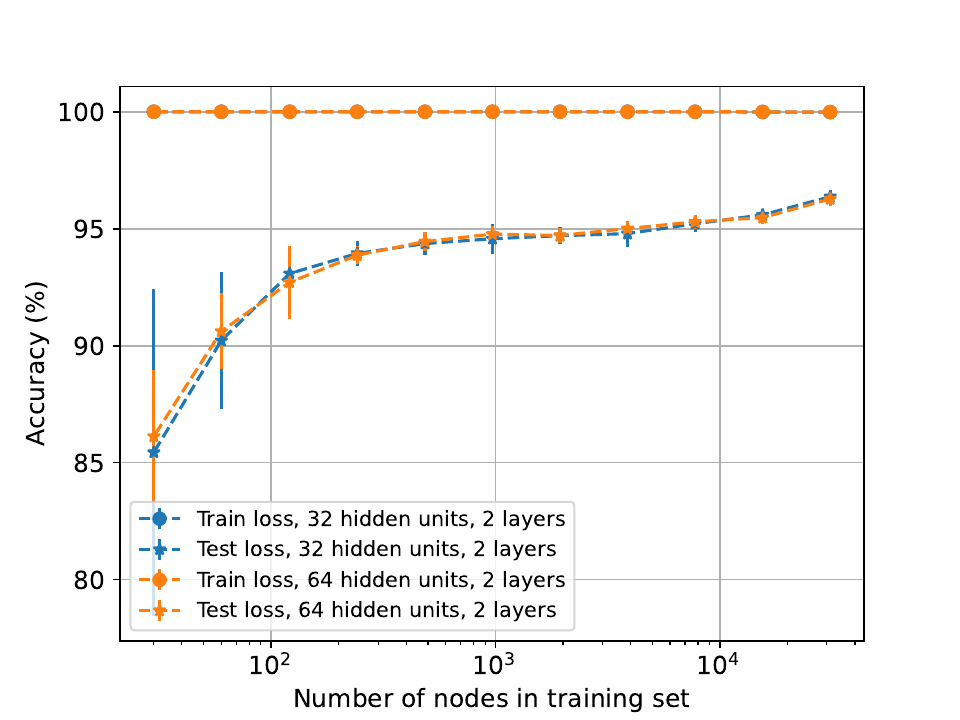}
        % \caption{Accuracy in Train/Test in Physics Dataset for a $2$ Layers GNN.}
        % \label{fig:Physics_acc_layer_2}
    \end{subfigure}%
    \begin{subfigure}{0.32\textwidth}
        \centering
        \includegraphics[width=\linewidth]{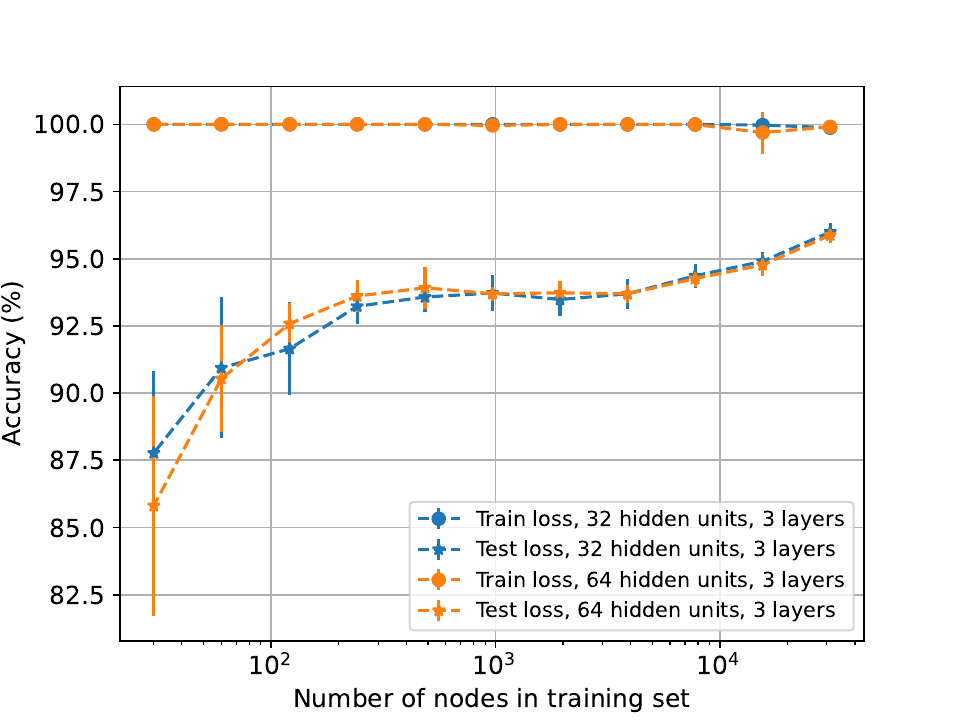}
        % \caption{Accuracy in Train/Test in Physics Dataset for a $3$ Layers GNN.}
        % \label{fig:Physics_acc_layer_3}
    \end{subfigure}\\
    \begin{subfigure}{0.32\textwidth}
        \centering
        \includegraphics[width=\linewidth]{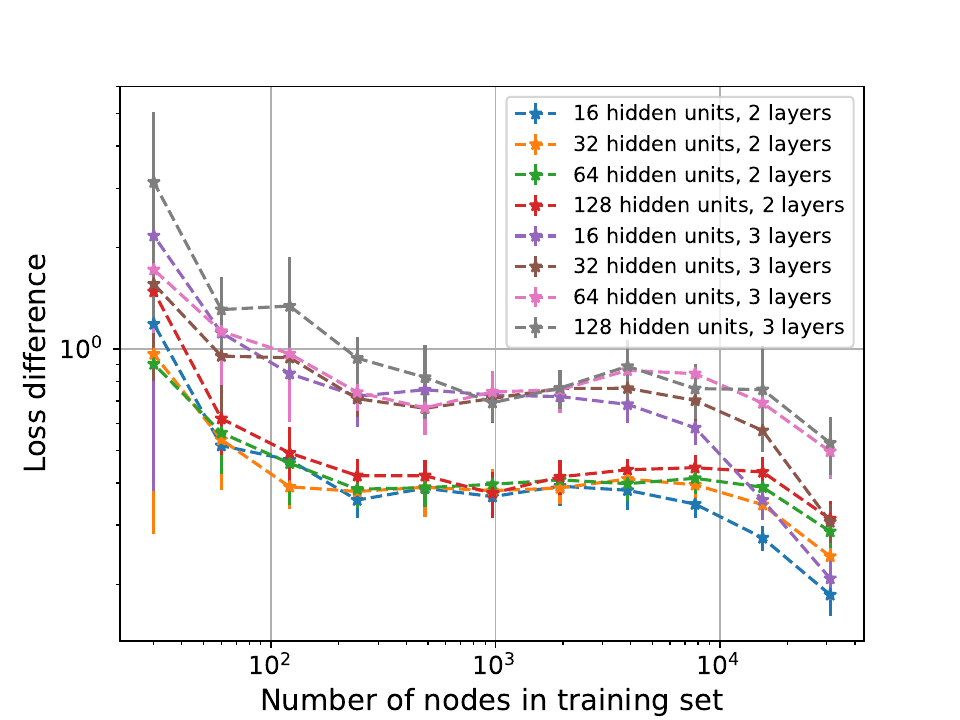}
        % \caption{Loss Difference in Physics Dataset.}
            \caption{Generalization Gap}
    \end{subfigure}
    \begin{subfigure}{0.32\textwidth}
        \centering
        \includegraphics[width=\linewidth]{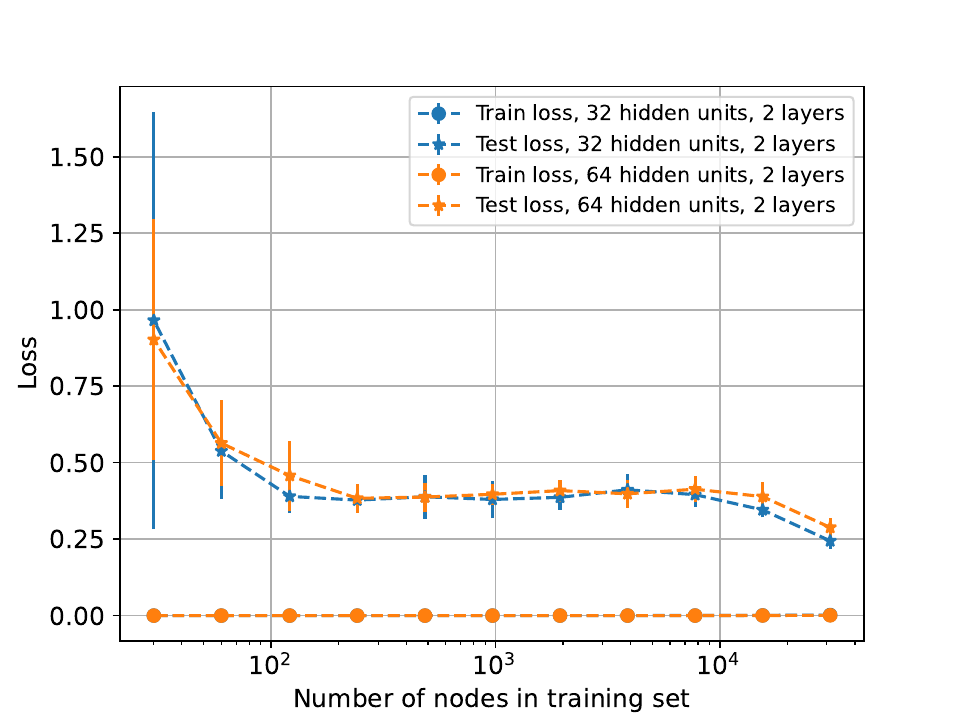}
        % \caption{Loss in Train/Test in Physics Dataset for a $2$ Layers GNN.}
        \caption{$2$ Layers}
    \end{subfigure}
    \begin{subfigure}{0.32\textwidth}
        \centering
        \includegraphics[width=\linewidth]{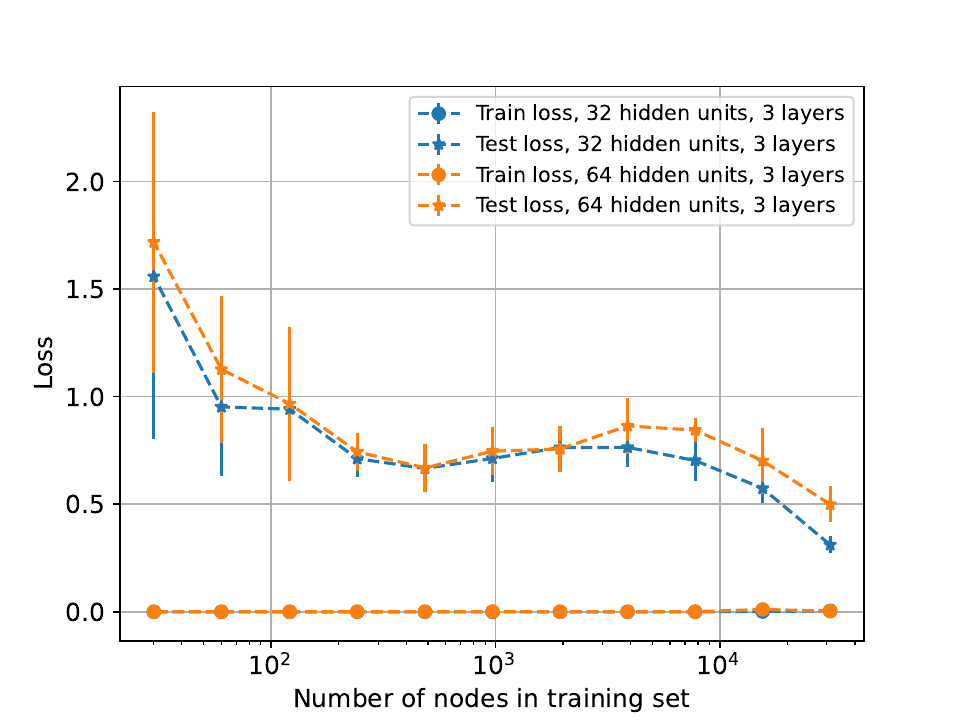}
        % \caption{Loss in Train/Test in Physics Dataset for a $3$ Layers GNN.}
        \caption{$3$ Layers}
    \end{subfigure}\\
    \caption{Generalization gap, testing, and training losses with respect to the number of nodes in the Physics dataset. The top row is in accuracy, and the bottom row is the cross-entropy loss. }
\end{figure*}

\begin{figure*}[ht!]
    \centering
    \begin{subfigure}{0.45\textwidth}
        \centering
        \includegraphics[width=0.9\linewidth]{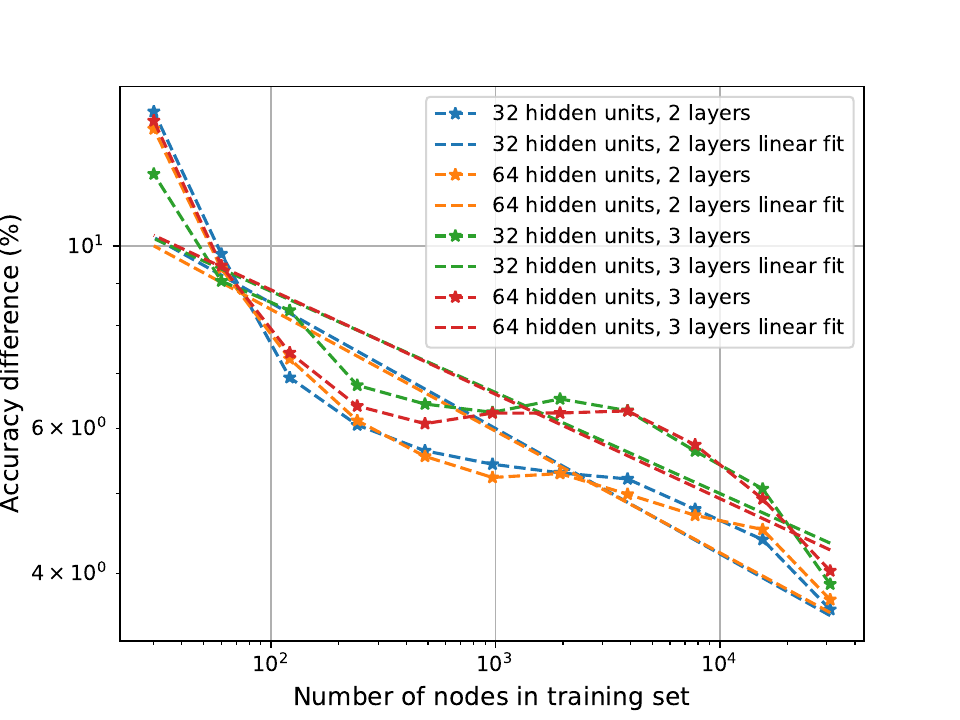}
        \caption{Linear fit for accuracy generalization gap}
        % \label{fig:CS_acc_1}
    \end{subfigure}%
        \begin{subfigure}{0.45\textwidth}
        \centering
        \includegraphics[width=0.9\linewidth]{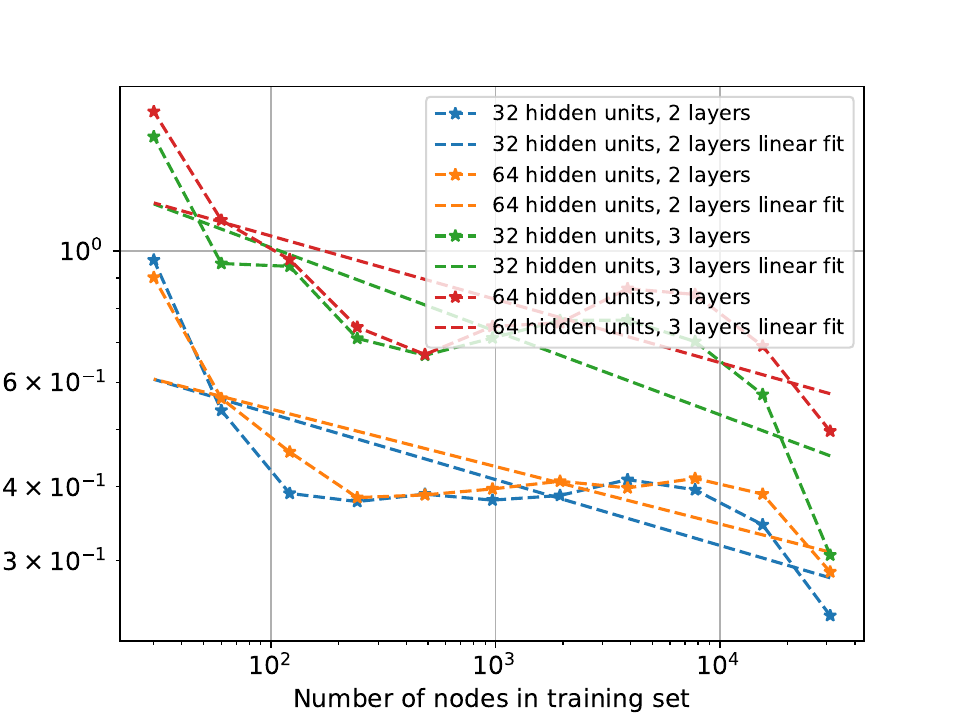}
        % \caption{Accuracy in Train/Test in CS Dataset for a $3$ Layers GNN.}
        \caption{Linear fit for loss generalization gap}
    \end{subfigure}%
    \caption{Generalization Gaps as a function of the number of nodes in the training set in the Physics dataset. }
\end{figure*}

\input{tables/Physics}

\subsubsection{Heterophilous Amazon ratings dataset}
For the Amazon dataset, we used the standard one, which can be obtained running $\texttt{torch\_geometric.datasets.HeterophilousGraphDataset(root="./data", name='Amazon')}$. In this case, we used the $10$ different splits that the dataset has assigned. 

\begin{figure*}[ht!]
    \centering
    \begin{subfigure}{0.32\textwidth}
        \centering
        \includegraphics[width=\linewidth]{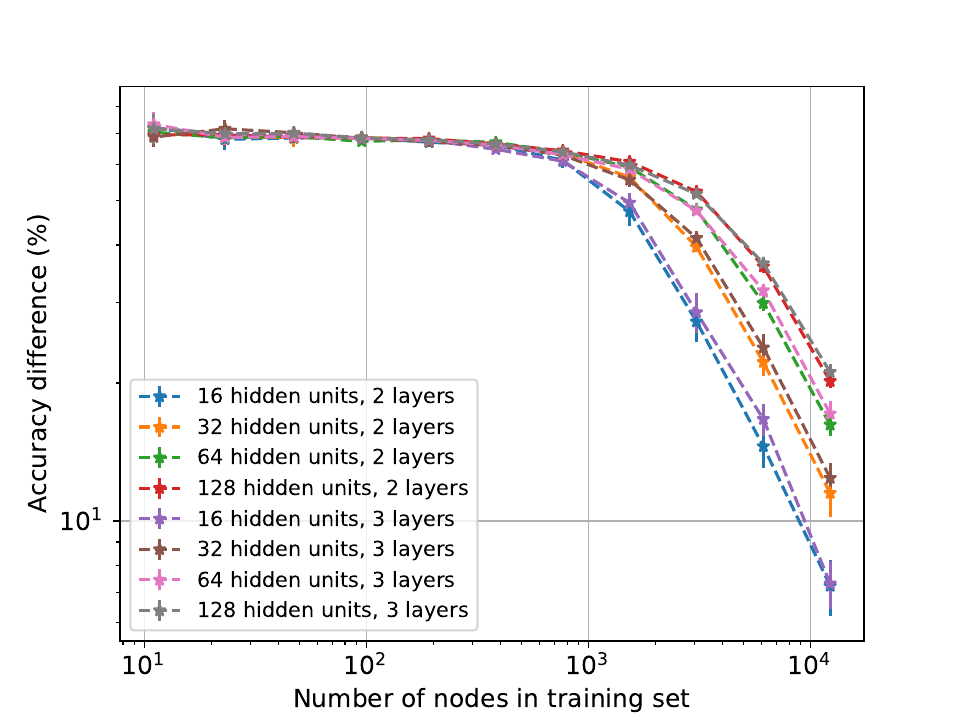}
        % \caption{Accuracy Difference in Amazon Dataset.}
        % \label{fig:Amazon_acc_1}
    \end{subfigure}%
    \begin{subfigure}{0.32\textwidth}
        \centering
        \includegraphics[width=\linewidth]{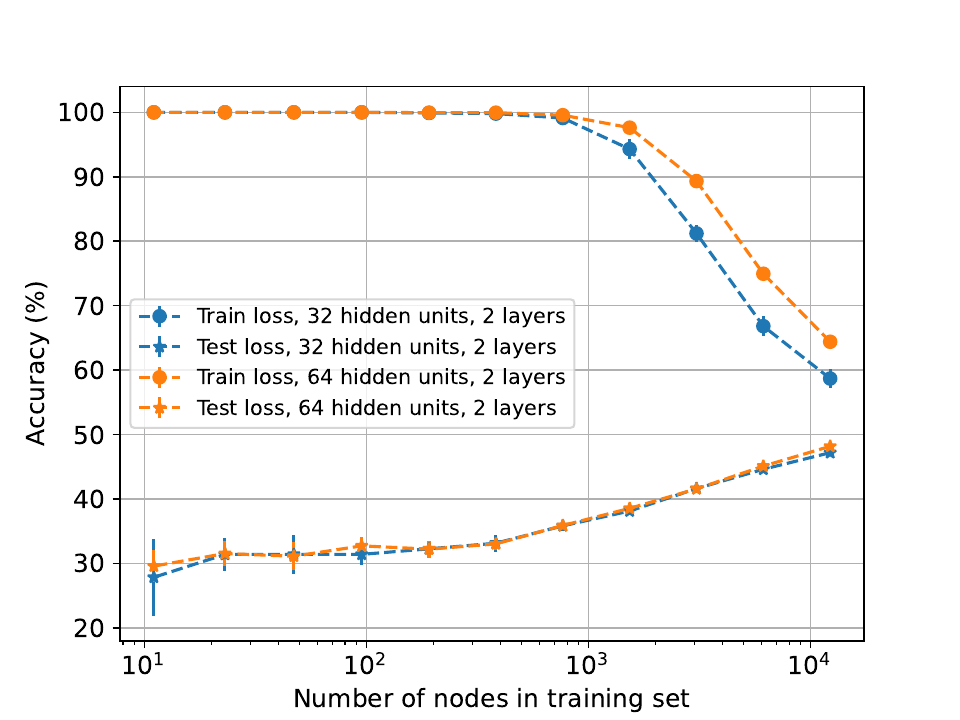}
        % \caption{Accuracy in Train/Test in Amazon Dataset for a $2$ Layers GNN.}
        % \label{fig:Amazon_acc_layer_2}
    \end{subfigure}%
    \begin{subfigure}{0.32\textwidth}
        \centering
        \includegraphics[width=\linewidth]{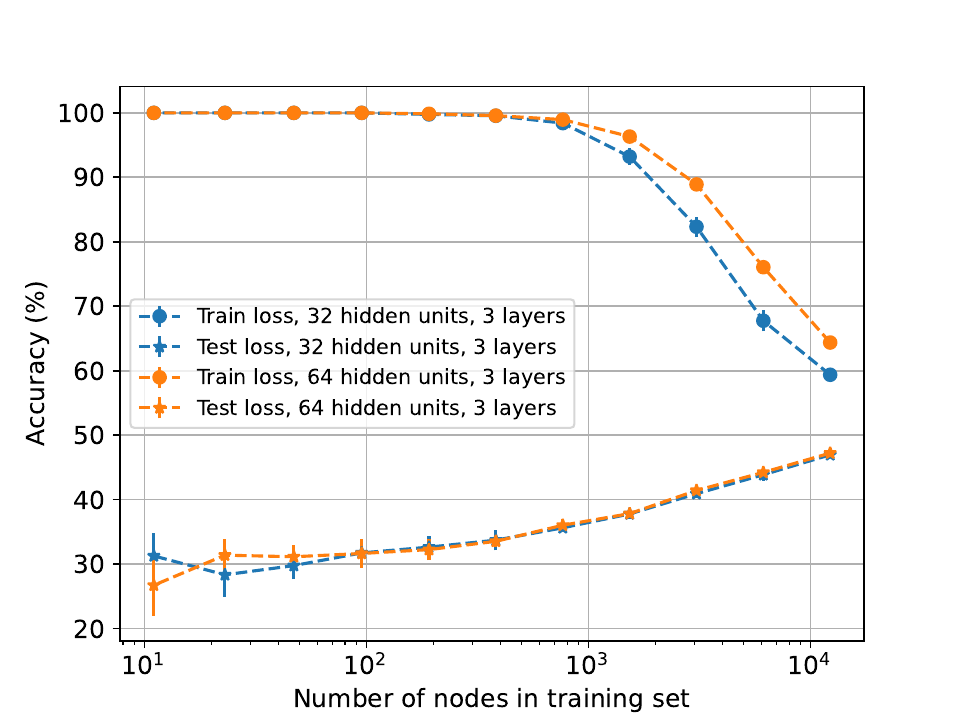}
        % \caption{Accuracy in Train/Test in Amazon Dataset for a $3$ Layers GNN.}
        % \label{fig:Amazon_acc_layer_3}
    \end{subfigure}\\
    \begin{subfigure}{0.32\textwidth}
        \centering
        \includegraphics[width=\linewidth]{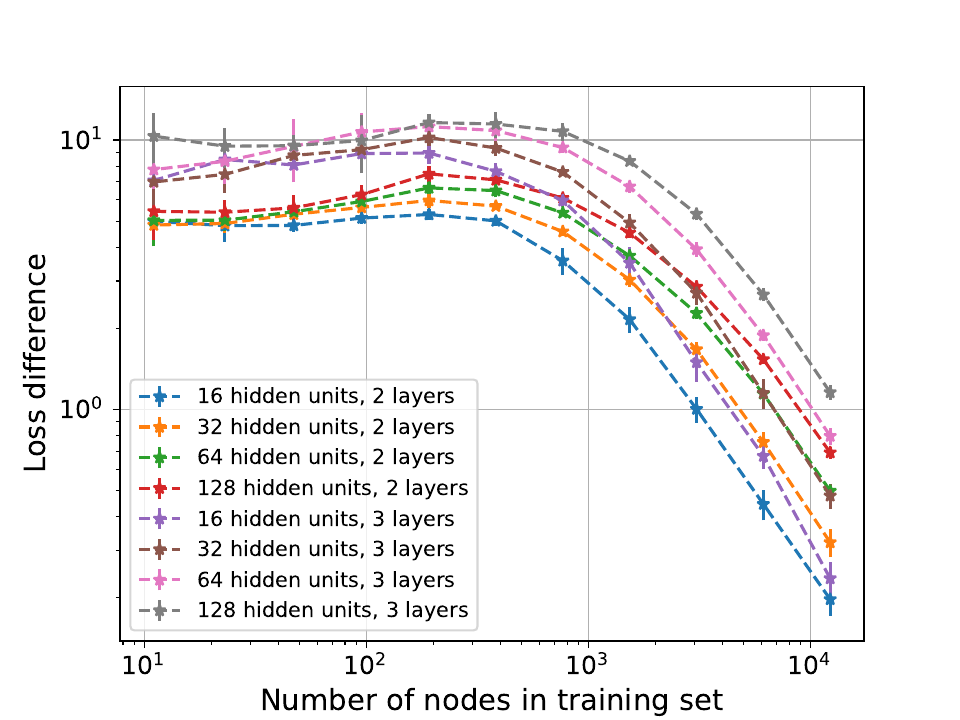}
            \caption{Generalization Gap}
        % \label{fig:Amazon_acc_3}
    \end{subfigure}
    \begin{subfigure}{0.32\textwidth}
        \centering
        \includegraphics[width=\linewidth]{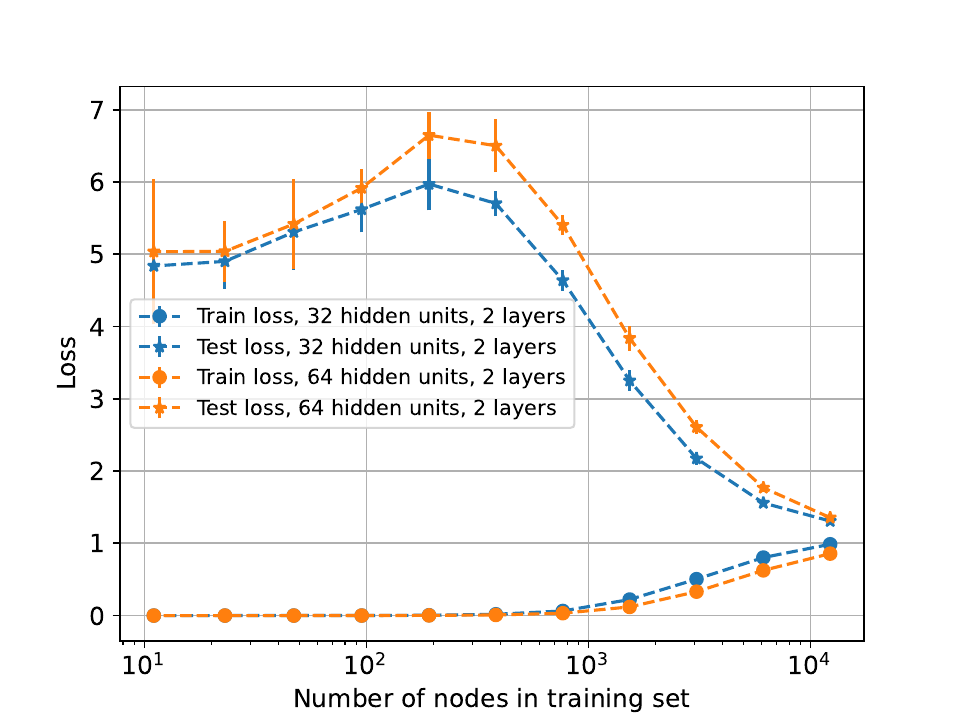}
        \caption{$2$ Layers}
        % \label{fig:Amazon_loss_layer_2}
    \end{subfigure}
    \begin{subfigure}{0.32\textwidth}
        \centering
        \includegraphics[width=\linewidth]{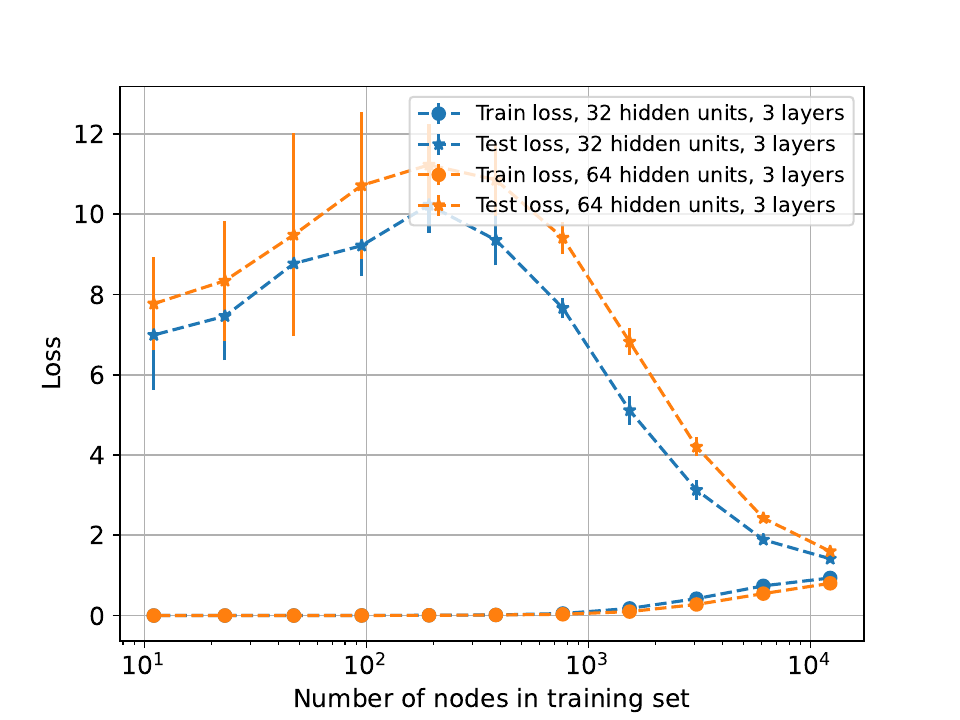}
        % \caption{Loss in Train/Test in Amazon Dataset for a $3$ Layers GNN.}
                    \caption{$3$ Layers}
    \end{subfigure}
    \caption{Generalization gap, testing, and training losses with respect to the number of nodes in the Amazon dataset. The top row is in accuracy, and the bottom row is the cross-entropy loss. }
\end{figure*}

\begin{figure*}[ht!]
    \centering
    \begin{subfigure}{0.5\textwidth}
        \centering
        \includegraphics[width=\linewidth]{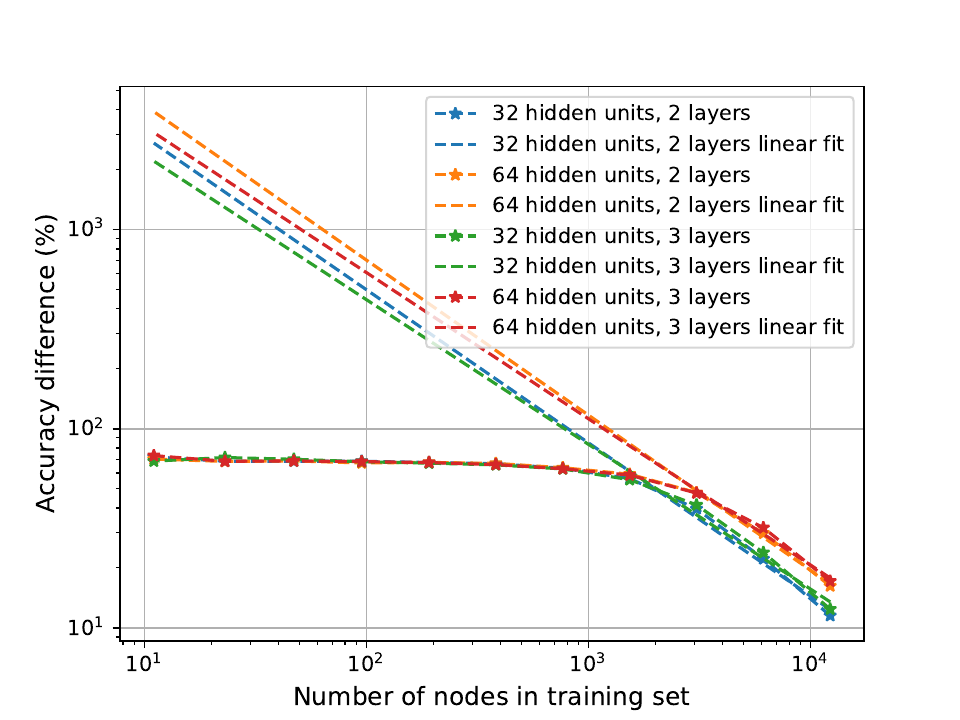}
        \caption{Linear fit for accuracy generalization gap}
        % \label{fig:CS_acc_1}
    \end{subfigure}%
        \begin{subfigure}{0.5\textwidth}
        \centering
        \includegraphics[width=\linewidth]{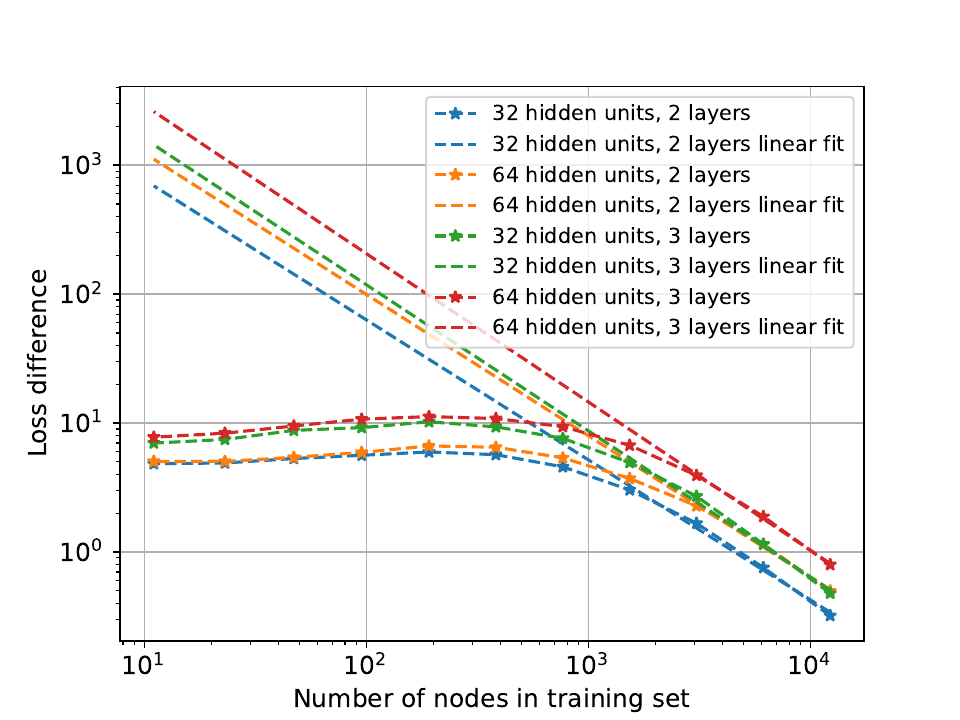}
        % \caption{Accuracy in Train/Test in CS Dataset for a $3$ Layers GNN.}
        \caption{Linear fit for loss generalization gap}
    \end{subfigure}%
    \caption{Generalization Gaps as a function of the number of nodes in the training set in the Amazon dataset. }
\end{figure*}

\input{tables/amazon}

\subsubsection{Heterophilous Roman Empire dataset}
For the Roman dataset, we used the standard one, which can be obtained running $\texttt{torch\_geometric.datasets.HeterophilousGraphDataset(root="./data", name='Roman')}$. In this case, we used the $10$ different splits that the dataset has assigned. 

\begin{figure*}[ht!]
    \centering
    \begin{subfigure}{0.32\textwidth}
        \centering
        \includegraphics[width=\linewidth]{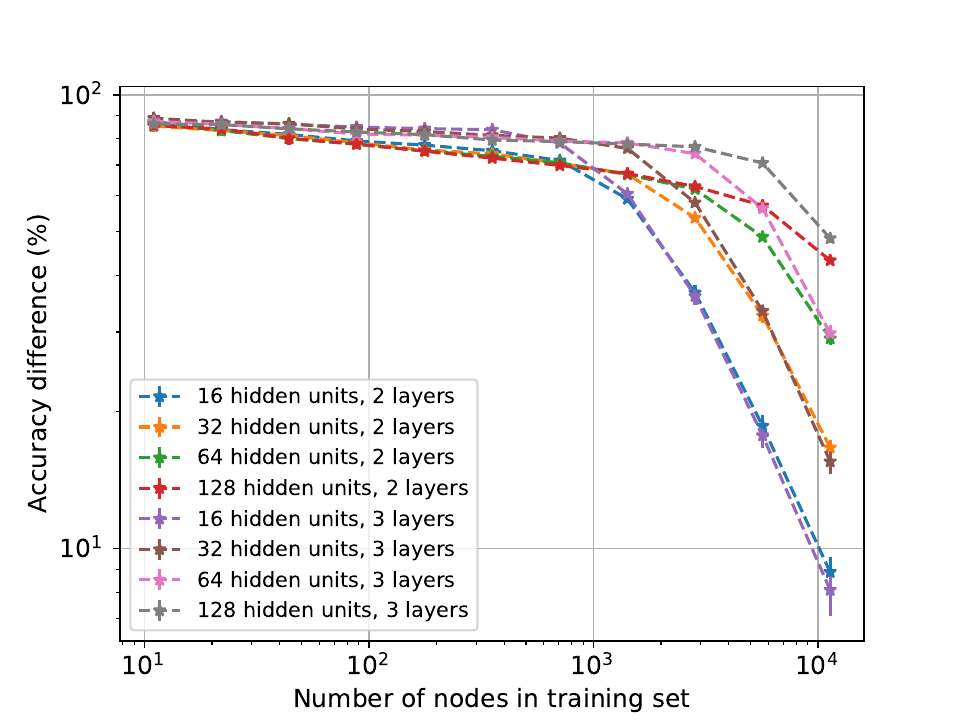}
        % \caption{Accuracy Difference in Roman Dataset.}
        % \label{fig:Roman_acc_1}
    \end{subfigure}%
    \begin{subfigure}{0.32\textwidth}
        \centering
        \includegraphics[width=\linewidth]{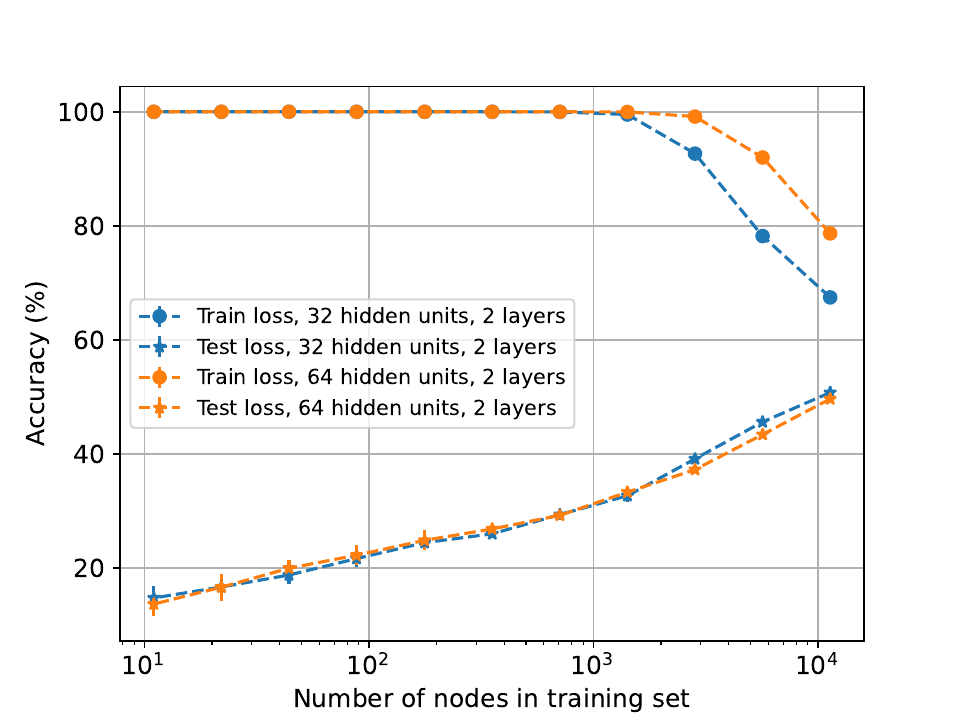}
        % \caption{Loss Difference in Roman Dataset.}
        % \label{fig:Roman_acc_3}
    \end{subfigure}
    \begin{subfigure}{0.32\textwidth}
        \centering
        \includegraphics[width=\linewidth]{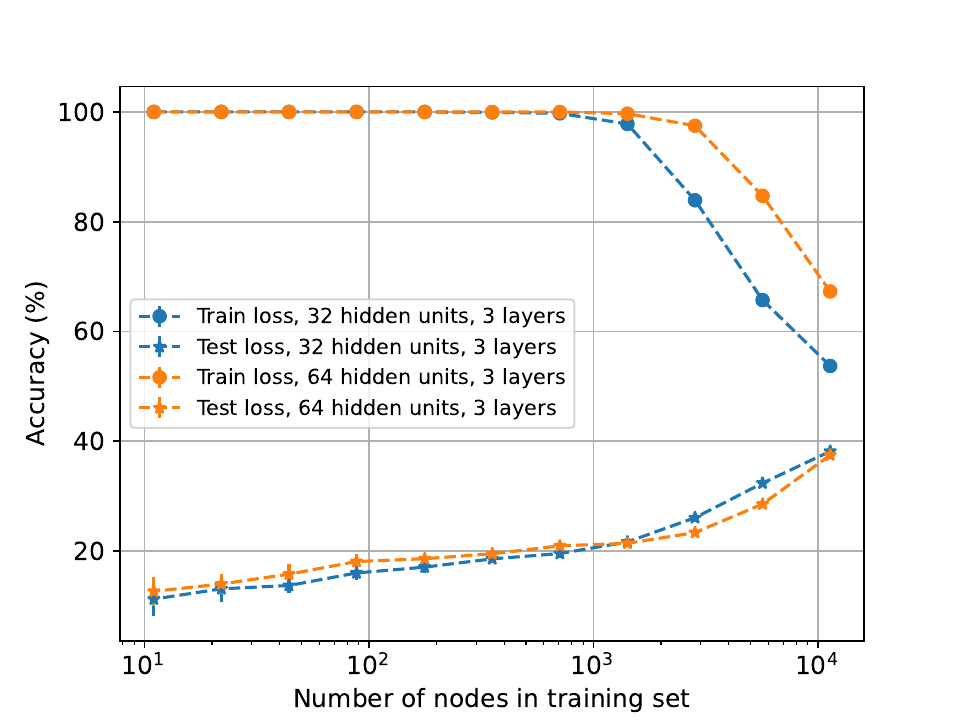}
        % \caption{Accuracy in Train/Test in Roman Dataset for a $2$ Layers GNN.}
        % \label{fig:Roman_acc_layer_2}
    \end{subfigure}\\
    \begin{subfigure}{0.32\textwidth}
        \centering
        \includegraphics[width=\linewidth]{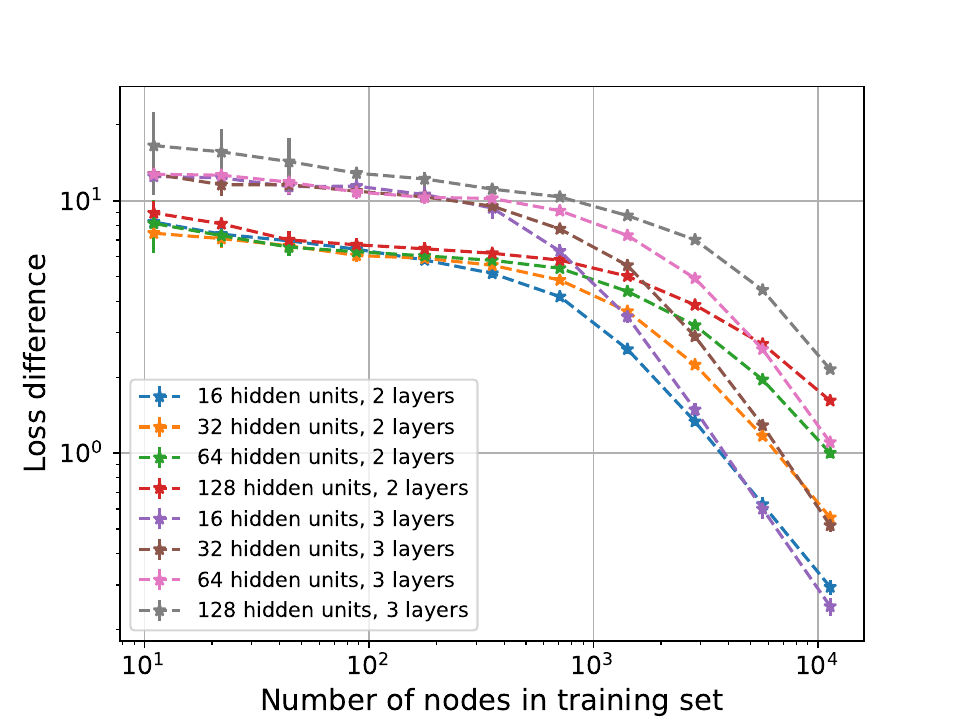}
        % \caption{Loss in Train/Test in Roman Dataset for a $2$ Layers GNN.}
            \caption{Generalization Gap}
    \end{subfigure}
    \begin{subfigure}{0.32\textwidth}
        \centering
        \includegraphics[width=\linewidth]{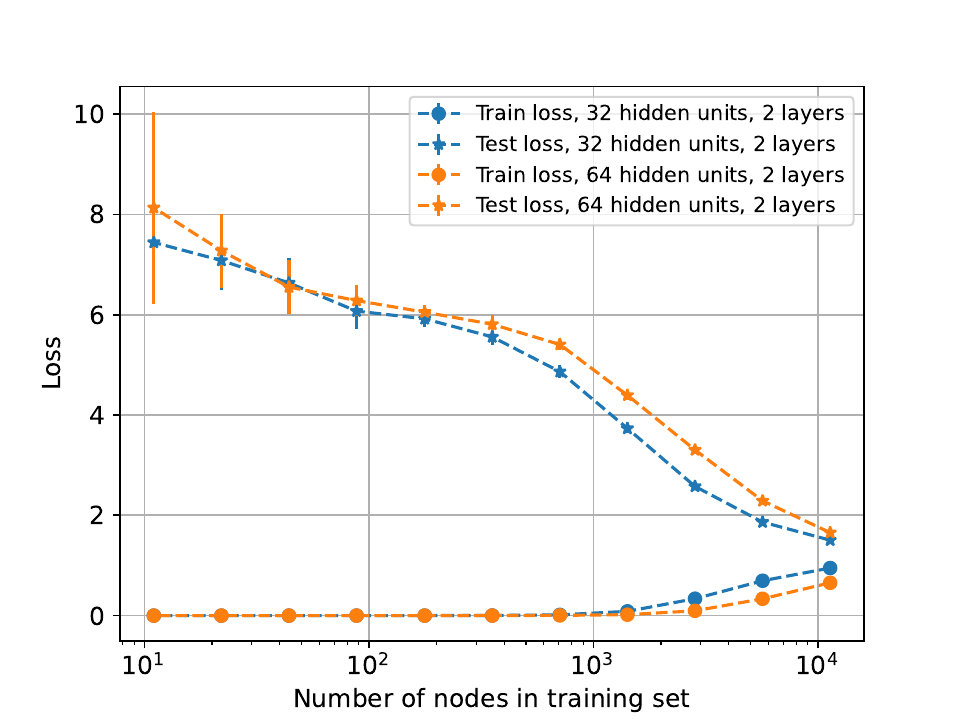}
        % \caption{Accuracy in Train/Test in Roman Dataset for a $3$ Layers GNN.}
        \caption{$2$ Layers}
    \end{subfigure}%
    \begin{subfigure}{0.32\textwidth}
        \centering
        \includegraphics[width=\linewidth]{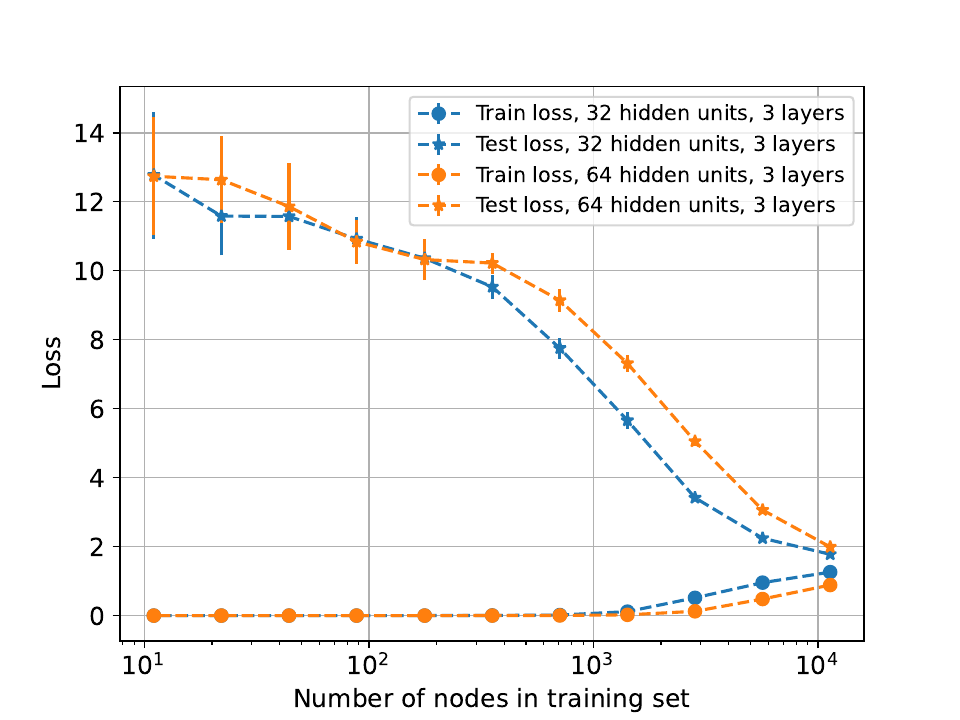}
        % \caption{Loss in Train/Test in Roman Dataset for a $3$ Layers GNN.}
        \caption{$3$ Layers}
    \end{subfigure}
\end{figure*}

\begin{figure*}[ht!]
    \centering
    \begin{subfigure}{0.5\textwidth}
        \centering
        \includegraphics[width=\linewidth]{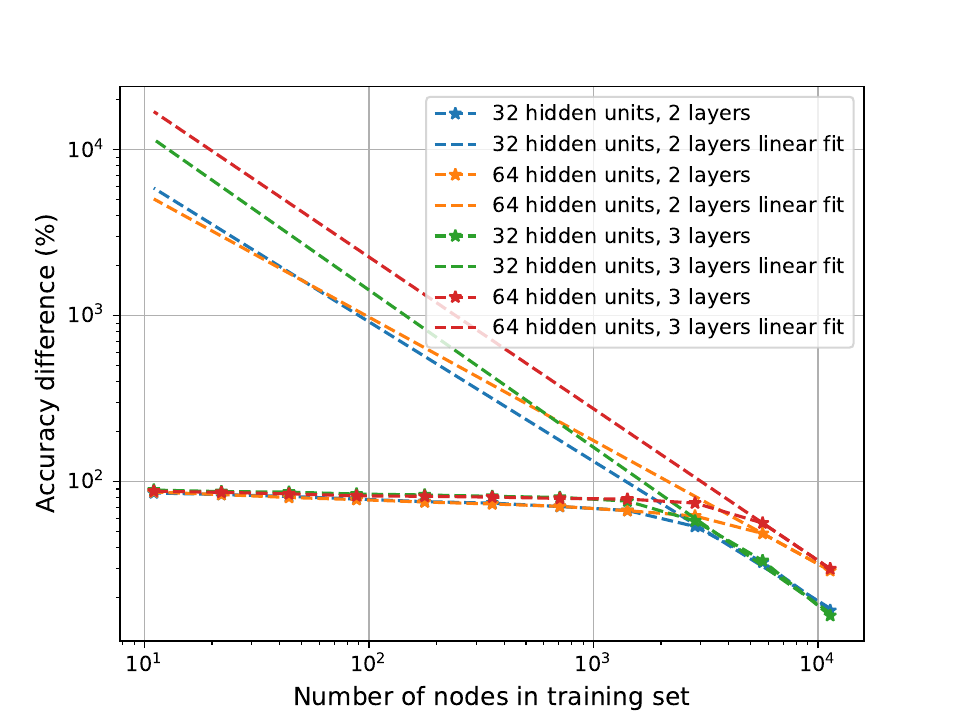}
        \caption{Linear fit for accuracy generalization gap}
        % \label{fig:CS_acc_1}
    \end{subfigure}%
        \begin{subfigure}{0.5\textwidth}
        \centering
        \includegraphics[width=\linewidth]{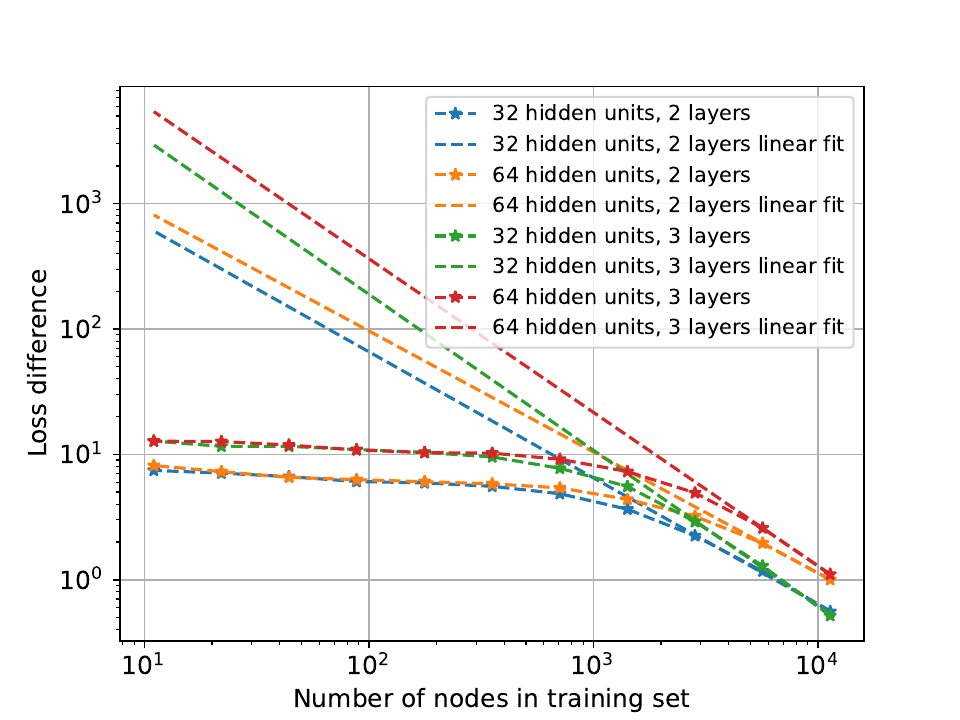}
        % \caption{Accuracy in Train/Test in CS Dataset for a $3$ Layers GNN.}
        \caption{Linear fit for loss generalization gap}
    \end{subfigure}%
    \caption{Generalization Gaps as a function of the number of nodes in the training set in the Roman dataset. }
\end{figure*}

\input{tables/roman}

%% file: tables/experiment_details.tex
\begin{table}[!ht]
\centering
\begin{tabular}{|c|c|c|c|c|c|}
\hline
Name             & Nodes     & Edges       & Features & \begin{tabular}[c]{@{}c@{}}Number\\ of Classes\end{tabular} & Reference                                         \\ \hline\hline
Arxiv            & $169,343$ & $1,166,243$ & $128$    & $40$                                                        & \cite{wang2020microsoft,mikolov2013distributed} \\ \hline
Cora             & $2,708$   & $10,556$    & $1,433$  & $7$                                                         & \cite{yang2016revisiting}                       \\ \hline
CiteSeer         & $3,327$   & $9,104$     & $3,703$  & $6$                                                         & \cite{yang2016revisiting}                       \\ \hline
PubMed           & $19,717$  & $88,648$    & $500$    & $3$                                                         & \cite{yang2016revisiting}                       \\ \hline
Coauthor Physics & $18,333$  & $163,788$   & $6,805$  & $15$                                                        & \cite{shchur2018pitfalls}                       \\ \hline
Coauthor CS      & $34,493$  & $495,924$   & $8,415$  & $5$                                                         & \cite{shchur2018pitfalls}                       \\ \hline
Amazon-ratings   & $24,492$  & $93,050$    & $300$    & $5$                                                         & \cite{platonov2023critical}                     \\ \hline
Roman-empire     & $22,662$  & $32,927$    & $300$    & $18$                                                        & \cite{platonov2023critical}                     \\ \hline
\end{tabular}
\caption{Details of the datasets considered in the experiments.}
\label{tab:datasets_information}
\end{table}

%% file: tables/arxiv.tex
\begin{table}[!ht]
\centering
\begin{tabular}{|c|c|c|c|c|c|}
\hline
Type             & Lay.     & Feat.       & Slope & Point & \begin{tabular}[c]{@{}c@{}}Pearson\\ Correlation \\ Coefficient\end{tabular}                                        \\ \hline\hline
Accuracy & $2$ & $64$  & $-6.301e-01$ & $3.621e+00$  & $-9.980e-01$  \\ \hline
Accuracy & $2$ & $128$  & $-6.034e-01$ & $3.663e+00$  & $-9.985e-01$\\ \hline
Accuracy & $2$ & $256$  & $-5.347e-01$ & $3.493e+00$  & $-9.952e-01$ \\ \hline
Accuracy & $2$ & $512$  & $-5.328e-01$ & $3.605e+00$  & $-9.975e-01$ \\ \hline
Accuracy & $3$ & $64$  & $-6.271e-01$ & $3.600e+00$  & $-9.987e-01$  \\ \hline
Accuracy & $3$ & $128$  & $-5.730e-01$ & $3.567e+00$  & $-9.970e-01$  \\ \hline
Accuracy & $3$ & $256$  & $-4.986e-01$ & $3.393e+00$  & $-9.910e-01$  \\ \hline
Accuracy & $3$ & $512$  & $-4.529e-01$ & $3.315e+00$  & $-9.934e-01$ \\ \hline
Accuracy & $4$ & $64$  & $-5.343e-01$ & $3.236e+00$  & $-9.971e-01$ \\ \hline
Accuracy & $4$ & $128$  & $-5.096e-01$ & $3.299e+00$  & $-9.987e-01$ \\ \hline
Accuracy & $4$ & $256$  & $-4.827e-01$ & $3.337e+00$  & $-9.920e-01$ \\ \hline
Accuracy & $4$ & $512$  & $-4.264e-01$ & $3.229e+00$  & $-9.927e-01$ \\ \hline
Loss & $2$ & $64$  & $-6.853e-01$ & $2.265e+00$  & $-9.975e-01$ \\ \hline
Loss & $2$ & $128$  & $-6.562e-01$ & $2.311e+00$  & $-9.988e-01$\\ \hline
Loss & $2$ & $256$  & $-5.907e-01$ & $2.174e+00$  & $-9.968e-01$ \\ \hline
Loss & $2$ & $512$  & $-5.848e-01$ & $2.280e+00$  & $-9.989e-01$ \\ \hline
Loss & $3$ & $64$  & $-6.739e-01$ & $2.228e+00$  & $-9.980e-01$ \\ \hline
Loss & $3$ & $128$  & $-6.229e-01$ & $2.224e+00$  & $-9.976e-01$ \\ \hline
Loss & $3$ & $256$  & $-5.581e-01$ & $2.111e+00$  & $-9.942e-01$ \\ \hline
Loss & $3$ & $512$  & $-5.141e-01$ & $2.057e+00$  & $-9.955e-01$ \\ \hline
Loss & $4$ & $64$  & $-6.039e-01$ & $1.964e+00$  & $-9.980e-01$ \\ \hline
Loss & $4$ & $128$  & $-5.701e-01$ & $2.014e+00$  & $-9.991e-01$ \\ \hline
Loss & $4$ & $256$  & $-5.379e-01$ & $2.051e+00$  & $-9.951e-01$ \\ \hline
Loss & $4$ & $512$  & $-4.810e-01$ & $1.957e+00$  & $-9.937e-01$ \\ \hline
\end{tabular}
\caption{Details of the linear approximation of the Arxiv Dataset. Note that in this case, we used only the values of the generalization gap whose training error is below $95\%$.}
\label{tab:arxiv}
\end{table}

%% file: tables/cora.tex
\begin{table}[!ht]
\centering
\begin{tabular}{|c|c|c|c|c|c|}
\hline
Type             & Lay.     & Feat.       & Slope & Point & \begin{tabular}[c]{@{}c@{}}Pearson\\ Correlation \\ Coefficient\end{tabular}                                      \\ \hline\hline
Accuracy & $2$ & $16$  & $-2.839e-01$ & $2.022e+00$  & $-9.803e-01$  \\ \hline
Accuracy & $2$ & $32$  & $-2.917e-01$ & $2.014e+00$  & $-9.690e-01$  \\ \hline
Accuracy & $2$ & $64$  & $-3.006e-01$ & $2.021e+00$  & $-9.686e-01$  \\ \hline
Accuracy & $3$ & $16$  & $-2.656e-01$ & $1.996e+00$  & $-9.891e-01$  \\ \hline
Accuracy & $3$ & $32$  & $-2.637e-01$ & $2.008e+00$  & $-9.679e-01$  \\ \hline
Accuracy & $3$ & $64$  & $-2.581e-01$ & $1.981e+00$  & $-9.870e-01$  \\ \hline
Loss & $2$ & $16$  & $-3.631e-01$ & $9.406e-01$  & $-9.250e-01$ \\ \hline
Loss & $2$ & $32$  & $-4.228e-01$ & $9.638e-01$  & $-9.657e-01$ \\ \hline
Loss & $2$ & $64$  & $-4.991e-01$ & $1.067e+00$  & $-9.776e-01$ \\ \hline
Loss & $3$ & $16$  & $-4.131e-01$ & $1.276e+00$  & $-9.753e-01$ \\ \hline
Loss & $3$ & $32$  & $-4.605e-01$ & $1.385e+00$  & $-9.730e-01$ \\ \hline
Loss & $3$ & $64$  & $-4.589e-01$ & $1.455e+00$  & $-9.756e-01$ \\ \hline
\end{tabular}
\caption{Details of the linear approximation of the Cora Dataset. Note that in this case we used all the values given that the training accuracy is $100\%$ for all nodes.}
\label{tab:Cora}
\end{table}

%% file: tables/citeseer.tex
\begin{table}[!ht]
\centering
\begin{tabular}{|c|c|c|c|c|c|}
\hline
Type             & Lay.     & Feat.       & Slope & Point & \begin{tabular}[c]{@{}c@{}}Pearson\\ Correlation \\ Coefficient\end{tabular}                                     \\ \hline\hline
Accuracy & $2$ & $16$  & $-1.699e-01$ & $1.972e+00$  & $-9.518e-01$  \\ \hline
Accuracy & $2$ & $32$  & $-1.856e-01$ & $1.978e+00$  & $-9.714e-01$  \\ \hline
Accuracy & $2$ & $64$  & $-1.749e-01$ & $1.966e+00$  & $-9.534e-01$  \\ \hline
Accuracy & $3$ & $16$  & $-1.585e-01$ & $1.956e+00$  & $-9.721e-01$  \\ \hline
Accuracy & $3$ & $32$  & $-1.659e-01$ & $1.963e+00$  & $-9.721e-01$  \\ \hline
Accuracy & $3$ & $64$  & $-1.658e-01$ & $1.967e+00$  & $-9.702e-01$  \\ \hline
Loss & $2$ & $16$  & $-1.049e-01$ & $7.757e-01$  & $-5.924e-01$ \\ \hline
Loss & $2$ & $32$  & $-1.762e-01$ & $7.646e-01$  & $-7.981e-01$ \\ \hline
Loss & $2$ & $64$  & $-2.186e-01$ & $8.384e-01$  & $-9.120e-01$ \\ \hline
Loss & $3$ & $16$  & $-1.802e-01$ & $1.169e+00$  & $-8.345e-01$ \\ \hline
Loss & $3$ & $32$  & $-1.629e-01$ & $1.200e+00$  & $-8.767e-01$ \\ \hline
Loss & $3$ & $64$  & $-5.917e-02$ & $1.283e+00$  & $-2.562e-01$ \\ \hline
\end{tabular}
\caption{Details of the linear approximation of the CiteSeer Dataset. Note that in this case we used all the values given that the training accuracy is $100\%$ for all nodes.}
\label{tab:CiteSeer}
\end{table}

%% file: tables/pubmed.tex
\begin{table}[!ht]
\centering
\begin{tabular}{|c|c|c|c|c|c|c|}
\hline
Type             & Lay.     & Feat.       & Slope & Point & \begin{tabular}[c]{@{}c@{}}Pearson\\ Correlation \\ Coefficient\end{tabular}                                         \\ \hline\hline
Accuracy & $2$ & $16$  & $-2.523e-01$ & $1.834e+00$  & $-9.942e-01$  \\ \hline
Accuracy & $2$ & $32$  & $-2.433e-01$ & $1.812e+00$  & $-9.583e-01$  \\ \hline
Accuracy & $2$ & $64$  & $-2.764e-01$ & $1.869e+00$  & $-9.761e-01$  \\ \hline
Accuracy & $3$ & $16$  & $-2.748e-01$ & $1.844e+00$  & $-9.910e-01$  \\ \hline
Accuracy & $3$ & $32$  & $-2.661e-01$ & $1.861e+00$  & $-9.712e-01$  \\ \hline
Accuracy & $3$ & $64$  & $-2.558e-01$ & $1.827e+00$  & $-9.890e-01$  \\ \hline
Loss & $2$ & $16$  & $-4.166e-01$ & $7.695e-01$  & $-9.718e-01$ \\ \hline
Loss & $2$ & $32$  & $-4.733e-01$ & $7.852e-01$  & $-9.137e-01$ \\ \hline
Loss & $2$ & $64$  & $-4.368e-01$ & $7.547e-01$  & $-9.718e-01$ \\ \hline
Loss & $3$ & $16$  & $-4.424e-01$ & $1.067e+00$  & $-9.549e-01$ \\ \hline
Loss & $3$ & $32$  & $-5.518e-01$ & $1.223e+00$  & $-9.655e-01$ \\ \hline
Loss & $3$ & $64$  & $-5.246e-01$ & $1.169e+00$  & $-9.632e-01$ \\ \hline
\end{tabular}
\caption{Details of the linear approximation of the PubMed Dataset. Note that in this case we used all the values given that the training accuracy is $100\%$ for all nodes.}
\label{tab:PubMed}
\end{table}

%% file: tables/CS.tex
\begin{table}[!ht]
\centering
\begin{tabular}{|c|c|c|c|c|c|}
\hline
Type             & Lay.     & Feat.       & Slope & Point & \begin{tabular}[c]{@{}c@{}}Pearson\\ Correlation \\ Coefficient\end{tabular}                              \\ \hline\hline
Accuracy & $2$ & $32$  & $-2.138e-01$ & $1.659e+00$  & $-9.007e-01$  \\ \hline
Accuracy & $2$ & $64$  & $-2.250e-01$ & $1.685e+00$  & $-8.969e-01$  \\ \hline
Accuracy & $3$ & $32$  & $-1.979e-01$ & $1.695e+00$  & $-9.009e-01$  \\ \hline
Accuracy & $3$ & $64$  & $-1.862e-01$ & $1.646e+00$  & $-8.980e-01$  \\ \hline
Loss & $2$ & $32$  & $-2.523e-01$ & $6.273e-01$  & $-8.244e-01$  \\ \hline
Loss & $2$ & $64$  & $-2.933e-01$ & $7.762e-01$  & $-7.925e-01$  \\ \hline
Loss & $3$ & $32$  & $-3.558e-01$ & $1.207e+00$  & $-8.924e-01$  \\ \hline
Loss & $3$ & $64$  & $-3.560e-01$ & $1.256e+00$  & $-8.568e-01$  \\ \hline
\end{tabular}
\caption{Details of the linear approximation of the CS Dataset. Note that in this case we used all the values given that the training accuracy is $100\%$ for all nodes.}
\label{tab:CS}
\end{table}

%% file: tables/Physics.tex
\begin{table}[!ht]
\centering
\begin{tabular}{|c|c|c|c|c|c|}
\hline
Type             & Lay.     & Feat.       & Slope & Point & \begin{tabular}[c]{@{}c@{}}Pearson\\ Correlation \\ Coefficient\end{tabular}                                      \\ \hline\hline
Accuracy & $2$ & $32$  & $-1.524e-01$ & $1.235e+00$  & $-9.064e-01$  \\ \hline
Accuracy & $2$ & $64$  & $-1.478e-01$ & $1.218e+00$  & $-9.145e-01$  \\ \hline
Accuracy & $3$ & $32$  & $-1.227e-01$ & $1.190e+00$  & $-9.328e-01$  \\ \hline
Accuracy & $3$ & $64$  & $-1.268e-01$ & $1.200e+00$  & $-8.826e-01$  \\ \hline
Loss & $2$ & $32$  & $-1.111e-01$ & $-5.257e-02$  & $-7.591e-01$  \\ \hline
Loss & $2$ & $64$  & $-9.684e-02$ & $-7.335e-02$  & $-7.696e-01$  \\ \hline
Loss & $3$ & $32$  & $-1.410e-01$ & $2.875e-01$  & $-8.280e-01$  \\ \hline
Loss & $3$ & $64$  & $-1.068e-01$ & $2.388e-01$  & $-7.679e-01$  \\ \hline
\end{tabular}
\caption{Details of the linear approximation of the Physics Dataset. Note that in this case we used all the values given that the training accuracy is $100\%$ for all nodes.}
\label{tab:Physics}
\end{table}

%% file: tables/amazon.tex
\begin{table}[!ht]
\centering
\begin{tabular}{|c|c|c|c|c|c|c|}
\hline
Type             & Lay.     & Feat.       & Slope & Point & \begin{tabular}[c]{@{}c@{}}Pearson\\ Correlation \\ Coefficient\end{tabular}                                    \\ \hline\hline
Accuracy & $2$ & $32$  & $-7.693e-01$ & $4.236e+00$  & $-9.914e-01$  \\ \hline
Accuracy & $2$ & $64$  & $-7.788e-01$ & $4.404e+00$  & $-9.972e-01$ \\ \hline
Accuracy & $3$ & $32$  & $-7.268e-01$ & $4.101e+00$  & $-9.868e-01$  \\ \hline
Accuracy & $3$ & $64$  & $-7.354e-01$ & $4.257e+00$  & $-9.921e-01$  \\ \hline
Loss & $2$ & $32$  & $-1.086e+00$ & $3.971e+00$  & $-9.968e-01$  \\ \hline
Loss & $2$ & $64$  & $-1.096e+00$ & $4.189e+00$  & $-9.985e-01$  \\ \hline
Loss & $3$ & $32$  & $-1.134e+00$ & $4.339e+00$  & $-9.965e-01$  \\ \hline
Loss & $3$ & $64$  & $-1.154e+00$ & $4.629e+00$  & $-9.991e-01$  \\ \hline
\end{tabular}
\caption{Details of the linear approximation of the Amazon Dataset. Note that in this case we used only the values of the generalization gap whose training error is below $95\%$.}
\label{tab:Amazon}
\end{table}

%% file: tables/roman.tex
\begin{table}[!ht]
\centering
\begin{tabular}{|c|c|c|c|c|c|}
\hline
Type             & Lay.     & Feat.       & Slope & Point & \begin{tabular}[c]{@{}c@{}}Pearson\\ Correlation \\ Coefficient\end{tabular}                                       \\ \hline\hline
Accuracy & $2$ & $32$  & $-8.408e-01$ & $4.644e+00$  & $-9.963e-01$  \\ \hline
Accuracy & $2$ & $64$  & $-7.435e-01$ & $4.477e+00$  & $-1.000e+00$  \\ \hline
Accuracy & $3$ & $32$  & $-9.476e-01$ & $5.049e+00$  & $-9.956e-01$  \\ \hline
Accuracy & $3$ & $64$  & $-9.145e-01$ & $5.182e+00$  & $-1.000e+00$  \\ \hline
Loss & $2$ & $32$  & $-1.006e+00$ & $3.829e+00$  & $-9.992e-01$ \\ \hline
Loss & $2$ & $64$  & $-9.656e-01$ & $3.915e+00$  & $-1.000e+00$  \\ \hline
Loss & $3$ & $32$  & $-1.244e+00$ & $4.764e+00$  & $-9.994e-01$ \\ \hline
Loss & $3$ & $64$  & $-1.225e+00$ & $5.011e+00$  & $-1.000e+00$  \\ \hline
\end{tabular}
\caption{Details of the linear approximation of the Roman Dataset. Note that in this case we used only the values of the generalization gap whose training error is below $95\%$}
\label{tab:Roman}
\end{table}